%% file: arxiv.tex
\documentclass[11pt]{article}
\usepackage[T1]{fontenc}
\usepackage{lmodern}
\usepackage[margin=1in]{geometry}
\usepackage{amsmath}
\usepackage{amssymb}
\usepackage{amsthm}
\usepackage{mathtools}
\usepackage{enumitem}
\usepackage{microtype}
\usepackage[x11names]{xcolor}
\usepackage[numbers,square]{natbib}
\usepackage{url}
\usepackage{placeins}
\usepackage[algo2e,ruled]{algorithm2e}
\usepackage{hyperref}
\input{macros.tex}
\hypersetup{
  colorlinks=true,
  linkcolor=blue!60!black,
  citecolor=blue!60!black,
  urlcolor=blue!60!black,
  pdftitle={\papertitle},
  pdfauthor={Vanessa Kosoy, Vinayak Pathak}
}
\newtheorem{theorem}{Theorem}
\newtheorem{lemma}[theorem]{Lemma}
\newtheorem{remark}[theorem]{Remark}
\newcommand{\acks}[1]{\section*{Acknowledgments}#1}

\title{\papertitle}
\author{%
  \textbf{Vanessa Kosoy}\textsuperscript{1,2}\hspace{2em}\textbf{Vinayak Pathak}\textsuperscript{2,}\thanks{Corresponding author: \href{mailto:path.vinayak@gmail.com}{\nolinkurl{path.vinayak@gmail.com}}.}\\[1.5ex]
  {\small\textsuperscript{1}Faculty of Mathematics, Technion, Haifa, Israel}\\
  {\small\textsuperscript{2}CORAL\thanks{Computational Rational Agents Laboratory, Delaware, USA. \url{https://coral-research.org/}}}}
\date{}

\begin{document}
\maketitle

\input{abstract.tex}

\input{manuscript.tex}
\end{document}

%% file: macros.tex
\newcommand{\papertitle}{On the Computational Tractability of Robust Bandits}
\makeatletter
\DeclareRobustCommand{\citewithinparens}[1]{%
  \ifNAT@numbers\citet{#1}\else\citealp{#1}\fi
}
\makeatother
\definecolor{draftblue}{HTML}{0057D9}
\long\def\draft#1{{#1}}
\newcommand{\credal}{%
  \mathord{\raisebox{0.055em}{%
    \setlength{\fboxsep}{0pt}\setlength{\fboxrule}{0.055em}%
    \fbox{\rule{0pt}{0.57em}\rule{0.57em}{0pt}}%
  }}\hspace{0.04em}%
}
\SetAlgoCaptionSeparator{.}
\DontPrintSemicolon
\setlist[enumerate]{label=\arabic*.,leftmargin=2em}

%% file: abstract.tex
\begin{abstract}%
Learning when the environment does not belong to the
learner's hypothesis class is typically handled using agnostic learning guarantees. However, for anything beyond supervised learning, agnostic guarantees are difficult to come by. Recently, \emph{imprecise bandits}~\citep{kosoy2025imprecise} (later renamed to \emph{robust bandits} in~\citewithinparens{appel2025regret}) were introduced as another approach to unrealizable learning in the bandits setting and a $\Theta(\sqrt{T})$ regret learner was shown for a large class. However, no computational guarantees were provided. In this
paper we identify a special case that admits a polynomial-time learner with
$\tilde{O}(\sqrt{T})$ regret. We also show that several small
generalizations of this special case are NP-hard thus indicating that the special case is at the boundary of what is tractable. It has been recently suggested~\citep{kosoy2018learningagenda} that computationally efficient learners for unrealizable learning problems are crucial for solving the AI alignment problem. This work is a small step in that direction.
\end{abstract}

%% file: manuscript.tex
% Translated from the unchanged Typst manuscript; hardness proofs moved to Appendix D.

\section{Introduction}
Making sure that an artificial superintelligence does not pursue goals that are unaligned with humans is a challenging problem, and arguably one of the most important of our times~\citep{ngo2024alignment,bengio2024managing}. Several attempts have been made to get a handle on it~\citep{russell2015research,amodei2016concrete,everitt2018agi,ji2023alignment} and a recurring theme is that at its core it is a learning problem~\citep{dewey2011learning,bostrom2014superintelligence,soares2016value,russell2019human,ng2000algorithms,hadfieldmenell2016cooperative,hadfieldmenell2017inverse,milli2017obedient,jeon2020reward,christiano2017deep,leike2018scalable,ouyang2022training,bai2022training,shah2019preferences,armstrong2018occam,hubinger2019risks,casper2023open}. Indeed, imagine several agents coexisting in the universe and interacting with each other. Suppose that some of these agents (the AIs) are trying to act in the interests of other groups of agents (their users). The agents only have partial information about each other's behaviours and goals as well as about the universe, and must ensure the universe evolves towards desirable states. Given the role of learning, it is not unreasonable to imagine that learning theory might help us pave the way towards solving the problem~\citep{kosoy2018learningagenda}.

However, the currently known results in learning theory are likely insufficient to solve alignment. In particular, it has been suggested that designing better learners for unrealizable learning is crucial for solving the problem~\citep{kosoy2018learningagenda, demski2019embedded}. Realizability is a strong assumption in general, but it becomes even more questionable in the multiagent learning scenarios described above.

One concrete issue with realizable learning is the grain of truth problem~\citep{hutter2009open,leike2016grain,wyeth2025grain,kosoy2023agenda}. Suppose Alice and Bob are playing a known game. Neither knows the other's strategy but assumes that they lie in a class $H$ of strategies~\citep{kalai1993rational,noguchi2015bayesian}. Can each player simultaneously achieve a low regret? The issue is that the strategy that achieves a low regret against all strategies in $H$ may not itself be a member of $H$~\citep{nachbar1997prediction,nachbar2005beliefs}. In that case, if Alice's strategy is in $H$, then Bob's strategy will have to be outside of $H$ and at least one of the players needs to solve an unrealizable learning problem. Finding a general and natural class $H$ whose learner lies within $H$ is called the grain of truth problem that remains open despite some recent progress using \emph{reflective oracles}~\citep{fallenstein2015reflective,leike2016grain,wyeth2025grain}.

More generally, realizability works only when the agent is more expressive than its environment. However, if the environment contains other agents that are similar or more powerful than the agent doing the learning then realizability fails. This can happen, for example, when Alice and Bob are running similar source codes~\citep{tennenholtz2004program,barasz2014robust}, or when Bob is able to predict Alice's actions~\citep{critch2019bounded,oesterheld2019robust}. 

The usual response to unrealizability is to provide agnostic learning guarantees. Here, one makes no assumption about the environment but guarantees that the learner's performance is competitive among a given class of strategies~\citep{kearns1994agnostic,bendavid2009agnostic}. However, designing learners with agnostic learning guarantees has been difficult. In supervised learning, there are many natural examples where agnostic learning is computationally intractable even though the realizable version is easy~\citep{daniely2016halfspaces,daniely2016dnf}. In more sophisticated learning frameworks, agnostic learning is even statistically hard~\citep{krishnamurthy2016pac,sekhari2021agnostic}. In online learning, one can generally get a much better regret for the realizable case than for agnostic~\citep{littlestone1988learning,bendavid2009agnostic}. In reinforcement learning, the agnostic case is usually statistically intractable~\citep{li2025agnostic,jia2023agnostic,sekhari2021agnostic} and only very simple problems are known to have good sample complexity bounds.

Closer to our current work, in an ``agnostic'' version of the bandits problem, one makes no assumptions about the mapping from arms to reward, but competes against the best policy in a given class. This can work if the class contains a small number of policies~\citep{auer2002nonstochastic}, however, impossibility results are easy to prove for infinite policy classes~\citep{bubeck2011pure,bubeck2011xarmed}. Moreover, even for finite classes, since the regret grows as the square root of the size of the policy class, the situation is hopeless if the class is finite but large. 

\paragraph{Our results} In this paper, we build on the framework of robust bandits as introduced in~\citet{kosoy2025imprecise}, which can be seen as an alternative to agnostic learning in the unrealizable case. The general idea is to say that we still have a known hypothesis class that contains a ground truth hypothesis, however, each hypothesis in our class provides only a partial specification of the environment. Given such a hypothesis, an adversary can pick any environment that is consistent with the specifications. The task of the learner is to compete with the optimal strategy that knows the true hypothesis but does not know the choices made by the adversary. Details can be found in Section~\ref{sec:setting}. In~\citet{kosoy2025imprecise}, a class of instances of the robust bandits problem, which can be seen as a generalization of stochastic linear bandits, was shown to have a $\Theta(\sqrt{T})$ regret. Here, we investigate the computational complexity of the learner. We demonstrate a special case of the setting of~\citet{kosoy2025imprecise} in which a polynomial time learner achieves $O(\sqrt{T})$ regret. We also show that generalizing this special instance in minor ways quickly leads to computational intractability, thus demonstrating that it is at the boundary of what is tractable.

\section{Setting}
\label{sec:setting}
A \emph{credal set} over a domain $D$ is a nonempty closed convex set of probability distributions over $D$, where closedness is with respect to the weak topology. We denote by $\credal D$ the set of all credal sets over $D$.

An instance of the robust bandits problem is a tuple $\mathcal{M} = \left( X , D , H , r \right)$, where $X$ is a set of arms, $D$ is a set of possible outcomes, $H$ is a set of mappings of type $X \to \credal D$, and $r : X \times D \to \mathbb{R}$ is a reward function.

The robust bandits game is played between a learner and an adversary for $T$ rounds (both adversary and learner know $T$). At the beginning of the game, the adversary picks some $h^{\star} \in H$. The learner knows $X , D , H , r$, and $T$, but does not have knowledge of $h^{\star}$. Then, at each round $t$:

\begin{enumerate}
\item Learner picks an arm $x_{t} \in X$.

\item Adversary picks a distribution $P_{t} \in h^{\star} \left( x_{t} \right)$.

\item An outcome $y_{t} \sim P_{t}$ is drawn and shown to the learner as feedback.

\item Learner gets reward $r \left( x_{t} , y_{t} \right)$.

\end{enumerate}
Given a hypothesis $h \in H$, we define the value functions
$v_{h} \left( x \right) \coloneqq \operatorname*{min}_{P \in h \left( x \right)} \mathbb{E}_{y \sim P} \left[ r \left( x , y \right) \right]$
and $V_{h} \coloneqq \operatorname*{max}_{x \in X} v_{h} \left( x \right)$.
We use the shorthand $v^{\ast} \left( x \right)$ and $V^{\star}$ to denote the value functions corresponding to the true hypothesis $h^{\star}$.

The regret over horizon $T$ is

\begin{equation*}
R_{T} \coloneqq T V^{\star} - \mathbb{E} \left[ \sum_{t = 1}^{T} r \left( x_{t} , y_{t} \right) \right] .
\end{equation*}

Here the expectation in the second term is taken wrt the randomness in the learner's and the adversary's strategies as well as the randomness arising from sampling $y_{t} \sim P_{t}$.

We are interested in learner strategies that can achieve a sublinear (in $T$) regret for all true hypotheses $h^{\star} \in H$. We say that a robust bandits instance $\mathcal{M} = \left( X , D , H , r \right)$ has a \emph{statistically efficient} learner if such a learner strategy exists.

\textbf{Computational tractability.} In this paper, we are interested in studying the computational complexity of the learner, and thus we need to precisely specify the computational problem being solved. We assume that the learner is specialized to a parameterized family $\{\mathcal{M}_\phi\}$ of robust bandits instances. Then, a \emph{learner policy} is a randomized algorithm that takes, as input, a history $h_t=((x_i,y_i))_{i=1}^t$ and an encoding of the parameters $\phi$, and outputs the next arm $x_{t+1}$. We use $|\phi|$ to denote the total encoding length
of $\phi$. We will say that a learner is \emph{computationally efficient} if its policy runs in time polynomial in $|\phi|$ and $T$ and achieves a regret $R_{T} \le \operatorname{poly} \left( | \phi | \right) T^{1 - \alpha}$ for some fixed constant $\alpha > 0$.\footnote{\label{fn:delay}
Note that this is a standard way of defining computationally efficient
no-regret learners~\citep[Section~3]{foster2016online}. In particular, it is important to require a $T^{1-\alpha}$ regret as opposed to merely sublinear in $T$ since a sublinear regret, exponential time learner can be easily converted into a sublinear regret, polynomial time one using Pitt's delaying trick: take an arbitrary learner and postpone its expensive computation until enough data have arrived
that the computation becomes polynomial in the accumulated input
size~\citep[Section~2.3.2]{eyraud2016efficiency}.
}

\subsection{Linear Robust Bandits}\label{sec:linear-imprecise-bandits}

For most of this work, we study the following linear specialization of the
instance $\mathcal{M} = \left( X , D , H , r \right)$ defined above. Let
$X \subseteq \mathbb{R}^{d_{X}}$ be the arm set, and let
$D \subseteq \mathbb{R}^{d_{D}}$ be the outcome set. We assume that $X$ and $D$ are
Euclidean balls of radii $R_{X}$ and $R_{D}$ (not necessarily centred at the origin).
The reward is linear in the arm and outcome,
$r \left( x , y \right) \coloneqq a^{T} x + b^{T} y$.
Write $C_{r} \coloneqq \operatorname*{max}_{X \times D} r - \operatorname*{min}_{X \times D} r$ for the range of the reward.

{\draft{Fix $d_C\in \mathbb{N}$.} Each hypothesis $h$ uses a set of $d_C$ affine constraints as follows. Let
$B\in\mathbb{R}^{d_C\times d_X}$, $C\in\mathbb{R}^{d_C\times d_D}$ be matrices, and $d\in\mathbb{R}^{d_C}$ a vector. Define
\begin{equation}
K_h(x)\coloneqq\{y\in D:Cy+Bx+d=0\},\label{eq:setting-compatible-set}
\end{equation}
and $h(x)\coloneqq\{P\in\Delta D:\mathbb{E}_{y\sim P}[y]\in K_h(x)\}$.
Let $H$ consist of all hypotheses of this form for which $K_h(x)$ is
nonempty for every $x\in X$} \draft{and which satisfy the transversality
condition of~\citet{kosoy2025imprecise} with a cutoff
$S \in \left( 0 , 1 \right]$. That is, $h$ admits a representation $B,C,d$
such that, for every $x \in X$ and every $p \in \mathbb{R}^{d_{D}}$
satisfying $C p + B x + d = 0$ with $p \notin D$, we have
$\operatorname{dist} \left( p , D \right) \ge S \operatorname{dist} \left( p , K_h \left( x \right) \right)$.}
For a Euclidean ball, this says that the affine solution spaces defined by
the \draft{constraints} remain uniformly bounded away from tangency to $D$.

\draft{We write $B^{\star},C^{\star},d^{\star}$ for such a representation of
the true hypothesis $h^{\star}$, and} define $K^{\star}(x)\coloneqq K_{h^{\star}}(x)$.

Because $r$ is linear in its outcome argument, \draft{the value function
$v^{\ast}$ defined above satisfies
$v^{\ast}(x)=\min_{y\in K^{\star}(x)}r(x,y)$.}

\draft{For the computational problem, the parameters $\phi$ of
Section~\ref{sec:setting} are the dimensions $d_X,d_D,d_C$, the centres and
radii of $X$ and $D$, the reward coefficients $a,b$, the cutoff $S$, and an
integer upper bound $C_{\max}\ge C_r$ on the reward range. They determine
$X$, $D$, $r$, and $H$.}
{We assume these numerical data are rational and given explicitly.}
\draft{Our regret bounds grow polynomially with $S^{-1}$ and $C_r$, so we want
$|\phi|$ to grow at least as fast. We therefore require $S^{-1}$ to be an
integer and give $S^{-1}$ and $C_{\max}$ in unary. Then $S^{-1}\le|\phi|$ and
$C_r\le C_{\max}\le|\phi|$. Requiring $S^{-1}$ to be an integer loses
nothing, since lowering $S$ only weakens the transversality condition.}

% \draft{The unary encoding also lets the running time grow polynomially with
% $S^{-1}$ and $C_{\max}$, but this hardly weakens the requirement on the
% learner. Suppose that $S^{-1}$ and $C_{\max}$ were given in binary, and that a
% learner runs in time polynomial in $S^{-1}$, $C_{\max}$, the encoding
% length, and $T$. The delaying trick of Footnote~\ref{fn:delay} makes its
% running time polynomial in the encoding length and $T$ alone. We play an
% arbitrary arm in the first $S^{-1}+C_{\max}$ rounds and only then start the
% learner. Each of these rounds costs at most $C_r\le C_{\max}$, so the regret
% grows by at most $C_{\max}(S^{-1}+C_{\max})$, which a regret bound polynomial
% in $S^{-1}$ and $C_{\max}$ already allows.}

Running time is measured in bit operations.

Now we state our main result.

\begin{theorem}[efficient $\sqrt{T}$ learning]\label{thm:efficient-upper-bound}
There exists a policy for linear robust bandits as defined above that,
  for every time horizon $T$ and every true hypothesis $h^{\star} \in H$,
  runs in time polynomial in $| \phi |$ and $T$ and achieves regret
  $R_{T} \le \tilde{O} \left( \sqrt{T} \right)$.
  The hidden factor depends polynomially on the dimensions, $S^{- 1}$,
  and $C_{r}$. The precise dependence is given in
  Theorem~\ref{thm:soft-square-root-regret}.

\end{theorem}

\draft{The dimensions, $S^{-1}$, and $C_r\le C_{\max}$ are all at most
$|\phi|$. So this policy is computationally efficient in the
sense of Section~\ref{sec:setting}.}

This dependence on $T$ is tight up to logarithmic factors. We first observe that
linear robust bandits contain stochastic linear bandits as a special case, and then
transfer the known lower bound for stochastic linear bandits.

\begin{remark}\label{rem:stochastic-linear-bandits}
Linear robust bandits contain stochastic linear bandits as a special case. In a
stochastic linear bandit, the learner picks arms $x_{t}$ from the unit ball
$X \subseteq \mathbb{R}^{d}$ and observes a reward in $\left[ - 1 , 1 \right]$ whose mean is
$\theta^{T} x_{t}$, for an unknown $\theta$ in the unit ball of $\mathbb{R}^{d}$. To write
this as an instance of linear robust bandits, we let the outcome be the reward itself:
take $D = \left[ - 1 , 1 \right]$, $r \left( x , y \right) = y$, and
{for each $\theta$ in the unit ball of $\mathbb{R}^{d}$, define $h_{\theta}$ by}
$K_{h_{\theta}} \left( x \right) \coloneqq \left\{ y \in D : y - \theta^{T} x = 0 \right\} = \left\{ \theta^{T} x \right\}$.
This has the required form with $C_{\theta} = 1$, $B_{\theta} = - \theta^{T}$, and
$d_{\theta} = 0$. Since $\left| \theta^{T} x \right| \le 1$, the only solution of the constraint
already lies in $D$, so transversality holds with $S = 1$. The hypothesis
$h_{\theta} \left( x \right)$ now contains every distribution on $\left[ - 1 , 1 \right]$ with mean $\theta^{T} x$.
In particular, it contains the reward distribution of the stochastic linear
bandit at $x$, and the adversary may choose that distribution at every round.
Against this adversary, the learner sees exactly what it would see in the
stochastic linear bandit. The regret is also the same in both problems, since
$v^{\ast} \left( x \right) = \theta^{T} x$ and $V^{\star} = \operatorname*{max}_{x \in X} \theta^{T} x$. So any learner for linear
robust bandits is also a learner for stochastic linear bandits with the
same regret.
\end{remark}

\begin{theorem}[lower bound, informal]\label{thm:lower-bound-informal}
Every learner for linear robust bandits has worst-case regret $\Omega \left( \sqrt{T} \right)$.
\end{theorem}

\begin{proof}
Stochastic linear bandits have a worst-case regret lower bound of
$\Omega \left( \sqrt{T} \right)$~\citep{dani2008stochastic}. By
Remark~\ref{rem:stochastic-linear-bandits}, the same bound holds for linear robust bandits.
\end{proof}

\section{Related work}\label{sec:related-work}

\paragraph{Imprecise probabilities and credal sets} While there is substantial theory to support the idea that a rational agent's uncertain beliefs can be encoded using probabilities~\citep{savage1954foundations, jaynes2003probability}, there are situations where probabilities aren't enough (see~\citewithinparens{dalrymple2026imprecise} for examples). Encoding one's belief using a convex set of probabilities~\citep{ellsberg1961risk} and making decisions by maximizing the worst case expected utility can fix many of the issues and this decision rule can be shown to be implied by natural axioms on preferences~\citep{gilboa1989maxmin}. Levi called a closed convex set of distributions that represents a belief a credal set~\citep{levi1974indeterminate,levi1980enterprise}, and Walley built a general theory of imprecise probability around such sets~\citep{walley1991statistical}. Robust bandits can be thought of as generalizing the notion of stochastic bandits using imprecise probabilities.

\paragraph{Partial specifications of the environment} The idea that a hypothesis in the learner's hypothesis class should be a partial specification of the environment has been considered in previous work. The most notable is the work on learning partial concept classes~\citep{alon2021theory}, where a hypothesis is allowed to output a `don't know' label. Closer to our work in spirit is the framework introduced by \citet*{pour2026learning} (see also~\citewithinparens{kosoy2026ambiguous}) where each hypothesis outputs a \emph{set} of labels for each instance, and the learner's job is to guess any label from that set. Or in other words, the regret of the learner is the difference in loss compared to the learner that knows the true hypothesis but has no idea what precise label from the set of labels prescribed by the hypothesis is the adversary going to pick. This is exactly what robust bandits model, but in the bandits setting.

\paragraph{Outcome feedback} In our model the learner does not just see its reward. It sees the outcome $y_{t}$, and the reward $r \left( x_{t} , y_{t} \right)$ is a known function of it. Similar feedback appears in partial monitoring, where the learner sees a signal that need not reveal its reward~\citep{bartok2014partial,kirschner2020information}, and in the framework of \citet{foster2021statistical}, where the learner receives an arbitrary observation along with its reward. However, these models define regret differently. They compare the learner with the best action for the true environment, or with the best fixed action in hindsight, rather than with the worst-case value of the true hypothesis.

\paragraph{Smoothed analysis and smoothed adversaries} Another way to restrict the adversary is to require it to pick a smooth distribution, i.e., one whose density with respect to a fixed base measure is at most $1/\sigma$. This is the idea behind smoothed analysis~\citep{spielman2004smoothed}, and in online learning it goes by the name of smoothed adversaries~\citep{rakhlin2011online,haghtalab2024smoothed}. The set of $\sigma$-smooth distributions is a credal set, so a smoothed adversary is a special case of our adversary. It turns out that this restriction makes online learning as easy as statistical learning~\citep{block2022smoothed}, and, closer to our concerns, it makes computationally efficient learning possible given an optimization oracle for the hypothesis class~\citep{haghtalab2022oracle}. There are two differences from our setting though. First, in smoothed online learning the credal set is known to the learner, whereas we need to learn it. Second, the regret there is measured against the best hypothesis in hindsight, whereas we compete with the worst-case value of the true hypothesis, which is easier.

{

\section{Warmup: A $T^{2/3}$ Learner}\label{sec:hard-ellipsoid-warmup}

For an arm $x$ and outcome $y$, consider the vector $w \left( x , y \right) = \left( 1 , x , y \right)$.
This vector lies in $\mathbb{R}^{1 + d_{X} + d_{D}}$, and the constraints defining the
feasible set of outcomes, $C^{\star} y + B^{\star} x + d^{\star} = 0$,
determine a linear subspace of this vector space. Define $\mathcal{L} \coloneqq \operatorname{ker} \begin{pmatrix}d^{\star} & B^{\star} & C^{\star}\end{pmatrix} .$ A distribution is feasible for arm $x$ precisely when its expectation $m$ satisfies $m \in D$ and $w \left( x , m \right) \in \mathcal{L}$.
}{The learner's task is to learn enough about $\mathcal{L}$ to get low regret.
Of course, we may never learn all of $\mathcal{L}$. For example, if the adversary keeps
choosing distributions with expected values in a strictly smaller subspace, then we cannot
distinguish that subspace from the true $\mathcal{L}$. But this is fine, since we do
not need to learn directions that the adversary never uses.}

The fact that we never get to observe the true conditional mean, but only a noisy observation of it, adds an extra complication. In a noiseless world, a new observation tells us a new direction of $\mathcal{L}$. Thus we can simply maintain a linear subspace of $\mathbb{R}^{1 + d_{X} + d_{D}}$, and incrementally grow it every time we get new information. However, in the presence of noise, no observation can be used to add an entire direction with certainty. Instead, we encode our knowledge of $\mathcal{L}$ with an \emph{ellipsoid}, and incrementally grow it as we get new information. The details are written below.

{
We first translate and rescale the balls so that
$X = \left\{ x : {\left\Vert x \right\Vert}_{2} \le 1 \right\}$ and $D = \left\{ y : {\left\Vert y \right\Vert}_{2} \le 1 \right\}$. We continue to write the
reward as $r \left( x , y \right) = a^{T} x + b^{T} y$ and the true subspace as $\mathcal{L}$ in
these coordinates. The new ${\left\Vert b \right\Vert}_{2}$ equals $R_{D} {\left\Vert b \right\Vert}_{2}$ in the
original coordinates, and the non-tangency constant $S$ stays the same.
{The translation also adds a constant to the reward. This constant
shifts every reward and $V^{\star}$ by the same amount, so it changes
neither the regret nor any choice of the learner, and we drop it.}
Write $q \coloneqq 1 + d_{X} + d_{D}$ for the number of coordinates of $w \left( x , y \right)$.
For a positive-definite matrix $V$, write
${\left\Vert w \right\Vert}_{V^{- 1}} \coloneqq \sqrt{w^{T} V^{- 1} w}$.

Define $p_{V} \left( x , y \right) \coloneqq {\left\Vert w \left( x , y \right) \right\Vert}_{V^{- 1}}^{2} ,$ and $E_{V} \coloneqq \left\{ w \in \mathbb{R}^{q} : {\left\Vert w \right\Vert}_{V^{- 1}}^{2} \le 1 \right\} .$

Here $E_{V}$ is the ellipsoid we described above, which we incrementally
expand as we get new information about $\mathcal{L}$.
The algorithm works in blocks of length $n$ (we will choose $n$ later
to get a small regret). At the beginning of each block, we choose an
arm $x$ and play the same arm throughout the block.
The ellipsoid $E_{V}$ stays fixed during these rounds.

To choose an arm, we use our current guess for the feasible expected values. For each $x$, this consists of the $m \in D$ with
$\left( 1 , x , m \right) \in E_{V}$. If some arm has no such $m$, we choose it to learn
more about it. Otherwise, we choose the arm with the largest worst-case
reward over its guessed feasible means. This is an example of the principle of ``optimism in the face of uncertainty'' because our set of feasible means is conservative, and therefore our arm values are optimistic. 

At the end of a full block, we average its $n$ observations to get
$\overline{y}$. We then check whether $\left( 1 , x , \overline{y} \right)$ lies in $E_{V}$.
If it does not, we enlarge the ellipsoid to include it. Otherwise, we
leave the ellipsoid unchanged.

For the time being, we assume that we can carry out the rule for
choosing an arm exactly. We will consider the computational aspects
of the algorithm later.

\begin{algorithm2e}
\caption{}\label{alg:hard-ellipsoid-warmup}
\KwIn{Horizon $T$ and block length $n\in\{1,\ldots,T\}$.}
$V\leftarrow n^{-1}I_q$\;
\For{$j\leftarrow 1$ \KwTo $\lceil T/n\rceil$}{
  Define $\hat{K}_V(x)\coloneqq\{y\in D:p_V(x,y)\le1\}$ for every $x\in X$\;
  \eIf{there exists $x\in X$ with $\hat{K}_V(x)=\varnothing$}{
    Choose an arm $x\in X$ with $\hat{K}_V(x)=\varnothing$\;
  }{
    Define $H_V(x)\coloneqq\min_{y\in\hat{K}_V(x)}r(x,y)$ for every $x\in X$\;
    Choose $x\in\operatorname*{argmax}_{x'\in X}H_V(x')$\;
  }
  $n_j\leftarrow\min\{n,T-(j-1)n\}$\;
  \For{$i\leftarrow1$ \KwTo $n_j$}{
    Play $x$ and observe $y_i$\;
  }
  $\overline{y}\leftarrow\frac{1}{n_j}\sum_{i=1}^{n_j}y_i$,\quad
  $\overline{w}\leftarrow w(x,\overline{y})$\;
  \If{$n_j=n$ and $p_V(x,\overline{y})>1$}{
    $V\leftarrow V+\overline{w}\,\overline{w}^{T}$\;
  }
}
\end{algorithm2e}

\begin{theorem}[Informal]\label{thm:hard-ellipsoid-warmup}
  Algorithm~\ref{alg:hard-ellipsoid-warmup} has expected regret $R_{T} \le \tilde{O} \left( n + \frac{T}{\sqrt{n}} \right)$. In particular, choosing $n = \left\lceil T^{2/3} \right\rceil$ gives
  $R_{T} \le \tilde{O} \left( T^{2/3} \right)$.

\end{theorem}

}

{
\begingroup
\renewcommand{\proofname}{Proof sketch}
\begin{proof}
The sketch relies on three lemmas, which we state here without proof. Their
precise versions, with explicit constants, and their proofs are in
Appendix~\ref{app:hard-ellipsoid-analysis}: see, respectively, Lemma~\ref{lem:subspace-distance-bound},
Lemma~\ref{lem:subspace-reward-bound}, and the proof of Theorem~\ref{thm:hard-ellipsoid-regret}.

\begin{lemma}\label{lem:sketch-ellipsoid-close}
With probability at least $1 - \delta$, every ellipsoid $E_{V}$ that
Algorithm~\ref{alg:hard-ellipsoid-warmup} maintains during the $T$ rounds satisfies
$\operatorname{dist} \left( w , \mathcal{L} \right) \le \tilde{O} \left( \frac{1}{\sqrt{n}} \right)$ for all
$w \in E_{V}$, where $\tilde{O}$ hides a factor polynomial in the dimensions and
logarithmic in $T$ and $\frac{1}{\delta}$.
\end{lemma}

\begin{lemma}\label{lem:sketch-close-reward}
There is a constant $G$, depending only on $S$ and ${\left\Vert b \right\Vert}_{2}$, such that
for every $x \in X$ and $y \in D$,
$r \left( x , y \right) \ge v^{\star} \left( x \right) - G \operatorname{dist} \left( w \left( x , y \right) , \mathcal{L} \right)$.
\end{lemma}

\begin{lemma}\label{lem:sketch-few-updates}
Algorithm~\ref{alg:hard-ellipsoid-warmup} updates $V$ at most $O \left( \log n  \right)$ times.
\end{lemma}

Lemma~\ref{lem:sketch-close-reward} is a straightforward consequence of the fact that $r$ is linear. Lemmas~\ref{lem:sketch-ellipsoid-close} and~\ref{lem:sketch-few-updates} are due to the specific way we update $V$. Adding $\overline{w} \, \overline{w}^{T}$ to $V$ has the effect that it expands $E_{V}$ in the correct direction. {This ensures that the distance between $E_{V}$ and the true $\mathcal{L}$ grows only by the noise in $\overline{w}$, which is about $\frac{1}{\sqrt{n}}$.} Also, one can show that each update increases the volume of $E_{V}$ by a constant factor, which explains why we can have at most a logarithmic number of updates.

Now consider an uninformative block, meaning a full block in which we
do not update $V$. The arm $x$ that we play during this block is picked
by maximizing $H_{V}$. Since the block is uninformative, its
empirical mean outcome satisfies $w \left( x , \overline{y} \right) \in E_{V}$. These two
facts, together with {Lemmas~\ref{lem:sketch-ellipsoid-close} and~\ref{lem:sketch-close-reward}},
imply that the average reward is at least $V^{\star}$ minus about
$\frac{1}{\sqrt{n}}$. Thus the block's regret (the best arm's guaranteed reward
$n V^{\star}$ minus the reward we actually achieved) is at most about
$\frac{n}{\sqrt{n}} = \sqrt{n}$.

On the other hand, for an informative block, the empirical mean can
lie outside the ellipsoid, so we only have the trivial regret bound
of $O \left( n \right)$.

Finally, by Lemma~\ref{lem:sketch-few-updates} there are at most $O(\log n)$ informative and $O\left( \frac{T}{n}\right)$
uninformative blocks, so the total regret is about {$n \operatorname{log} n + \sqrt{n} \frac{T}{n}$}.

Ignoring logarithmic factors, this bound is minimized by choosing
$n = \left\lceil T^{2/3} \right\rceil$, giving an upper bound of
$\tilde{O} \left( T^{2/3} \right)$ on the regret.
\end{proof}
\endgroup

The precise statement and proof are in Appendix~\ref{app:hard-ellipsoid-analysis}.
}

\subsection{Computational tractability}\label{sec:hard-computation}
In this section we describe a polytime implementation of Algorithm~\ref{alg:hard-ellipsoid-warmup}.

The matrix updates $V \leftarrow V + \overline{w}\overline{w}^{T}$ and the test $p_{V} \left( x , \overline{y} \right) > 1$ can be computed
in polynomial time. The main issue is choosing an arm, which involves solving the following two computational problems:

\begin{enumerate}
\item Decide whether there exists an arm $x\in X$ with
  $\hat{K}_V(x)=\varnothing$, and if so, find such an arm.
\item If $\hat{K}_V(x)$ is nonempty for every $x\in X$, find an arm
  $x\in\operatorname*{argmax}_{x'\in X}H_V(x')$, where
  $H_V(x)=\min_{y\in\hat{K}_V(x)}r(x,y)$.
\end{enumerate}

One difficulty is that this is a max-min problem, since the score $H_V$ itself is a result of a minimization problem. However, the minimization problem is convex, and thus using convex duality, can be written as a maximization problem. We use the following lemma.

\begin{lemma}\label{lem:hard-dual-forms}
Fix $x\in X$ and write $z=(z_0,z_X,z_D)\in\mathbb{R}^q$.
\begin{enumerate}
\item[(a)] $\hat K_V(x)=\varnothing$ if and only if
\begin{equation*}
\max_{\|z\|_2\le1}\left[z_0+z_X^Tx-\sqrt{z^TVz}-\|z_D\|_2\right]>0 .
\end{equation*}
\item[(b)] If $\hat K_V(x)\ne\varnothing$, then
\begin{equation*}
H_V(x)=a^Tx+\sup_{z\in\mathbb{R}^q}\left[z_0+z_X^Tx-\sqrt{z^TVz}-\|b+z_D\|_2\right].
\end{equation*}
\end{enumerate}
\end{lemma}

\begin{proof}
Fix $x$. We start with (b). Note that $\hat{K}_V(x)$ is an intersection of two ellipsoids. If it was just one ellipsoid, then the inner minimization problem would have an easy closed-form solution. To handle the intersection of two ellipsoids, we can use Lagrangian duality.

We first show how this works in general, and then apply it to our case. Let $A,B\subseteq\mathbb{R}^q$ be compact convex sets with a nonempty intersection, and suppose we want to minimize $g^Tv$
over $v\in A\cap B$. Introduce two copies of $v$: one vector $s\in A$ and
another $u\in B$, with the equality $s=u$ requiring them to represent the
same point. Let $z\in\mathbb{R}^q$ be the Lagrangian multipliers for the constraints $s=u$, and consider the objective function
$L(s,u,z)\coloneqq g^Ts+z^T(s-u)=(g+z)^Ts-z^Tu$.
Whenever $s=u$, the added term vanishes and $L(s,u,z)=g^Ts$. Thus, for
every fixed $z$, dropping the equality and minimizing $L$ over all
$(s,u)\in A\times B$ gives a lower bound on the original minimum. This
minimization separates into one linear optimization over $A$ and another
over $B$. We then choose $z$ to make the lower bound as large as possible.
This is the dual problem:
\begin{equation*}
\sup_{z\in\mathbb{R}^q}\min_{(s,u)\in A\times B}L(s,u,z)
=\sup_{z\in\mathbb{R}^q}
\left[\min_{s\in A}(g+z)^Ts+\min_{u\in B}(-z)^Tu\right].
\end{equation*}

We still need to show that the best lower bound equals the original
minimum. First note that it's clear that $\min_{v\in A\cap B}g^Tv
=\min_{(s,u)\in A\times B}\sup_{z\in\mathbb{R}^q}L(s,u,z)$. Indeed, for fixed $s,u$, the supremum of $L(s,u,z)$ over $z$ is $g^Ts$
if $s=u$ and $+\infty$ otherwise: when $s\ne u$, take $z=t(s-u)$ and let
$t\to\infty$. Next, since $A\times B$ is compact and convex and $L$ is
continuous and affine in each of $(s,u)$ and $z$, Sion's minimax
theorem~\citep{sion1958minimax} lets us exchange the minimum and supremum.

In our problem, take $A=S_x\coloneqq\{w(x,y):y\in D\}$, $B=E_V$, and
$g=(0,0,b)$, so that $g^Tw(x,y)=b^Ty$. The equality $s=u$ above is now
$u=w(x,y)=(1,x,y)$: the vector $u$ in the ellipsoid must be the same point
as the one specified by $y\in D$. For fixed $z$, Cauchy--Schwarz gives
the two separate minima explicitly:
\begin{equation*}
\begin{aligned}
\min_{s\in S_x}(g+z)^Ts
&=z_0+z_X^Tx+\min_{y\in D}(b+z_D)^Ty
=z_0+z_X^Tx-\|b+z_D\|_2,\\
\min_{u\in E_V}(-z)^Tu&=-\sqrt{z^TVz}.
\end{aligned}
\end{equation*}
Substituting these formulas into the general identity proves (b). No strictly feasible point is needed, and the
supremum need not be attained.

For (a), the set $\hat K_V(x)$ is empty exactly when the compact convex
sets $S_x$ and $E_V$ are disjoint.
By the strict separating hyperplane theorem, this is equivalent to the
existence of $z$ with
\begin{equation*}
\min_{s\in S_x}z^Ts-\max_{u\in E_V}z^Tu
=z_0+z_X^Tx-\|z_D\|_2-\sqrt{z^TVz}>0.
\end{equation*}
The expression is positively homogeneous in $z$, so we may restrict to
$\|z\|_2\le1$. It is continuous, so its maximum over that ball is attained.
This proves (a).
\end{proof}

Thus having converted the inner problem into a maximization\footnote{Technically it's a supremum. {It need not be attained when the two ellipsoids just touch. In Appendix~\ref{app:hard-computation}, we enlarge the ellipsoid to create a margin that lets us bound the Lagrange multipliers and obtain an attained maximum. This changes the score being optimized, but the resulting approximate implementation preserves the regret rate.}} problem, the entire optimization becomes a maximization. However it is not convex, and it's still unclear if it can be solved efficiently. Fortunately, it is a \draft{\emph{quadratically constrained quadratic program} (QCQP)} with some nice properties \draft{(see Appendix~\ref{app:hard-computation} for details)} and thus has a polynomial time algorithm. We use the following result of \citet[Theorem~1.3]{bienstock2016cdt}.

\begin{theorem}[\citewithinparens{bienstock2016cdt}]
  \label{thm:bienstock}
Consider a maximization problem with a nonempty feasible set that has: a) a fixed number of quadratic constraints with at least one of them being a bounded ellipsoid, b) a quadratic objective function, and c) rational coefficients for both the objective function and the constraints. Then, there exists an algorithm that, for every
$\varepsilon\in(0,1)$, returns rational assignments to the variables that violate
each constraint by at most $\varepsilon$ and achieve an objective value that is at
least the optimum minus $\varepsilon$. Its running time is polynomial in
the input size and $\log(1/\varepsilon)$.
\end{theorem}

Of course, Theorem~\ref{thm:bienstock} only gives us an approximate
solution. This is fine because one can choose the $\varepsilon$
as an appropriate function of $T$ such that the approximate solution is
still good enough that it does not change the regret rate. Since the
runtime depends logarithmically on $1/\varepsilon$, this only costs us
an extra $\log T$ factor in the runtime. 

\section{A $\sqrt{T}$ learner}\label{sec:soft-square-root}

{
The previous section attempts to apply the principle of ``optimism in the
face of uncertainty,'' but does not fully embrace it. Indeed, the arm for
each block is picked optimistically, but each block still looks like a
standard exploration block. The learner does not always need to wait for
$n$ steps before identifying that the arm it is playing is problematic.
If the outcomes collected so far have an empirical mean $\overline{y}$
with $w(x,\overline{y})$ already far outside the current ellipsoid, then
the block can be cut short more quickly. However, the current arm score $H_V(x)$ 
does not account for outcomes outside the ellipsoid. First, we design a
score function that takes every outcome in $D$ into account, not just
those represented inside our chosen ellipsoid. The score function will
be linked to the regret in a way that lets us decide when to move on to
the next block.
}

{
We construct this score by replacing the hard constraint
$p_{V} \left( x , y \right) \le 1$ in $H_{V}$ with a penalty, so that we can
minimize over all of $D$. For a parameter $\lambda > 0$, which we will
choose later to get a small regret, define
$U_{V , \lambda} \left( x \right) \coloneqq \operatorname*{min}_{y \in D} \left[ r \left( x , y \right) + \lambda p_{V} \left( x , y \right) \right]$.
We pick the arm $x$ by maximizing $U_{V,\lambda}$. For now, assume that we can
maximize the score exactly.

\noindent The following lemma gives us a way to control regret when the outcome $y$
appears anywhere in $D$.

\begin{lemma}\label{lem:soft-outcome-regret}
  Fix $\delta\in(0,1)$. With probability at least $1-\delta$, simultaneously
  for every matrix $V$ that Algorithm~\ref{alg:soft-square-root} considers,
  every $\lambda>0$, and every arm $x$ maximizing $U_{V,\lambda}$,
  $V^\star-r(x,y)\le \lambda p_V(x,y)+\tilde{O}(1/\lambda)$ for every $y\in D$.
  Here $\tilde{O}$ hides logarithmic factors in $T/\delta$.
\end{lemma}

\begin{proof}
We first claim that, with probability at least $1-\delta$, for every matrix
$V$ that Algorithm~\ref{alg:soft-square-root} considers, every arm $x'\in X$,
and every outcome $y\in D$,
$v^\star(x')-r(x',y)\le c\sqrt{p_V(x',y)}$,
where $c$ does not depend on $x'$, $y$, or $\lambda$; it is polynomial in
the dimensions and logarithmic in $T/\delta$, so $c=\tilde{O}(1)$. We prove
this in Appendix~\ref{app:soft-square-root-analysis}, as
Equation~\ref{eq:soft-reward-certificate}.

Now fix any arm $x'$ and outcome $y\in D$. Rearranging the claim and adding
the penalty $\lambda p_V(x',y)$ to both sides gives
$r(x',y)+\lambda p_V(x',y)\ge v^\star(x')+\lambda p_V(x',y)-c\sqrt{p_V(x',y)}$.
By the AM--GM inequality,
$\lambda p+\frac{c^2}{4\lambda}\ge2\sqrt{\lambda p\cdot\frac{c^2}{4\lambda}}=c\sqrt{p}$
for every $p\ge0$, so the right-hand side is at least
$v^\star(x')-\frac{c^2}{4\lambda}$. Minimizing over $y$ gives
$U_{V,\lambda}(x')\ge v^\star(x')-\frac{c^2}{4\lambda}$, and since $x$
maximizes the score, $U_{V,\lambda}(x)\ge V^\star-\frac{c^2}{4\lambda}$.
Finally, the definition of the score gives
$U_{V,\lambda}(x)\le r(x,y)+\lambda p_V(x,y)$ for every $y\in D$. Combining
the last two bounds,
$V^\star-r(x,y)\le\lambda p_V(x,y)+\frac{c^2}{4\lambda}$, which is the claim
with $\tilde{O}(1/\lambda)=\frac{c^2}{4\lambda}$. Since $c$ does not depend on
$\lambda$, the argument holds for all $\lambda>0$ on the same event.
\end{proof}

To use this bound during a block, we keep $V$ fixed and play the same
arm $x$ repeatedly, but recompute the empirical mean after every
observation. Let ${\overline{y}}_n$ be the mean of the first $n$ outcomes
in the block. Since the reward is linear, the reward accumulated so far
is $n r(x,{\overline{y}}_n)$. Applying Lemma~\ref{lem:soft-outcome-regret}
to this mean gives

\begin{equation*}
n V^\star-\sum_{i=1}^n r(x,y_i)
\le \lambda n p_V(x,{\overline{y}}_n)+\tilde{O}(n/\lambda).
\end{equation*}

The first term is observable. We can compute it after every observation,
and we end the block and update $V$ the first time it reaches $\lambda$,
or equivalently, when $n p_V(x,{\overline{y}}_n)\ge1$.
Before this happens, the block's regret is bounded by
$\lambda+\tilde{O}(n/\lambda)$.

A large empirical penalty can therefore make us update after only a
few observations. A small one lets us keep playing the arm for longer. Thus we choose when to update $V$ by monitoring this bound
throughout the block, rather than waiting for a fixed number of rounds.

\begin{algorithm2e}
\caption{}\label{alg:soft-square-root}
\KwIn{Horizon $T$ and penalty parameter $\lambda>0$.}
$V\leftarrow I_q$,\quad $t\leftarrow0$\;
\While{$t<T$}{
  Define $U_{V,\lambda}(x)\coloneqq\min_{y\in D}[r(x,y)+\lambda p_V(x,y)]$ for every $x\in X$\;
  Choose $x\in\operatorname*{argmax}_{x'\in X}U_{V,\lambda}(x')$\;
  $n\leftarrow0$\;
  \Repeat{$n p_V(x,\overline{y}_n)\ge1$ or $t=T$}{
    Play $x$ and observe $y_{n+1}$\;
    $n\leftarrow n+1$,\quad $t\leftarrow t+1$\;
    $\overline{y}_n\leftarrow\frac{1}{n}\sum_{i=1}^n y_i$\;
  }
  \If{$n p_V(x,\overline{y}_n)\ge1$}{
    $\overline{w}_n\leftarrow w(x,\overline{y}_n)$\;
    $V\leftarrow V+n\overline{w}_n\overline{w}_n^{T}$\;
  }
}
\end{algorithm2e}
% \FloatBarrier

\begin{theorem}[Informal]\label{thm:soft-square-root}
  Algorithm~\ref{alg:soft-square-root}, with $\lambda = \lceil \sqrt{T} \rceil$,
  has expected regret $R_T \le \tilde{O}\left(\sqrt{T}\right)$.
\end{theorem}
\begingroup
\renewcommand{\proofname}{Proof sketch}
\begin{proof}
As before, each completed block at least doubles
the determinant of $V$, so there are
only $O(\log T)$ blocks.

By Lemma~\ref{lem:soft-outcome-regret}, with high probability, a block of
length $n$ has regret at most
$\lambda n p_V(x,\overline{y})+\tilde{O}(n/\lambda)$.
The stopping rule keeps $n p_V(x,\overline{y})$ bounded by a constant,
so the first term costs $O(\lambda)$ per block, regardless of its length.
The remaining term costs about $1/\lambda$ per round.
Summing both costs gives regret
at most about $\lambda \log T + T/\lambda$. Ignoring
logarithmic factors, choosing $\lambda = \lceil\sqrt{T}\rceil$ balances
the two terms.
\end{proof}
\endgroup

To run the algorithm in polynomial time we use an almost identical approach as the previous section. The main differences are that there is no emptiness test any more and the inner minimization problem has an extra quadratic term. But the same techniques can be applied to convert the entire optimization problem into a maximization problem and the Bienstock's result can be used to solve it in polytime. The precise regret bound and detailed proofs of the algorithm are in
Appendix~\ref{app:soft-square-root-analysis}.
}

\section{NP-hardness}

The precise statements and proofs of the results below are in Appendix~\ref{app:np-hardness-proofs}.

What other instances of robust bandits have efficient, sublinear regret learners? Interestingly, the linear instance studied in the previous sections appears to be at the boundary of what's tractable. Generalizing it in natural ways quickly gives us computationally intractable problems. 

First, we show that if $D$ is a simplex instead of a Euclidean ball, the problem is computationally hard (Appendix~\ref{app:hardness-simplex}). If $D$ remains a ball, but we change $X$ to be a simplex, then the problem is essentially the same as having {$d_X+1$ arms, one for each vertex of $X$, since a non-vertex arm is never better than randomizing over the vertices. With finitely many arms, robust bandits reduce to adversarial bandits~\citep{kosoy2024thesis}, so running Exp3~\citep{auer2002nonstochastic} on the vertices gives $\tilde O(\sqrt{d_X T})$ regret in polynomial time.} However, if $X$ is a convex polytope with $k$ faces, then the problem once again becomes computationally hard (Appendix~\ref{app:hardness-polytope}). Next, we analyze other kinds of hypothesis classes $H$ than the ones defined by the linear constraints of Section~\ref{sec:linear-imprecise-bandits}. In \citet{kosoy2025imprecise}, it was assumed that for a given hypothesis $z$ and a given arm $x$, the set of feasible expected outcomes was given by constraints $F(x,y,z) = 0$ where $F$ was bilinear in $y$ and $z$. We ask what happens if we impose \emph{trilinearity} on $F$ while keeping $X$ and $D$ as Euclidean balls. We show once again, that this small generalization leads to a computationally hard problem (Appendix~\ref{app:hardness-trilinear}).

All the above hardness proofs rely on proving that the \emph{planning} problem is already computationally hard. That is, given the true hypothesis, finding an arm that maximizes $v^*(x)$ is hard. This, along with the fact that planning can be reduced to learning (Appendix~\ref{app:hardness-planning-from-learning}) shows that learning is also hard. This leads to a natural question: if we were given a polynomial time planning oracle, would learning become easy? We answer this in the negative by showing an example where planning is easy but learning is computationally hard (Appendix~\ref{app:hardness-easy-planning}).

\section{Discussion and open problems}\label{sec:discussion}

Robust bandits give a new way of dealing with unrealizability. It will be interesting to see if
similar definitions can be made for other frameworks of learning, and if
computationally efficient learners can be designed for them.

For robust bandits themselves, it will be nice to see which other
hypothesis classes admit computationally efficient learners. Our hardness
results show that small changes to the linear instance already make planning
hard, so such classes will need some structure that keeps planning easy, and
the example in Appendix~\ref{app:hardness-easy-planning} shows that easy
planning alone is not enough. It would be nice to understand computational efficiency in the presence of certain oracles. For example, is there a set of natural oracles, in addition to the planning oracle, whose existence guarantees a computationally efficient learner?

\citet{appel2025regret} generalize robust bandits to
online decision making with structured observations, which includes robust
linear bandits and tabular reinforcement learning, and prove regret bounds in
this setting, but they do not address computational efficiency. It will be
nice to see if the techniques of this paper could be used to design efficient
learners for their setting.

\draft{
\acks{This work was supported by the Advanced Research+Invention Agency (ARIA) of the United Kingdom and Coefficient Giving in San Francisco, California.}
}

\draft{
\section*{AI Disclosure}

This work grew out of a few months of research involving many brainstorming
sessions with GPT 5.6 and later GPT 6. The NP-hardness results and the
$\tilde{O}(T^{2/3})$ upper bound were proved mostly by the authors, with some
help from GPT in finding relevant literature. Given the accumulated notes
and observations as context, GPT 6 Astra was able to improve the upper
bound to $\tilde{O}(\sqrt{T})$.

The writing also involved help from GPT 6, Fable 5.1, and Opus 5.5. The
main body is largely human-written. The appendix proofs were written over
many iterations with AI. We supplied the models with notes containing
proofs and detailed writing guidelines, which they used to draft the
appendix. Over several iterations of feedback, we refined the guidelines
and revised the drafts.
}

\bibliography{references}

\appendix

\section{Regret analysis of the $T^{2/3}$ learner}\label{app:hard-ellipsoid-analysis}

\draft{

We give the precise statement and proof of Theorem~\ref{thm:hard-ellipsoid-warmup}.
We use the coordinates and notation from Section~\ref{sec:hard-ellipsoid-warmup},
and assume that Algorithm~\ref{alg:hard-ellipsoid-warmup} carries out its rule
for choosing an arm exactly. A full block is \emph{informative} if it causes an update,
and \emph{uninformative} otherwise.

For the analysis, fix $\delta \in \left( 0 , 1 \right)$ and define

\begin{equation*}
\begin{aligned}
J_{\text{\normalfont max}} & \coloneqq \left\lceil q \left( 4 + 2 \operatorname{log}_{2} n \right) \right\rceil , \\
 h_{T}^{2} & \coloneqq 8 d_{D} \operatorname{log} \left( \frac{2 d_{D} T^{2}}{\delta} \right) , \\
 B_{T}^{2} & \coloneqq 1 + J_{\text{\normalfont max}} h_{T}^{2} , \\
 G & \coloneqq \sqrt{3} \left( 1 + S^{- 1} \right) {\left\Vert b \right\Vert}_{2} .
\end{aligned}
\end{equation*}

We write $J$ for the number of updates actually made. We will show
that $J \le J_{\text{\normalfont max}}$.

If $w \left( x , y \right)$ is close to $\mathcal{L}$, then $r \left( x , y \right)$ cannot be much smaller
than the worst feasible reward $v^{\star} \left( x \right)$. We first prove this.

\begin{lemma}\label{lem:subspace-reward-bound}
  For every $x \in X$ and $y \in D$,

\begin{equation*}
v^{\star} \left( x \right) \le r \left( x , y \right) + G \operatorname{dist} \left( w \left( x , y \right) , \mathcal{L} \right) .
\end{equation*}

\end{lemma}

\begin{proof}
Let $N \coloneqq \operatorname{ker} C^{\star}$ and let $Q$ be the orthogonal projector
onto $N^{\perp}$. Every arm has a feasible outcome, so we can write the
minimum-norm solution of the affine constraints as
$f \left( x \right) = A x + e \in N^{\perp}$. Thus

\begin{equation*}
K^{\star} \left( x \right) = \left( A x + e + N \right) \cap D .
\end{equation*}

Since there is a feasible outcome in the unit ball for every arm,
${\left\Vert A x + e \right\Vert}_{2} \le 1$ on $X$. Taking $x = 0$ gives ${\left\Vert e \right\Vert}_{2} \le 1$, and
comparing opposite unit vectors gives ${\left\Vert A \right\Vert}_{\text{\normalfont op}} \le 1$. Now define

\begin{equation*}
\Phi \left( s , x , y \right) \coloneqq Q y - A x - s e .
\end{equation*}

This map has operator norm at most $\sqrt{3}$ and kernel $\mathcal{L}$. Indeed, the
original constraints are equivalent to $Q y = A x + e$, and homogenizing
either equation gives the same subspace. Therefore

\begin{equation*}
\operatorname{dist} \left( y , A x + e + N \right) = {\left\Vert \Phi \left( 1 , x , y \right) \right\Vert}_{2} \le \sqrt{3} \operatorname{dist} \left( w \left( x , y \right) , \mathcal{L} \right) .
\end{equation*}

Let $p$ be the point in $A x + e + N$ closest to $y$. Since $y \in D$,
the transversality condition gives

\begin{equation*}
\operatorname{dist} \left( p , K^{\star} \left( x \right) \right) \le S^{- 1} \operatorname{dist} \left( p , D \right) \le S^{- 1} {\left\Vert p - y \right\Vert}_{2} ,
\end{equation*}

where the first inequality is immediate if $p \in D$, since then
$p \in K^{\star} \left( x \right)$. The triangle inequality now gives

\begin{equation*}
\operatorname{dist} \left( y , K^{\star} \left( x \right) \right) \le \left( 1 + S^{- 1} \right) \operatorname{dist} \left( y , A x + e + N \right) .
\end{equation*}

Moving $y$ to the nearest point in $K^{\star} \left( x \right)$ changes the reward by
at most ${\left\Vert b \right\Vert}_{2}$ times this distance, which proves the claim.
\end{proof}

We also need to know how close a block average is to the average of the
true conditional means.

\begin{lemma}\label{lem:interval-noise}
  Let $m_{t}$ be the conditional mean of $y_{t}$ given the history and the
  chosen arm on round $t$. For each interval in the first $T$ rounds,
  write $n$ for its length and $\overline{y}$ and $\overline{m}$ for the
  averages of $y_{t}$ and $m_{t}$ over that interval. With probability at
  least $1 - \delta$, the bound

\begin{equation*}
n {\left\Vert \overline{y} - \overline{m} \right\Vert}_{2}^{2} \le h_{T}^{2}
\end{equation*}

  holds simultaneously on all these intervals.

\end{lemma}

\begin{proof}
Each coordinate of $y_{t} - m_{t}$ is a martingale difference
bounded in absolute value by two. For any fixed interval of length
$n$ and coordinate $i$, Azuma--Hoeffding gives

\begin{equation*}
\operatorname{Pr} \left( \left| {\overline{y}}_{i} - {\overline{m}}_{i} \right| > \frac{h_{T}}{\sqrt{n d_{D}}} \right) \le 2 \operatorname{exp} \left( - \frac{h_{T}^{2}}{8 d_{D}} \right) = \frac{\delta}{d_{D} T^{2}} .
\end{equation*}

There are $d_{D}$ coordinates and at most $T^{2}$ intervals. A union bound
and summing the squared coordinate bounds give the claim.
\end{proof}

We need to know how far a point in $E_{V}$ can be from $\mathcal{L}$.
The next lemma bounds this distance using the noise bound we just
proved. We can then apply Lemma~\ref{lem:subspace-reward-bound} to the outcomes
the learner treats as feasible.

\begin{lemma}\label{lem:subspace-distance-bound}
  With probability at least $1 - \delta$, after any number $J$ of updates
  in Algorithm~\ref{alg:hard-ellipsoid-warmup}, the matrix $V$ satisfies, for every
  $w \in \mathbb{R}^{q}$,

\begin{equation*}
\operatorname{dist} \left( w , \mathcal{L} \right) \le \sqrt{\frac{1 + J h_{T}^{2}}{n}} {\left\Vert w \right\Vert}_{V^{- 1}} .
\end{equation*}

\end{lemma}

\begin{proof}
Number the informative blocks $j = 1 , \ldots , J$. Write $x_{j}$ for
the arm played in block $j$, and ${\overline{y}}_{j}$ and ${\overline{m}}_{j}$
for its average observation and average conditional mean. With
${\overline{w}}_{j} \coloneqq w \left( x_{j} , {\overline{y}}_{j} \right)$, the update rule gives

\begin{equation*}
V = n^{- 1} I_{q} + \sum_{j = 1}^{J} {\overline{w}}_{j} {\overline{w}}_{j}^{T} .
\end{equation*}

Fix $u \in {\mathcal{L}}^{\perp}$ with ${\left\Vert u \right\Vert}_{2} \le 1$. Since the arm is held
fixed within each block, $w \left( x_{j} , {\overline{m}}_{j} \right) \in \mathcal{L}$. Each
informative block has length $n$. By Lemma~\ref{lem:interval-noise}, with
probability at least $1 - \delta$, the bound

\begin{equation*}
{\left( u^{T} {\overline{w}}_{j} \right)}^{2} = {\left( u^{T} \left( 0 , 0 , {\overline{y}}_{j} - {\overline{m}}_{j} \right) \right)}^{2} \le \frac{h_{T}^{2}}{n}
\end{equation*}

holds simultaneously for all informative blocks and all such $u$.
On this event, $u^{T} V u \le \frac{1 + J h_{T}^{2}}{n}$. For any $w$,
Cauchy--Schwarz gives $u^{T} w \le \sqrt{u^{T} V u} {\left\Vert w \right\Vert}_{V^{- 1}}$.
Taking the supremum over the unit ball of ${\mathcal{L}}^{\perp}$ proves the
claim.
\end{proof}

\begin{theorem}\label{thm:hard-ellipsoid-regret}
  Algorithm~\ref{alg:hard-ellipsoid-warmup} satisfies, with probability at least $1 - \delta$,

\begin{equation*}
T V^{\star} - \sum_{t = 1}^{T} r \left( x_{t} , y_{t} \right) \le C_{r} n \left( J_{\text{\normalfont max}} + 1 \right) + \frac{G B_{T} T}{\sqrt{n}} .
\end{equation*}

  Its expected regret is at most the same expression plus
  $C_{r} T \delta$. Taking $n = \left\lceil T^{2/3} \right\rceil$ and
  $\delta = {\left( T + 1 \right)}^{- 2}$ give $R_{T} \le \tilde{O} \left( T^{2/3} \right)$ for fixed problem
  parameters.

\end{theorem}

\begin{proof}
First, there cannot be very many informative blocks. Each one
more than doubles the determinant:

\begin{equation*}
\operatorname{det} \left( V + \overline{w} {\overline{w}}^{T} \right) = \operatorname{det} \left( V \right) \left( 1 + p_{V} \left( x , \overline{y} \right) \right) > 2 \operatorname{det} \left( V \right) .
\end{equation*}

\draft{Each update adds at most three to the trace, since ${\left\Vert \overline{w} \right\Vert}_{2}^{2} \le 3$,
so after $J$ updates $\operatorname{tr} V \le \frac{q}{n} + 3 J$. The initial determinant
is $n^{- q}$. Using the arithmetic--geometric mean inequality on the
eigenvalues, we get

\begin{equation*}
2^{J} \le \frac{\operatorname{det} V}{\operatorname{det} \left( n^{- 1} I_{q} \right)} \le {\left( 1 + 3 \frac{n J}{q} \right)}^{q} .
\end{equation*}

Write $\xi \coloneqq \frac{J}{q}$, so that $\xi \le \operatorname{log}_{2} \left( 1 + 3 n \xi \right)$. If $\xi < 4$, then
$\xi \le 4 + 2 \operatorname{log}_{2} n$ holds trivially. If $\xi \ge 4$, then
$\operatorname{log}_{2} \left( 1 + 3 n \xi \right) \le \operatorname{log}_{2} \left( 4 n \xi \right) = 2 + \operatorname{log}_{2} n + \operatorname{log}_{2} \xi \le 2 + \operatorname{log}_{2} n + \frac{\xi}{2}$,
using $\operatorname{log}_{2} \xi \le \frac{\xi}{2}$ for $\xi \ge 4$, and hence again $\xi \le 4 + 2 \operatorname{log}_{2} n$. In either case
$J \le q \left( 4 + 2 \operatorname{log}_{2} n \right) \le J_{\text{\normalfont max}}$. Note that the horizon $T$ does not enter.
The count is the logarithm of the ratio between the final volume of $E_{V}$, which is bounded
in terms of $J$ alone, and its initial volume, which is set by the block length $n$.}

Since $J \le J_{\text{\normalfont max}}$, Lemma~\ref{lem:subspace-distance-bound} gives, with probability
at least $1 - \delta$, simultaneously for every matrix $V$ that the
learner maintains and every $\left( x , y \right) \in X \times D$,

\begin{equation}
\operatorname{dist} \left( w \left( x , y \right) , \mathcal{L} \right) \le \frac{B_{T}}{\sqrt{n}} \cdot \sqrt{p_{V} \left( x , y \right)} .\label{eq:hard-subspace-certificate}
\end{equation}

On this event, if $y \in {\hat{K}}_{V} \left( x \right)$, then $p_{V} \left( x , y \right) \le 1$, so
Lemma~\ref{lem:subspace-reward-bound} gives
$r \left( x , y \right) \ge v^{\star} \left( x \right) - G \frac{B_{T}}{\sqrt{n}}$. Taking the minimum over
${\hat{K}}_{V} \left( x \right)$, whenever this set is nonempty, gives

\begin{equation*}
H_{V} \left( x \right) \ge v^{\star} \left( x \right) - G \frac{B_{T}}{\sqrt{n}} .
\end{equation*}

Consider a full block in which we do not update the ellipsoid. Its
empirical mean belongs to ${\hat{K}}_{V} \left( x \right)$. Such a block could not have
been played at an arm $x$ with ${\hat{K}}_{V} \left( x \right)$ empty, since then its
mean would lie outside the ellipsoid. Thus all the sets ${\hat{K}}_{V} \left( x^{\prime} \right)$ were nonempty, and we
chose $x$ by maximizing $H_{V}$. It follows that

\begin{equation*}
r \left( x , \overline{y} \right) \ge H_{V} \left( x \right) = \operatorname*{max}_{x^{\prime} \in X} H_{V} \left( x^{\prime} \right) \ge V^{\star} - G \frac{B_{T}}{\sqrt{n}} .
\end{equation*}

Since the reward is linear, $r \left( x , \overline{y} \right)$ is the average reward
in the block. The block's regret is therefore at most $G B_{T} \sqrt{n}$.
There are at most $\frac{T}{n}$ such blocks, for a total of $G B_{T} \frac{T}{\sqrt{n}}$.
An informative block may cost as much as $C_{r} n$, but there are at most
$J_{\text{\normalfont max}}$ of them. A final incomplete block also costs at most $C_{r} n$.
Adding these bounds proves the high-probability claim. Outside this
event, regret is still at most $C_{r} T$, so taking expectations
adds at most $C_{r} T \delta$.
\end{proof}
}

\subsection{Computational tractability and finite-precision implementation}\label{app:hard-computation}

\draft{
In Section~\ref{sec:hard-computation}, Lemma~\ref{lem:hard-dual-forms}
turned problems~1 and~2 into maximization problems over an arm $x$ and
a dual vector $z$. In this section, we first show how each of them can
be written as a QCQP that fits Theorem~\ref{thm:bienstock},
and then we explain how to use an approximate solution. For the second
step, we will modify the rule for choosing an arm and show that the
changes preserve the $\tilde O(T^{2/3})$ regret rate.

We start with problem~1, which asks whether there is an arm $x$ with
$\hat K_V(x)=\varnothing$ (recall that $\hat K_V(x)$ is the set of
outcomes $y$ in the Euclidean ball $D$ with $w(x,y)$ in the ellipsoid
$E_V$) and, if so, returns such an arm. By
Lemma~\ref{lem:hard-dual-forms}(a), such an arm exists exactly when
\begin{equation}\label{eq:hard-empty-max}
\max_{x\in X,\ \|z\|_2\le1}
\left[z_0+z_X^Tx-\sqrt{z^TVz}-\|z_D\|_2\right]>0,
\end{equation}
where, as in Lemma~\ref{lem:hard-dual-forms}, we split the dual vector
as $z=(z_0,z_X,z_D)$ with $z_0\in\mathbb R$, $z_X\in\mathbb R^{d_X}$,
and $z_D\in\mathbb R^{d_D}$, so that $z^Tw(x,y)=z_0+z_X^Tx+z_D^Ty$. If
the maximum is positive, the arm of every maximizing pair $(x,z)$ is
such an arm. The next lemma shows that Theorem~\ref{thm:bienstock} can
compute this maximum.

\begin{lemma}\label{lem:hard-qcqp-empty}
The maximum in Equation~\ref{eq:hard-empty-max} is the optimal value of
a QCQP in $x$, $z$, and two scalar variables, with nine
quadratic constraints, that satisfies the hypotheses of
Theorem~\ref{thm:bienstock}.
\end{lemma}

\begin{proof}
The maximization in Equation~\ref{eq:hard-empty-max} is not yet a
QCQP, because its objective contains the square roots
$\sqrt{z^TVz}$ and $\|z_D\|_2$. We
remove them with two new scalar variables that bound them from above.
We require $t\ge0$ and $z^TVz\le t^2$, so that $t\ge\sqrt{z^TVz}$, and
$s\ge0$ and $\|z_D\|_2^2\le s^2$, so that $s\ge\|z_D\|_2$. These
constraints are quadratic. In the objective, we write $-t-s$ in place
of the two square roots. For fixed $x$ and $z$, maximizing $-t-s$
pushes $t$ and $s$ down to the two square roots, so the optimal value
does not change.

Theorem~\ref{thm:bienstock} also needs a constraint that defines a
bounded ellipsoid in all the variables. The arm $x$ and the dual
vector $z$ already lie in unit balls, so it remains to bound $t$ and
$s$. Since $V$ is positive semidefinite, its largest eigenvalue is at
most its trace $\operatorname{tr}V$. Hence, for $\|z\|_2\le1$, the two
square roots are at most $\sqrt{\operatorname{tr}V}\le1+\operatorname{tr}V$
and $1$. We bound $t$ and $s$ by $2+\operatorname{tr}V$ and $2$. These
bounds exceed the values above by at least one, so they do not change
the optimal value. We use this extra room when we round the
solver's output below. With these bounds, the last constraint below is
implied by the others. We add it because it defines a bounded ellipsoid
in all the variables. This gives the QCQP
\begin{equation}\label{eq:hard-qcqp-empty}
\begin{aligned}
\max_{\substack{x\in\mathbb R^{d_X},\,z\in\mathbb R^q\\t,s\in\mathbb R}}
&\quad z_0+z_X^Tx-t-s\\
\text{subject to}
&\quad \|x\|_2^2\le1,\qquad \|z\|_2^2\le1,\\
&\quad z^TVz\le t^2,\qquad \|z_D\|_2^2\le s^2,\\
&\quad 0\le t\le2+\operatorname{tr}V,\qquad 0\le s\le2,\\
&\quad \|x\|_2^2+\|z\|_2^2
       +\frac{t^2}{(2+\operatorname{tr}V)^2}+\frac{s^2}{4}\le4.
\end{aligned}
\end{equation}

This program fits Theorem~\ref{thm:bienstock}. It has nine quadratic
constraints, however large $d_X$ and $d_D$ are, and a nonempty
feasible set. The objective is quadratic, and the coefficients are
rational when $V$ has rational entries, which the rounding below
ensures. The
feasible set is convex, but the objective is not concave because of
the bilinear term $z_X^Tx$.

At an optimal solution, $t$ and $s$ equal the two square roots, so the
pair $(x,z)$ attains the maximum in Equation~\ref{eq:hard-empty-max},
and the two optimal values agree. If this value is positive,
Lemma~\ref{lem:hard-dual-forms}(a) gives $\hat K_V(x)=\varnothing$ at
the arm $x$ of every optimal solution. If it is not positive, the same
lemma gives $\hat K_V(x')\ne\varnothing$ at every arm $x'$.
\end{proof}

Problem~2 asks for an arm maximizing $H_V$ when $\hat K_V(x)$ is
nonempty for every arm, and it has the extra complication that in
Lemma~\ref{lem:hard-dual-forms}(b), $H_V(x)$ is a supremum over
all $z\in\mathbb R^q$ instead of a maximum, and this supremum need not
be attained. For now, we make the simplifying assumption that, for a
known rational $R$, the supremum is attained at some $z$ with
$\|z\|_2\le R$ at every arm. We show later how to remove this
assumption by enlarging the ellipsoid. Under this assumption, it is
easy to write problem~2 as a QCQP as well.

\begin{lemma}\label{lem:hard-qcqp-score}
Suppose that $\hat K_V(x)$ is nonempty for every arm and that, for a
rational $R$, the supremum in Lemma~\ref{lem:hard-dual-forms}(b) is
attained at some $z$ with $\|z\|_2\le R$ at every arm. Then
$\max_{x\in X}H_V(x)$ is the optimal value of a QCQP in
$x$, $z$, and two scalar variables, with nine quadratic constraints,
that satisfies the hypotheses of Theorem~\ref{thm:bienstock} whenever
$V$ has rational entries. The arm of every optimal solution maximizes
$H_V$.
\end{lemma}

\begin{proof}
By the assumption and Lemma~\ref{lem:hard-dual-forms}(b),
\begin{equation*}
\max_{x\in X}H_V(x)
=\max_{x\in X,\ \|z\|_2\le R}
\left[a^Tx+z_0+z_X^Tx-\sqrt{z^TVz}-\|b+z_D\|_2\right].
\end{equation*}
Using the same trick as in the proof of Lemma~\ref{lem:hard-qcqp-empty},
we get the following QCQP, where $U\coloneqq1+(1+\operatorname{tr}V)R$
and $W\coloneqq1+\|b\|_1+R$ exceed the largest values of
$\sqrt{z^TVz}$ and $\|b+z_D\|_2$ on the ball $\|z\|_2\le R$ by at
least one:
\begin{equation}\label{eq:hard-qcqp-score}
\begin{aligned}
\max_{\substack{x\in\mathbb R^{d_X},\,z\in\mathbb R^q\\t,s\in\mathbb R}}
&\quad a^Tx+z_0+z_X^Tx-t-s\\
\text{subject to}
&\quad \|x\|_2^2\le1,\qquad \|z\|_2^2\le R^2,\\
&\quad z^TVz\le t^2,\qquad \|b+z_D\|_2^2\le s^2,\\
&\quad 0\le t\le U,\qquad 0\le s\le W,\\
&\quad \|x\|_2^2+\frac{\|z\|_2^2}{R^2}
       +\frac{t^2}{U^2}+\frac{s^2}{W^2}\le4.
\end{aligned}
\end{equation}
For the same reasons as in the proof of Lemma~\ref{lem:hard-qcqp-empty},
this program has the same optimal value as the maximization above and
satisfies the hypotheses of Theorem~\ref{thm:bienstock}, and the arm of
every optimal solution maximizes $H_V$.
\end{proof}

Theorem~\ref{thm:bienstock} only gives an approximate solution,
allowing errors in both the objective and the constraints. We now
explain how to use such a solution, starting again with problem~1.
An approximate solution cannot tell exactly whether some
$\hat K_V(x)$ is empty. The next lemma shows that it either finds an
arm with $\hat K_V(x)=\varnothing$ or certifies that every arm has a
nonempty set once we enlarge the ellipsoid slightly.

\begin{lemma}\label{lem:hard-approx-empty}
Let $V$ be a matrix with rational entries maintained by the learner,
let $\eta\in(0,1)$ be rational, and write
$\hat K_{V,\eta}^{\mathrm{out}}(x)\coloneqq\{y\in D:p_V(x,y)\le(1+\eta\sqrt n)^2\}$,
(recall that $n$ is the block length of
Algorithm~\ref{alg:hard-ellipsoid-warmup} and appears in the definition of $V$). There is an algorithm that computes a rational arm $x\in X$ and either
certifies that $\hat K_V(x)=\varnothing$ or certifies that
$\hat K_{V,\eta}^{\mathrm{out}}(x')\ne\varnothing$ for every arm
$x'\in X$. Its running time is polynomial in $d_X$, $d_D$,
$\log(1/\eta)$, and the bit length of $V$.
\end{lemma}

\begin{proof}
If $\eta'\le\eta$, then
$\hat K_{V,\eta'}^{\mathrm{out}}(x)\subseteq\hat K_{V,\eta}^{\mathrm{out}}(x)$
for every arm $x$. So we may replace $\eta$ by the largest power of
$2$ that is at most $\eta$. This changes $\log_2(1/\eta)$ by less
than one and gives $\eta$ a bit length of $O(\log(1/\eta))$.

Write $S_x\coloneqq\{w(x,y):y\in D\}$, and let $\operatorname{dist}$
denote the Euclidean distance between two sets. We first show that,
for every arm $x$,
\begin{equation}\label{eq:hard-empty-distance}
\max_{\|z\|_2\le1}\left[z_0+z_X^Tx-\sqrt{z^TVz}-\|z_D\|_2\right]
=\operatorname{dist}(E_V,S_x).
\end{equation}
By Lemma~\ref{lem:hard-qcqp-empty} and its proof, the optimal value of
Equation~\ref{eq:hard-qcqp-empty} is then
$\max_{x\in X}\operatorname{dist}(E_V,S_x)$. The Euclidean distance
between $u$ and $v$ is $\max_{\|z\|_2\le1}z^T(v-u)$. Both
$E_V\times S_x$ and the unit ball for $z$ are compact and convex, and
$z^T(v-u)$ is affine in each variable. Sion's minimax
theorem~\citep{sion1958minimax} therefore gives
\begin{equation*}
\operatorname{dist}(E_V,S_x)
=\max_{\|z\|_2\le1}
  \left[\min_{v\in S_x}z^Tv-\max_{u\in E_V}z^Tu\right],
\end{equation*}
and the linear optimization formulas from the proof of
Lemma~\ref{lem:hard-dual-forms} turn the right-hand side into the
left-hand side of Equation~\ref{eq:hard-empty-distance}.

Next, we run the algorithm of Theorem~\ref{thm:bienstock} on
Equation~\ref{eq:hard-qcqp-empty} with tolerance
$\varepsilon\coloneqq(\eta/(21+2\operatorname{tr}V))^2$. It returns a
rational point $(\hat x,\hat z,\hat t,\hat s)$ that violates each
constraint by at most $\varepsilon$ and whose objective value is at least the
optimum minus $\varepsilon$. We repair this point so that it becomes feasible.
Since $\|\hat x\|_2^2\le1+\varepsilon$ and $\|\hat z\|_2^2\le1+\varepsilon$, we can
scale $\hat x$ and $\hat z$ into the unit ball and round their
coordinates toward zero to the grid whose spacing is the largest power
of $2$ that is at most $\sqrt\varepsilon/q$. This gives rational $x$ and $z$
in the unit ball with $\|x-\hat x\|_2\le2\sqrt\varepsilon$ and
$\|z-\hat z\|_2\le2\sqrt\varepsilon$. Since the spacing is a power of $2$, the
grids used in different calls are nested. We then choose rational $t$
and $s$ with $\sqrt{z^TVz}\le t\le\sqrt{z^TVz}+\sqrt\varepsilon$ and
$\|z_D\|_2\le s\le\|z_D\|_2+\sqrt\varepsilon$. Since
$\sqrt{z^TVz}\le\sqrt{\operatorname{tr}V}\le1+\operatorname{tr}V$,
$\|z_D\|_2\le1$, and $\sqrt\varepsilon\le1$, the point $(x,z,t,s)$ satisfies
every constraint of Equation~\ref{eq:hard-qcqp-empty}.

The repair lowers the objective by at most
$(20+2\operatorname{tr}V)\sqrt\varepsilon$. Indeed, the solver's constraints
give $\hat t,\hat s\ge-\varepsilon$,
$\sqrt{\hat z^TV\hat z}\le|\hat t|+\sqrt\varepsilon$, and
$\|\hat z_D\|_2\le|\hat s|+\sqrt\varepsilon$. Since
$\|V^{1/2}v\|_2\le\sqrt{\operatorname{tr}V}\,\|v\|_2$ for every $v$,
the triangle inequality gives
\begin{equation*}
t-\hat t\le\bigl(4+2\sqrt{\operatorname{tr}V}\bigr)\sqrt\varepsilon,
\qquad
s-\hat s\le6\sqrt\varepsilon.
\end{equation*}
Also $\|\hat z\|_2\le2$, so
$|\hat z_0-z_0|+|\hat z_X^T\hat x-z_X^Tx|\le8\sqrt\varepsilon$. Adding these
bounds proves the claim. Write $\ell$ for the objective value of
$(x,z,t,s)$, which we compute exactly. Then $\ell$ is at least the
optimum minus $\varepsilon+(20+2\operatorname{tr}V)\sqrt\varepsilon\le\eta$,
so
\begin{equation*}
\max_{x'\in X}\operatorname{dist}(E_V,S_{x'})\le\ell+\eta.
\end{equation*}
On the other hand, $t\ge\sqrt{z^TVz}$ and $s\ge\|z_D\|_2$, so
Equation~\ref{eq:hard-empty-distance} gives
$\ell\le\operatorname{dist}(E_V,S_x)$.

If $\ell>0$, then $S_x$ and $E_V$ are disjoint, so no $y\in D$ has
$p_V(x,y)\le1$, and we certify that $\hat K_V(x)=\varnothing$.
Otherwise, every arm $x'$ satisfies
$\operatorname{dist}(E_V,S_{x'})\le\eta$. Since both sets are
compact, there are $y\in D$ and $u\in E_V$ with
$\|w(x',y)-u\|_2\le\eta$. Since $V$ starts at $n^{-1}I_q$ and
each update adds a positive semidefinite matrix, we have
$V\succeq n^{-1}I_q$. Hence $\|v\|_{V^{-1}}\le\sqrt n\,\|v\|_2$ for
every $v$, and the triangle inequality gives
$\|w(x',y)\|_{V^{-1}}\le\|u\|_{V^{-1}}+\eta\sqrt n\le1+\eta\sqrt n$.
Thus $y\in\hat K_{V,\eta}^{\mathrm{out}}(x')$, and we certify
the second case.

For the running time, the program in Equation~\ref{eq:hard-qcqp-empty}
has $2d_X+d_D+3$ variables and nine constraints. Its coefficients are
the entries of $V$ and numbers computed from $\operatorname{tr}V$, so
their total bit length is polynomial in $q$ and the bit length of $V$.
By Theorem~\ref{thm:bienstock}, the solver runs in time polynomial in
this input size and
$\log(1/\varepsilon)=2\log((21+2\operatorname{tr}V)/\eta)$, which
is polynomial in $\log(1/\eta)$, $q$, and the bit length of $V$,
since $\operatorname{tr}V$ is a sum of $q$ entries of $V$. The
repair and the computation of $\ell$ use rational arithmetic and
square roots to precision $\sqrt\varepsilon$. They take time polynomial in
the same quantities and in the bit length of the solver's output,
which is bounded by the solver's running time.
\end{proof}

We now use Lemma~\ref{lem:hard-approx-empty} to choose the arm, with
a value of $\eta$ that we pick later. At the start of each block, we
run its algorithm. If it certifies that $\hat K_V(x)=\varnothing$, we
play the arm $x$ that it returns, as
Algorithm~\ref{alg:hard-ellipsoid-warmup} would. Otherwise, it
certifies that $\hat K_{V,\eta}^{\mathrm{out}}(x)$ is nonempty for every
arm, and we solve problem~2 with these sets in place of $\hat K_V(x)$.
Recall, however, that Lemma~\ref{lem:hard-dual-forms}(b) writes the
score as a supremum over all $z\in\mathbb R^q$, which need not be
attained at a bounded $z$. We fix this with the following lemma.

\begin{lemma}\label{lem:hard-dual-bound}
Let $V$ be a positive definite matrix with $V\succeq n^{-1}I_q$, and let
$x\in X$. Suppose that some outcome $y_0\in D$ satisfies
$\|w(x,y_0)\|_{V^{-1}}\le1-\mu$ for some $\mu\in(0,1]$. Then the
supremum in Lemma~\ref{lem:hard-dual-forms}(b) is attained at some $z$
with $\|z\|_2\le2\sqrt n\,\|b\|_2/\mu$.
\end{lemma}

\begin{proof}
The outcome $y_0$ belongs to $\hat K_V(x)$, so
Lemma~\ref{lem:hard-dual-forms}(b) applies. Write
$w_0\coloneqq w(x,y_0)$, and let
\begin{equation*}
F(z)\coloneqq z_0+z_X^Tx-\sqrt{z^TVz}-\|b+z_D\|_2
\end{equation*}
be the expression in its supremum. Since $\|y_0\|_2\le1$, we have
$-\|b+z_D\|_2\le(b+z_D)^Ty_0$, so
$F(z)\le b^Ty_0+z^Tw_0-\sqrt{z^TVz}$. Cauchy--Schwarz in the $V$ norm
gives $z^Tw_0\le\|w_0\|_{V^{-1}}\sqrt{z^TVz}\le(1-\mu)\sqrt{z^TVz}$.
Hence
\begin{equation*}
F(z)\le\|b\|_2-\mu\sqrt{z^TVz}.
\end{equation*}
Since $F(0)=-\|b\|_2$, every $z$ with $\sqrt{z^TVz}>2\|b\|_2/\mu$ gives
a smaller value than $z=0$. The set of $z$ with
$\sqrt{z^TVz}\le2\|b\|_2/\mu$ is compact and $F$ is continuous, so $F$
attains its supremum on this set. Finally, $V\succeq n^{-1}I_q$ gives
$\|z\|_2\le\sqrt n\,\sqrt{z^TVz}\le2\sqrt n\,\|b\|_2/\mu$.
\end{proof}

Lemma~\ref{lem:hard-dual-bound} needs an outcome strictly inside the
ellipsoid, with a known margin. The sets
$\hat K_{V,\eta}^{\mathrm{out}}(x)$ do not provide one, since the
outcomes that Lemma~\ref{lem:hard-approx-empty} certifies may lie on
the boundary of their ellipsoid. We therefore enlarge the ellipsoid a
little further and maximize $H_{4V}$, the score of
Algorithm~\ref{alg:hard-ellipsoid-warmup} for the matrix $4V$. Its sets
$\hat K_{4V}(x)=\{y\in D:p_V(x,y)\le4\}$ use the ellipsoid
$E_{4V}=2E_V$, and when $\eta\sqrt n\le1/2$, they contain the sets
$\hat K_{V,\eta}^{\mathrm{out}}(x)$ with room to spare. Like the program
for problem~1, the program for this score can only be solved
approximately. The next lemma shows that we can still find an arm
whose score is within any $\kappa\in(0,1)$ of the maximum.

\begin{lemma}\label{lem:hard-approx-score}
Let $V$ be a matrix with rational entries maintained by the learner,
and let $\eta,\kappa\in(0,1)$ be rational with $\eta\sqrt n\le1/2$.
Suppose that $\hat K_{V,\eta}^{\mathrm{out}}(x)$ is nonempty for every
arm $x$. There is an algorithm that computes a rational arm $x\in X$ with
$H_{4V}(x)\ge\max_{x'\in X}H_{4V}(x')-\kappa$. Its running time is
polynomial in $d_X$, $d_D$, $\log n$, $\log(1/\kappa)$, and the bit
lengths of $V$, $a$, and $b$.
\end{lemma}

\begin{proof}
We first write the problem as a program that fits
Theorem~\ref{thm:bienstock}. Fix an arm $x$. Since
$\hat K_{V,\eta}^{\mathrm{out}}(x)$ is nonempty, some outcome
$y_0\in D$ satisfies
$\|w(x,y_0)\|_{V^{-1}}\le1+\eta\sqrt n\le3/2$. Since
$\|\cdot\|_{(4V)^{-1}}=\frac12\|\cdot\|_{V^{-1}}$, this gives
$\|w(x,y_0)\|_{(4V)^{-1}}\le3/4$. So $\hat K_{4V}(x)$ is nonempty, and
Lemma~\ref{lem:hard-dual-bound}, applied to the matrix
$4V\succeq n^{-1}I_q$ with $\mu=1/4$, shows that the supremum in
Lemma~\ref{lem:hard-dual-forms}(b) for $4V$ is attained at some $z$ with
$\|z\|_2\le8\sqrt n\,\|b\|_2\le M$, where $M\coloneqq8n(1+\|b\|_1)$.
Hence the assumption of Lemma~\ref{lem:hard-qcqp-score} holds for the
matrix $4V$ with $R=M$, and $\max_{x\in X}H_{4V}(x)$ is the optimal
value of the program in Equation~\ref{eq:hard-qcqp-score} with $V$
replaced by $4V$ and $R=M$.

A smaller $\kappa$ gives a stronger guarantee, so, as in the proof of
Lemma~\ref{lem:hard-approx-empty}, we may assume that $\kappa$ is a
power of $2$. We run the algorithm of Theorem~\ref{thm:bienstock} on
this program, using the tolerance
$\varepsilon\coloneqq(\kappa/(K+1))^2$, where
$K\coloneqq20+2\|a\|_1+2M+4\operatorname{tr}V$. We repair its output as
in the proof of Lemma~\ref{lem:hard-approx-empty}, now scaling
$\hat z$ into the ball of radius $M$. The repaired point is feasible,
and the computation in that proof shows that the repair lowers the
objective by at most $K\sqrt\varepsilon$, with three changes. The bound
$\|\hat z\|_2\le M+1$ replaces $\|\hat z\|_2\le2$, so the terms $z_0$
and $z_X^Tx$ change by at most $(2M+6)\sqrt\varepsilon$ in place of
$8\sqrt\varepsilon$. The matrix $4V$ replaces $V$, which turns
$\sqrt{\operatorname{tr}V}$ into $2\sqrt{\operatorname{tr}V}$ in the
bound on $t-\hat t$. Finally, the term $a^Tx$ changes by at most
$2\|a\|_1\sqrt\varepsilon$.

Let $(x,z,t,s)$ be the repaired point. Since
$t\ge\sqrt{z^T(4V)z}$ and $s\ge\|b+z_D\|_2$, its objective value is at
most $a^Tx+F_x(z)$, where
$F_x(z)\coloneqq z_0+z_X^Tx-2\sqrt{z^TVz}-\|b+z_D\|_2$ is the
expression in Lemma~\ref{lem:hard-dual-forms}(b) for the matrix $4V$.
The proof of
Lemma~\ref{lem:hard-dual-forms}(b) shows that
$a^Tx+F_x(z)\le H_{4V}(x)$ for every $z$. Hence
\begin{equation*}
H_{4V}(x)
\ge\max_{x'\in X}H_{4V}(x')-\varepsilon-K\sqrt\varepsilon
\ge\max_{x'\in X}H_{4V}(x')-\kappa.
\end{equation*}
The running time follows as in the proof of
Lemma~\ref{lem:hard-approx-empty}. The coefficients now also include
$a$, $b$, and $M$, and
$\log(1/\varepsilon)=2\log((K+1)/\kappa)$ is polynomial in
$\log(1/\kappa)$, $\log n$, $q$, and the bit lengths of $V$, $a$,
and $b$.
\end{proof}

It remains to pick $\eta$ and $\kappa$ so that these approximations
do not change the regret rate. We take
\begin{equation*}
\eta\coloneqq\frac{1}{2\lceil\sqrt n\rceil},
\qquad
\kappa\coloneqq\frac1T.
\end{equation*}
The choice of $\eta$ gives $\eta\sqrt n\le1/2$, so in the second case
of Lemma~\ref{lem:hard-approx-empty} the hypotheses of
Lemma~\ref{lem:hard-approx-score} hold, and we choose the arm with the
algorithm of that lemma. The choice of $\kappa$ puts the
score of this arm within $1/T$ of the maximum of $H_{4V}$, which adds
at most $1$ to the regret over all $T$ rounds, as we show next. Both
choices keep the running times polynomial, since
$\log(1/\eta)=\log(2\lceil\sqrt n\rceil)$ and $\log(1/\kappa)=\log T$.

We now bound the regret, following the proof of
Theorem~\ref{thm:hard-ellipsoid-regret}. That proof uses the rule for
choosing arms only to bound the regret of blocks that do not update
$V$. Its other steps, namely the event of
Lemma~\ref{lem:interval-noise}, the bound $J\le J_{\max}$ on the number
of updates, and Equation~\ref{eq:hard-subspace-certificate}, depend
only on the update rule, which we have not changed. A full block
played at an arm with $\hat K_V(x)=\varnothing$ always updates $V$,
since its average lies in $D$ and therefore has $p_V>1$. So every full
block that does not update $V$ was chosen by
Lemma~\ref{lem:hard-approx-score}. In such a block, the average
$\bar y$ has $p_V(x,\bar y)\le1$, so $\bar y\in\hat K_{4V}(x)$. Every
$y\in\hat K_{4V}(x')$ has $p_V(x',y)\le4$, so
Equation~\ref{eq:hard-subspace-certificate} and
Lemma~\ref{lem:subspace-reward-bound} give
$r(x',y)\ge v^\star(x')-2GB_T/\sqrt n$. Taking $x'$ to be an arm with
$v^\star(x')=V^\star$, we get
\begin{equation*}
r(x,\bar y)
\ge H_{4V}(x)
\ge\max_{x'\in X}H_{4V}(x')-\kappa
\ge V^\star-\frac{2GB_T}{\sqrt n}-\kappa.
\end{equation*}
Since $r$ is linear in $y$, the reward of the block is $n\,r(x,\bar y)$.
There are at most $J_{\max}$ blocks that update $V$, each costing at
most $C_rn$, and the final incomplete block costs at most another
$C_rn$. Hence, with probability at least $1-\delta$,
\begin{equation*}
TV^\star-\sum_{t=1}^Tr(x_t,y_t)
\le C_rn(J_{\max}+1)+\frac{2GB_TT}{\sqrt n}+T\kappa.
\end{equation*}
This is the bound of Theorem~\ref{thm:hard-ellipsoid-regret} with two
changes. The middle term doubles, because we score arms with $2E_V$
instead of $E_V$, and the new term $T\kappa$ equals $1$. The expected
regret is at most this bound plus $C_rT\delta$, so taking
$n=\lceil T^{2/3}\rceil$ and $\delta=(T+1)^{-2}$, as in that theorem,
gives the same $\tilde O(T^{2/3})$ rate.

Finally, we bound the running time of the whole learner. We assume
that the observations are rational, with bit length polynomial in
$|\phi|$ and $T$. Every matrix maintained by the learner
satisfies $V\preceq(n^{-1}+3T)I_q$, since $V$ starts at $n^{-1}I_q$ and
each of the at most $T$ updates adds a positive semidefinite matrix of
operator norm at most three. So $\operatorname{tr}V\le q(n^{-1}+3T)$,
and the grid spacings used in Lemmas~\ref{lem:hard-approx-empty}
and~\ref{lem:hard-approx-score} are powers of $2$ whose exponents are
polynomial in $q$, $\log T$, and the bit lengths of $a$ and $b$. Hence
every stored matrix has rational entries of polynomial bit length, and
matrix inversion and the test $p_V>1$ can be performed exactly in
polynomial time. Each block calls the algorithm of
Lemma~\ref{lem:hard-approx-empty} once and, in its second case, the
algorithm of Lemma~\ref{lem:hard-approx-score} once, and there are at
most $\lceil T/n\rceil$ blocks. With our choices of $\eta$ and
$\kappa$, each call takes time polynomial in $|\phi|$ and $T$, as
required.
}

\section{Regret analysis of the $\sqrt{T}$ learner}\label{app:soft-square-root-analysis}

\draft{

We give the precise statement and proof of Theorem~\ref{thm:soft-square-root}.
We use the coordinates and notation from Section~\ref{sec:soft-square-root},
including $\lambda=\lceil\sqrt T\rceil$, and assume that
Algorithm~\ref{alg:soft-square-root} maximizes $U_{V,\lambda}$ exactly.
Appendix~\ref{app:soft-computation} shows how to maximize it approximately
in polynomial time without changing the regret rate.

For the analysis, fix $\delta\in(0,1)$, use $h_T$ and $G$ from
Appendix~\ref{app:hard-ellipsoid-analysis}, and define
\begin{equation*}
J_{\max}'\coloneqq\left\lceil q\log_2\left(1+\frac{3T}{q}\right)\right\rceil,
\qquad
{B_T'}^2\coloneqq1+J_{\max}'h_T^2 .
\end{equation*}
These play the roles of $J_{\max}$ and $B_T$ in
Appendix~\ref{app:hard-ellipsoid-analysis}. The horizon $T$ takes the
place of the block length $n$ because this learner starts from $V=I_q$
and weights each block by its length, so the trace of $V$ can grow to
$q+3T$. We use the concentration bound from Lemma~\ref{lem:interval-noise}
and adapt the proof of Lemma~\ref{lem:subspace-distance-bound} to the new
matrix update.

\begin{theorem}\label{thm:soft-square-root-regret}
Algorithm~\ref{alg:soft-square-root} satisfies, with probability at least $1-\delta$,
\begin{equation*}
TV^\star-\sum_{t=1}^Tr(x_t,y_t)\le4\lambda\left(J_{\max}'+1\right)+\frac{G^2{B_T'}^2T}{4\lambda}.
\end{equation*}
Its expected regret is at most the same expression plus $C_rT\delta$, and taking $\delta=(T+1)^{-2}$ gives
\begin{equation*}
R_T\le\tilde O\left(\left[q+\left(1+S^{-1}\right)^2\|b\|_2^2\,q\,d_D\right]\sqrt T\right)+O\left(\frac{C_r}{T}\right).
\end{equation*}
\end{theorem}

Note that the learner does not need to know $S$ or the quantities
  $J_{\text{\normalfont max}}^{\prime} , h_{T} , B_{T}^{\prime} , G$.

\begin{proof}
As in the proof of Theorem~\ref{thm:hard-ellipsoid-regret}, each
completed block at least doubles the determinant of $V$. The stopping
rule gives
\begin{equation*}
\det\left(V+n\overline w\,\overline w^T\right)=\det(V)\left(1+np_V(x,\overline y)\right)\ge2\det(V).
\end{equation*}
Since $\|\overline w\|_2^2\le3$ and the block lengths add up to at most
$T$, the final trace is at most $q+3T$. Thus, if we complete $J$ blocks,
the arithmetic--geometric mean inequality on the eigenvalues of $V$ gives
$2^J\le\det V\le(1+3T/q)^q$, so $J\le J_{\max}'$.

We also need to bound how far $np_V(x,\overline y_n)$ can jump above
1 on the last observation. For $n\ge2$, convexity gives
\begin{equation*}
np_V(x,\overline y_n)\le(n-1)p_V(x,\overline y_{n-1})+p_V(x,y_n)<1+3=4 .
\end{equation*}
The first term is less than one because we did not stop at $n-1$.
The second is at most three because $V\succeq I_q$ and
$p_V(x,y_n)\le\|w(x,y_n)\|_2^2\le3$. If we stop at $n=1$, the bound is
at most three. If the horizon ends before $np_V(x,\overline y_n)$
reaches $1$, it is less than one. Thus every block satisfies
\begin{equation}
np_V(x,\overline y)\le4 .\label{eq:soft-block-cost}
\end{equation}

We now adapt the proof of Lemma~\ref{lem:subspace-distance-bound}.
Number the completed blocks $j=1,\ldots,J$, and let $n_j$ be the length
of block $j$. Write $x_j$ for its arm, $\overline y_j$ for its average
observation, $\overline m_j$ for the average over the block of
$m_t=\mathbb E_{y\sim P_t}[y]$, the expected value of the distribution
chosen on round $t$, and $\overline w_j\coloneqq w(x_j,\overline y_j)$.
After $J$ completed blocks, the matrix has the form
\begin{equation*}
V=I_q+\sum_{j=1}^Jn_j\overline w_j\overline w_j^T .
\end{equation*}
Fix $u\in\mathcal L^\perp$ with $\|u\|_2\le1$. Since the arm is held
fixed within each block, $w(x_j,\overline m_j)\in\mathcal L$, so
$u^T\overline w_j=u^T(0,0,\overline y_j-\overline m_j)$. By
Lemma~\ref{lem:interval-noise}, with probability at least $1-\delta$,
the bound $(u^T\overline w_j)^2\le h_T^2/n_j$ holds simultaneously for
all completed blocks and all such $u$. On this event,
\begin{equation*}
u^TVu\le1+\sum_{j=1}^Jn_j\cdot\frac{h_T^2}{n_j}=1+Jh_T^2\le{B_T'}^2 .
\end{equation*}
This is where the weight $n_j$ matters. It cancels the $1/n_j$ in the
squared noise bound, so each block contributes at most $h_T^2$, however
long it is. Applying Cauchy--Schwarz and taking the supremum over these
$u$ gives
\begin{equation}
\operatorname{dist}(w,\mathcal L)\le B_T'\|w\|_{V^{-1}} .\label{eq:soft-subspace-certificate}
\end{equation}

Combining Lemma~\ref{lem:subspace-reward-bound} with
Equation~\ref{eq:soft-subspace-certificate} gives, for every $x\in X$
and $y\in D$,
\begin{equation}
v^\star(x)\le r(x,y)+GB_T'\sqrt{p_V(x,y)} .\label{eq:soft-reward-certificate}
\end{equation}
This is the claim in the proof of Lemma~\ref{lem:soft-outcome-regret},
with $c=GB_T'$. That proof then shows that, on the same event, the arm
$x$ of every block satisfies
\begin{equation*}
V^\star-r(x,y)\le\lambda p_V(x,y)+\frac{G^2{B_T'}^2}{4\lambda}
\qquad\text{for every }y\in D.
\end{equation*}
We apply this to the block's average $\overline y$, which lies in $D$
because $D$ is convex. Since the reward is linear, the block's reward is
$nr(x,\overline y)$, and Equation~\ref{eq:soft-block-cost} bounds its
regret by
\begin{equation*}
n\left(V^\star-r(x,\overline y)\right)\le\lambda np_V(x,\overline y)+\frac{nG^2{B_T'}^2}{4\lambda}\le4\lambda+\frac{nG^2{B_T'}^2}{4\lambda}.
\end{equation*}
There are at most $J_{\max}'+1$ blocks, namely the completed ones and
possibly a final block cut short by the horizon, and their lengths add
up to $T$. Adding the bounds proves the high-probability claim. Outside
this event, regret is still at most $C_rT$, so taking expectations adds
at most $C_rT\delta$.

Finally, with $\lambda=\lceil\sqrt T\rceil$, the first term of the bound
is $\tilde O(q\sqrt T)$. Since $G^2=3(1+S^{-1})^2\|b\|_2^2$ and
${B_T'}^2=\tilde O(q\,d_D)$, the second is
$\tilde O((1+S^{-1})^2\|b\|_2^2\,q\,d_D\sqrt T)$. With
$\delta=(T+1)^{-2}$, we have $C_rT\delta\le C_r/T$.
\end{proof}
}

\subsection{Computational tractability and finite-precision implementation}\label{app:soft-computation}

\draft{
Algorithm~\ref{alg:soft-square-root} has a single computational
problem, which is to find an arm maximizing $U_{V,\lambda}$. Unlike
Algorithm~\ref{alg:hard-ellipsoid-warmup}, it has no emptiness test. We
follow Appendix~\ref{app:hard-computation}. We first write the
maximization as a QCQP that fits Theorem~\ref{thm:bienstock}, and then
we explain how to use an approximate solution and show that this
preserves the $\tilde O(\sqrt T)$ rate.

Recall that $X$ and $D$ are unit balls, that $r(x,y)=a^Tx+b^Ty$, that
$w(x,y)=(1,x,y)$, and that $p_V(x,y)=\|w(x,y)\|_{V^{-1}}^2$. As in
Lemma~\ref{lem:hard-dual-forms}, we split a vector $z\in\mathbb R^q$ as
$z=(z_0,z_X,z_D)$, so that $z^Tw(x,y)=z_0+z_X^Tx+z_D^Ty$. Every matrix
maintained by the learner satisfies $V\succeq I_q$, since $V$ starts at
$I_q$ and each update adds a positive semidefinite matrix.

Lemma~\ref{lem:hard-dual-forms}(b) turned the minimum over outcomes in
$H_V$ into a maximum over a dual vector $z$. The next lemma does the
same for $U_{V,\lambda}$. The difference is where $z$ comes from. In
Lemma~\ref{lem:hard-dual-forms}(b), $z$ was the Lagrange multiplier that
tied $w(x,y)$ to the ellipsoid $E_V$. Here $y$ ranges over all of $D$,
and $z$ comes from writing the penalty $\lambda p_V(x,y)$ itself as a
maximum.

\begin{lemma}\label{lem:soft-dual-form}
Let $V\succeq I_q$ and $\lambda>0$. For every $x\in X$,
\begin{equation*}
U_{V,\lambda}(x)=a^Tx+\max_{\|z\|_2\le4\lambda}\left[z_0+z_X^Tx-\|b+z_D\|_2-\frac{z^TVz}{4\lambda}\right].
\end{equation*}
\end{lemma}

\begin{proof}
For every $w\in\mathbb R^q$,
\begin{equation*}
\lambda w^TV^{-1}w=\max_{z\in\mathbb R^q}\left[z^Tw-\frac{z^TVz}{4\lambda}\right].
\end{equation*}
Indeed, the expression in brackets is concave in $z$, and its gradient
$w-Vz/(2\lambda)$ vanishes at $z=2\lambda V^{-1}w$, where the expression
equals $2\lambda w^TV^{-1}w-\lambda w^TV^{-1}w$. For $w=w(x,y)$ with
$x\in X$ and $y\in D$, we have $\|w\|_2^2=1+\|x\|_2^2+\|y\|_2^2\le3$,
and $V\succeq I_q$ gives
$\|2\lambda V^{-1}w\|_2\le2\lambda\|w\|_2\le4\lambda$. So we may
restrict the maximum to the ball $\|z\|_2\le4\lambda$, and
\begin{equation*}
U_{V,\lambda}(x)=\min_{y\in D}\max_{\|z\|_2\le4\lambda}\left[r(x,y)+z^Tw(x,y)-\frac{z^TVz}{4\lambda}\right].
\end{equation*}
For fixed $x$, the expression in brackets is affine in $y$ and concave
in $z$, and both $D$ and the ball are compact and convex. Sion's minimax
theorem~\citep{sion1958minimax} therefore lets us exchange the minimum
and the maximum. For fixed $z$, the minimum over $y\in D$ is
$a^Tx+z_0+z_X^Tx-\|b+z_D\|_2$, by the same linear optimization over $D$
as in the proof of Lemma~\ref{lem:hard-dual-forms}. The maximum over
$z$ is attained because the resulting expression is continuous in $z$
and the ball is compact.
\end{proof}

Unlike the supremum in Lemma~\ref{lem:hard-dual-forms}(b), this maximum
is over a bounded set, so we do not need a margin argument such as
Lemma~\ref{lem:hard-dual-bound}. The penalty term $z^TVz/(4\lambda)$ is
already quadratic, so writing the maximization over arms as a QCQP only
needs one of the two substitutions from the proof of
Lemma~\ref{lem:hard-qcqp-empty}.

\begin{lemma}\label{lem:soft-qcqp}
Let $V\succeq I_q$ have rational entries, and let $\lambda\ge1$ be an
integer. Then $\max_{x\in X}U_{V,\lambda}(x)$ is the optimal value of a
QCQP in $x$, $z$, and one scalar variable, with six quadratic
constraints, that satisfies the hypotheses of
Theorem~\ref{thm:bienstock}. The arm of every optimal solution maximizes
$U_{V,\lambda}$.
\end{lemma}

\begin{proof}
By Lemma~\ref{lem:soft-dual-form},
\begin{equation*}
\max_{x\in X}U_{V,\lambda}(x)=\max_{x\in X,\ \|z\|_2\le4\lambda}\left[a^Tx+z_0+z_X^Tx-\|b+z_D\|_2-\frac{z^TVz}{4\lambda}\right].
\end{equation*}
As in the proof of Lemma~\ref{lem:hard-qcqp-empty}, we replace the norm
by a scalar variable $s$ with $s\ge0$ and $\|b+z_D\|_2^2\le s^2$, and
write $-s$ in its place. Let $W\coloneqq1+\|b\|_1+4\lambda$, which
exceeds the largest value of $\|b+z_D\|_2$ on the ball
$\|z\|_2\le4\lambda$ by at least one. We use this extra room when we
repair the solver's output below. This gives the QCQP
\begin{equation}\label{eq:soft-qcqp}
\begin{aligned}
\max_{\substack{x\in\mathbb R^{d_X},\,z\in\mathbb R^q\\s\in\mathbb R}}
&\quad a^Tx+z_0+z_X^Tx-s-\frac{z^TVz}{4\lambda}\\
\text{subject to}
&\quad \|x\|_2^2\le1,\qquad \|z\|_2^2\le16\lambda^2,\qquad \|b+z_D\|_2^2\le s^2,\\
&\quad 0\le s\le W,\qquad
\|x\|_2^2+\frac{\|z\|_2^2}{16\lambda^2}+\frac{s^2}{W^2}\le3.
\end{aligned}
\end{equation}
For the same reasons as in the proof of Lemma~\ref{lem:hard-qcqp-empty},
this program has the same optimal value as the maximization above, and
the arm of every optimal solution maximizes $U_{V,\lambda}$. It has six
quadratic constraints and a nonempty feasible set. Its last constraint
is implied by the others and defines a bounded ellipsoid in all the
variables, and its coefficients are rational because the entries of $V$,
$a$, and $b$ and the number $\lambda$ are. So it satisfies the
hypotheses of Theorem~\ref{thm:bienstock}.
\end{proof}

Theorem~\ref{thm:bienstock} only gives an approximate solution, allowing
errors in both the objective and the constraints. As
Lemma~\ref{lem:hard-approx-score} did for the score of
Algorithm~\ref{alg:hard-ellipsoid-warmup}, the next lemma shows that we
can still find an arm whose score is within any $\kappa\in(0,1)$ of the
maximum.

\begin{lemma}\label{lem:soft-approx-score}
Let $V\succeq I_q$ have rational entries, let $\lambda\ge1$ be an
integer, and let $\kappa\in(0,1)$ be rational. There is an algorithm
that computes an arm $x\in X$ with rational coordinates and
$U_{V,\lambda}(x)\ge\max_{x'\in X}U_{V,\lambda}(x')-\kappa$. Its running
time is polynomial in $d_X$, $d_D$, $\log\lambda$, $\log(1/\kappa)$, and
the bit lengths of $V$, $a$, and $b$.
\end{lemma}

\begin{proof}
As in the proof of Lemma~\ref{lem:hard-approx-score}, we may assume that
$\kappa$ is a power of $2$. Let
$K\coloneqq12+2\|a\|_1+8\lambda+6\operatorname{tr}V$. We run the
algorithm of Theorem~\ref{thm:bienstock} on the program in
Equation~\ref{eq:soft-qcqp} with tolerance
$\varepsilon\coloneqq(\kappa/(K+1))^2$. It returns a point
$(\hat x,\hat z,\hat s)$ with rational coordinates that violates each
constraint by at most $\varepsilon$ and whose objective value is at
least the optimum minus $\varepsilon$. We repair this point as in the
proof of Lemma~\ref{lem:hard-approx-empty}. We scale $\hat x$ into the
unit ball and $\hat z$ into the ball of radius $4\lambda$, round their
coordinates toward zero to the grid whose spacing is the largest power
of $2$ that is at most $\sqrt\varepsilon/q$, and choose a rational $s$
with $\|b+z_D\|_2\le s\le\|b+z_D\|_2+\sqrt\varepsilon$. This gives $x$
and $z$ with rational coordinates, $\|x-\hat x\|_2\le2\sqrt\varepsilon$,
and $\|z-\hat z\|_2\le2\sqrt\varepsilon$. Since
$\|b+z_D\|_2\le\|b\|_1+4\lambda$ and $\sqrt\varepsilon\le1$, the point
$(x,z,s)$ satisfies every constraint of Equation~\ref{eq:soft-qcqp}.

The repair lowers the objective by at most $K\sqrt\varepsilon$. The
terms $a^Tx$ and $z_0$ change by at most $2\|a\|_1\sqrt\varepsilon$ and
$2\sqrt\varepsilon$. Since $\|z\|_2\le4\lambda$ and $\|\hat x\|_2\le2$,
the term $z_X^Tx$ changes by at most $(8\lambda+4)\sqrt\varepsilon$. As
in the proof of Lemma~\ref{lem:hard-approx-empty},
$s-\hat s\le6\sqrt\varepsilon$. Finally, $\|\hat z\|_2\le4\lambda+1$ and
the largest eigenvalue of $V$ is at most $\operatorname{tr}V$, so
\begin{equation*}
\frac{\left|z^TVz-\hat z^TV\hat z\right|}{4\lambda}
=\frac{\left|(z-\hat z)^TV(z+\hat z)\right|}{4\lambda}
\le\frac{2\sqrt\varepsilon\operatorname{tr}V\,(8\lambda+1)}{4\lambda}
\le6\operatorname{tr}V\sqrt\varepsilon .
\end{equation*}
Adding these bounds proves the claim.

Write $\ell$ for the objective value of $(x,z,s)$, which we compute
exactly. By Lemma~\ref{lem:soft-qcqp}, the optimum of the program is
$\max_{x'\in X}U_{V,\lambda}(x')$, so
\begin{equation*}
\ell\ge\max_{x'\in X}U_{V,\lambda}(x')-\varepsilon-K\sqrt\varepsilon\ge\max_{x'\in X}U_{V,\lambda}(x')-\kappa .
\end{equation*}
On the other hand, $s\ge\|b+z_D\|_2$ and $\|z\|_2\le4\lambda$, so
Lemma~\ref{lem:soft-dual-form} gives $\ell\le U_{V,\lambda}(x)$.

For the running time, the program has $2d_X+d_D+2$ variables and six
constraints. Its coefficients are the entries of $V$, $a$, and $b$ and
numbers computed from $\lambda$ and $\|b\|_1$. By
Theorem~\ref{thm:bienstock}, the solver runs in time polynomial in this
input size and $\log(1/\varepsilon)=2\log((K+1)/\kappa)$, which is
polynomial in $\log(1/\kappa)$, $\log\lambda$, $q$, and the bit lengths
of $V$ and $a$. As in the proof of Lemma~\ref{lem:hard-approx-empty},
the repair and the computation of $\ell$ take time polynomial in the
same quantities and in the bit length of the solver's output.
\end{proof}

At the start of each block, we choose the arm with the algorithm of
Lemma~\ref{lem:soft-approx-score}, with $\kappa\coloneqq1/T$. This
choice makes $\log(1/\kappa)=\log T$ in the running time and, as we show
next, adds at most $1$ to the regret.

The proof of Theorem~\ref{thm:soft-square-root-regret} uses the rule for
choosing arms only through the proof of
Lemma~\ref{lem:soft-outcome-regret}, and that proof uses it only in the
step $U_{V,\lambda}(x)\ge V^\star-c^2/(4\lambda)$. For an arm whose
score is within $\kappa$ of the maximum, the same step gives
$U_{V,\lambda}(x)\ge V^\star-c^2/(4\lambda)-\kappa$. The other steps,
namely the event of Lemma~\ref{lem:interval-noise}, the bound
$J\le J_{\max}'$, and Equation~\ref{eq:soft-block-cost}, depend only on
the rules for updating $V$ and ending a block, which we have not
changed. So the regret of a block of length $n$ grows by at most
$n\kappa$, and the high-probability bound of
Theorem~\ref{thm:soft-square-root-regret} grows by $T\kappa=1$. The
expected regret therefore has the same rate.

Finally, we bound the running time of the whole learner. As in
Appendix~\ref{app:hard-computation}, we assume that the observations
are rational, with bit length polynomial in $|\phi|$ and $T$.
Every matrix maintained by the learner satisfies
$I_q\preceq V\preceq(1+3T)I_q$, since $V$ starts at $I_q$, a completed
block of length $n$ adds the positive semidefinite matrix
$n\overline w\,\overline w^T$ of operator norm at most $3n$, and the
block lengths add up to at most $T$. So $\operatorname{tr}V\le q(1+3T)$,
and the grid spacings used in Lemma~\ref{lem:soft-approx-score} are
powers of $2$ whose exponents are polynomial in $q$, $\log T$, and the
bit length of $a$. Hence every stored matrix has rational entries of
polynomial bit length, and matrix inversion and the test
$np_V(x,\overline y_n)\ge1$ after each observation can be performed
exactly in polynomial time. Each block calls the algorithm of
Lemma~\ref{lem:soft-approx-score} once, and there are at most
$J_{\max}'+1$ blocks. With $\lambda=\lceil\sqrt T\rceil$ and
$\kappa=1/T$, each call takes time polynomial in $|\phi|$ and
$T$. Together with the regret bound above, this proves
Theorem~\ref{thm:efficient-upper-bound}.
}

\section{NP-hardness results}\label{app:np-hardness-proofs}

In this section we show that several natural variations of the linear robust bandits problem are NP-hard. Thus in some sense, the problem is
the hardest problem that's still tractable.

\subsection{Planning from learning}\label{app:hardness-planning-from-learning}

We first show that efficient planning can be reduced to efficient learning.
Most proofs below show that planning is already NP-hard. Due to the reduction
from planning to learning, this shows that learning is also NP-hard.

\begin{lemma}\label{lem:planning-from-learning}
  {
  Fix \draft{a robust bandits instance $\mathcal{M}_\phi=(X,D,H,r)$ from a
  parameterized family} and a known hypothesis
  $h\in H$ whose description has encoding length $|h|$.
  Suppose that, for every arm $x$ and every $\eta\in(0,1)$, there is a
  distribution $Q_{\eta} \left( x \right)\in h \left( x \right)$ satisfying
  $\mathbb{E}_{y\sim Q_{\eta} \left( x \right)} \left[ r \left( x , y \right) \right] \le v_{h} \left( x \right) + \eta$,
  from which one can sample in time polynomial in $|\phi|$, $|h|$,
  the encoding length of $x$, and $1/\eta$.
  Suppose also that, for some fixed $\alpha > 0$ and polynomial $P$, a
  polynomial-time learner guarantees
  }
  {$R_{T} \le P \left( | \phi | \right) T^{1 - \alpha}$}.
  {Then, for every $\varepsilon \in \left( 0 , 1 \right)$, one can compute in
  time $\operatorname{poly} \left( | \phi | , |h|, \frac{1}{\varepsilon} \right)$ an arm $\hat{x}$ such
  that with probability at least 2/3, $V_{h} - v_{h} \left( \hat{x} \right) \le \varepsilon$.
  }
\end{lemma}

\draft{
\begin{proof}
Fix $\varepsilon \in \left( 0 , 1 \right)$, set $\eta \coloneqq \frac{\varepsilon}{6}$, and
simulate the learner on $\mathcal{M}_\phi$ for

\begin{equation*}
T \coloneqq \left\lceil \operatorname*{max} \left( 1 , {\left( 6 \frac{P \left( | \phi | \right)}{\varepsilon} \right)}^{\frac{1}{\alpha}} \right) \right\rceil
\end{equation*}

rounds, with true hypothesis $h$. Whenever the learner plays $x_{t}$, let
the adversary choose $P_{t} \coloneqq Q_{\eta} \left( x_{t} \right)$, and draw the outcome $y_{t}$
shown to the learner using the sampler. This is allowed because
$Q_{\eta} \left( x_{t} \right) \in h \left( x_{t} \right)$. Finally, choose $\hat{t}$ uniformly from
$\left\{ 1 , \ldots , T \right\}$ and return $\hat{x} \coloneqq x_{\hat{t}}$.

The guarantee on $Q_{\eta}$ gives
$V_{h} - v_{h} \left( x_{t} \right) \le V_{h} - \mathbb{E}_{y \sim P_{t}} \left[ r \left( x_{t} , y \right) \right] + \eta$
in every round. The expected reward of round $t$ is
$\mathbb{E} \left[ r \left( x_{t} , y_{t} \right) \right] = \mathbb{E} \left[ \mathbb{E}_{y \sim P_{t}} \left[ r \left( x_{t} , y \right) \right] \right]$,
so averaging over $\hat{t}$ and using the definition of $R_{T}$ for this
adversary policy gives

\begin{equation*}
\mathbb{E} \left[ V_{h} - v_{h} \left( \hat{x} \right) \right] = \frac{1}{T} \sum_{t = 1}^{T} \mathbb{E} \left[ V_{h} - v_{h} \left( x_{t} \right) \right] \le \frac{R_{T}}{T} + \eta \le P \left( | \phi | \right) T^{- \alpha} + \eta \le \frac{\varepsilon}{3} .
\end{equation*}

The gap $V_{h} - v_{h} \left( \hat{x} \right)$ is nonnegative, so Markov's inequality shows
that it exceeds $\varepsilon$ with probability at most $\frac{1}{3}$.
The horizon and the learner's output lengths are polynomial in
$|\phi|$ and $1/\varepsilon$, since $\alpha$ is fixed. The assumed
sampler therefore makes the simulation polynomial in
$|\phi|$, $|h|$, and $1/\varepsilon$.
\end{proof}
}

{In all our applications, the reward is linear and the hypothesis
\draft{constrains only the expected outcome. That is,}
$h(x)=\{P\in\Delta D:\mathbb{E}_{y\sim P}[y]\in K_h(x)\}$ for some set
$K_h(x)\subseteq D$. Then it suffices to compute an outcome $y\in K_h(x)$
with $r(x,y)\le v_h(x)+\eta$, since the point mass at $y$ lies in
$h(x)$ and has expected reward $r(x,y)$.}

Below we show that the ``smoothness'' of $X$ and $D$ are crucial for efficiency. It is easy to land in the NP-hard territory if we make either $X$ or $D$ a convex shape with edges and corners.

\subsection{D = simplex, X = ball}\label{app:hardness-simplex}

First we show that if, instead of a Euclidean ball, we allow $D$ to be a simplex, then even known hypothesis planning becomes NP-hard. Then, using the planning to learning reduction (Lemma~\ref{lem:planning-from-learning}), we conclude that learning is NP-hard.

\draft{We take $D$ to be the probability simplex
$\{y\in\mathbb{R}^{d_D}:y\ge0,\ \sum_{i}y_i=1\}$, which is determined by $d_D$.
So the parameters $\phi$ are those of
Section~\ref{sec:linear-imprecise-bandits} without the centre and radius
of $D$.}

\begin{theorem}
  {Consider the setting of Section~\ref{sec:linear-imprecise-bandits}
  with $D$ a simplex instead of a ball. Suppose that, for some fixed
  $\alpha>0$ and polynomial $P$, one randomized polynomial-time learner
  guarantees $R_T\le P(|\phi|)T^{1-\alpha}$ for \draft{all parameters
  $\phi$ and every true hypothesis $h^{\star}\in H$}. Then
  $\text{\normalfont NP}\subseteq\text{\normalfont BPP}$.
  If a deterministic learner achieves the same guarantee, then
  $\text{\normalfont P}=\text{\normalfont NP}$.}

 \label{thm:additive-simplex-np-hardness}

\end{theorem}

\begin{proof}
We reduce from MAX-CUT\draft{, which asks, given a graph $G$ and an integer $k$, whether some partition of the vertices into two sides has at least $k$ edges between the sides. It is NP-complete, even for unweighted graphs~\citep{karp1972reducibility,garey1976simplified}}. Let $\left( G , k \right)$ be an
instance, where $G = \left( V , E \right)$ has $n$ vertices and $m \ge 1$ edges and
$k \in \left\{ 1 , \ldots , m \right\}$. Orient the edges arbitrarily and let
$B_{G} \in \mathbb{R}^{m \times n}$ be the resulting edge-vertex incidence matrix, i.e., if
$e = \left\{ i , j \right\}$ is oriented from $j$ to $i$, then row $e$ has $1$ in column $i$,
$- 1$ in column $j$, and zeros elsewhere. Put

\begin{equation*}
X \coloneqq \left\{ x \in \mathbb{R}^{m} : {\left\Vert x \right\Vert}_{2} \le 1 \right\} .
\end{equation*}

Write an outcome as $y = \left( p , q , s \right) \in \mathbb{R}^{n} \times \mathbb{R}^{n} \times \mathbb{R}$ and take

\begin{equation*}
D \coloneqq \left\{ \left( p , q , s \right) : p_{i} \ge 0 , q_{i} \ge 0 \text{\normalfont  for } i = 1 , \ldots , n , s \ge 0 , \sum_{i = 1}^{n} \left( p_{i} + q_{i} \right) + s = 1 \right\} .
\end{equation*}

\draft{Take $d_C=n$, the cutoff $S\coloneqq1/(1+4\lceil\sqrt n\rceil)$, and
$C_{\max}\coloneqq1$, and let $H$ be the hypothesis class of
Section~\ref{sec:linear-imprecise-bandits} with these sets $X$ and $D$. For a scale $\gamma>0$ to be chosen later, define the true
hypothesis $h_G$ by}

\begin{equation*}
\draft{K_{h_G} \left( x \right) \coloneqq \left\{ \left( p , q , s \right) \in D : p - q = \gamma B_{G}^{T} x \right\} .}
\end{equation*}

\draft{Its coefficients are
$B_{h_G}=-\gamma B_G^T$, $C_{h_G}=(I_n\ -I_n\ 0)$, and $d_{h_G}=0$,
as in Equation~\ref{eq:setting-compatible-set}.}

Let the reward be

\begin{equation*}
r \left( x , \left( p , q , s \right) \right) \coloneqq \sum_{i = 1}^{n} \left( p_{i} + q_{i} \right) .
\end{equation*}

\draft{This reward takes values in $[0,1]$ on $D$, so $C_r\le C_{\max}$.
For fixed $n,m$, the parameters $\phi$, and hence the sets $X$ and $D$,
the reward, and the class $H$, do not depend on the graph. Only the true hypothesis $h_G$ does.
We next show that $(G,k)$ is a yes-instance of MAX-CUT if and only if
there is an arm $x\in X$ satisfying}

\begin{equation*}
v_{h_G} \left( x \right) \ge 2 \gamma \sqrt{k} .
\end{equation*}

A vector $\sigma \in {\left\{ - 1 , 1 \right\}}^{n}$ represents a partition of the vertices of $G$
into the two sides $\left\{ i : \sigma_{i} = 1 \right\}$ and $\left\{ i : \sigma_{i} = - 1 \right\}$. Let
$c_{G} \left( \sigma \right)$ denote the number of edges between these sides.

\begin{lemma}\label{lem:simplex-cut-norm}
  For every $\sigma \in {\left\{ - 1 , 1 \right\}}^{n}$,

\begin{equation*}
{\left\Vert B_{G} \sigma \right\Vert}_{2} = 2 \sqrt{c_{G} \left( \sigma \right)} .
\end{equation*}

\end{lemma}

\begin{proof}
Fix $\sigma \in {\left\{ - 1 , 1 \right\}}^{n}$. For every edge $e = \left\{ i , j \right\}$,

\begin{equation*}
{\left( B_{G} \sigma \right)}_{e}^{2} = {\left( \sigma_{i} - \sigma_{j} \right)}^{2} = \begin{cases}
4 & \text{\normalfont if } \sigma_{i} \ne \sigma_{j} \\
0 & \text{\normalfont otherwise}
\end{cases} .
\end{equation*}

Summing over $e \in E$ and taking the square root proves the identity.
\end{proof}

For the forward implication, suppose $\left( G , k \right)$ is a yes-instance. Choose
$\sigma$ such that $c_{G} \left( \sigma \right) \ge k$, and define the arm

\begin{equation*}
x_{\sigma} \coloneqq B_{G} \frac{\sigma}{{\left\Vert B_{G} \sigma \right\Vert}_{2}} .
\end{equation*}

Since $c_{G} \left( \sigma \right) \ge k \ge 1$, Lemma~\ref{lem:simplex-cut-norm} shows that $x_{\sigma}$ is
well-defined. Moreover, we get the following inequality.

\begin{lemma}\label{lem:simplex-duality}
  For every $\sigma \in {\left\{ - 1 , 1 \right\}}^{n}$ such that $B_{G} \sigma \ne 0$, we have
  ${\left\Vert B_{G}^{T} x_{\sigma} \right\Vert}_{1} \ge {\left\Vert B_{G} \sigma \right\Vert}_{2}$. Moreover, equality holds when
  $\sigma$ maximizes ${\left\Vert B_{G} \sigma \right\Vert}_{2}$ over ${\left\{ - 1 , 1 \right\}}^{n}$; in particular,
  equality holds for at least one $\sigma$.

\end{lemma}

\draft{
\begin{proof}
\draft{
Choosing the sign of each coordinate gives
}

\draft{

\begin{equation*}
{\left\Vert B_{G}^{T} x_{\sigma} \right\Vert}_{1} = \operatorname*{max}_{\tau \in {\left\{ - 1 , 1 \right\}}^{n}} \tau^{T} B_{G}^{T} x_{\sigma} = \operatorname*{max}_{\tau \in {\left\{ - 1 , 1 \right\}}^{n}} x_{\sigma}^{T} B_{G} \tau .
\end{equation*}

}

\draft{
Taking $\tau = \sigma$ yields
${\left\Vert B_{G}^{T} x_{\sigma} \right\Vert}_{1} \ge x_{\sigma}^{T} B_{G} \sigma = {\left\Vert B_{G} \sigma \right\Vert}_{2}$.
Now let $\sigma^{\star}$ maximize ${\left\Vert B_{G} \sigma \right\Vert}_{2}$. Since $m \ge 1$, this maximum
is positive, so $x_{\sigma^{\star}}$ is well-defined. For every
$\tau \in {\left\{ - 1 , 1 \right\}}^{n}$, Cauchy--Schwarz gives
}

\draft{

\begin{equation*}
x_{\sigma^{\star}}^{T} B_{G} \tau \le {\left\Vert x_{\sigma^{\star}} \right\Vert}_{2} {\left\Vert B_{G} \tau \right\Vert}_{2} \le {\left\Vert B_{G} \sigma^{\star} \right\Vert}_{2} .
\end{equation*}

}

\draft{
Taking the maximum over $\tau$ gives the reverse inequality, proving
equality.
}
\end{proof}
}

The lemma above gives us a way to link properties of cuts with properties of arms. The missing ingredient is to show that the value of an arm has something to do with ${\left\Vert B_{G}^{T} x \right\Vert}_{1}$. We show this in the next two lemmas.

By construction, $x_{\sigma} \in X$. \draft{Write
$K(x)\coloneqq K_{h_G}(x)$ for the feasible outcomes under this true hypothesis.}

\begin{lemma}\label{lem:simplex-l1-bound}
  For every $x \in X$, the vector $t \coloneqq \gamma B_{G}^{T} x$ satisfies
  \draft{${\left\Vert t \right\Vert}_{1} \le 2 \gamma \sqrt{m}$}. Therefore, if
  \draft{$\gamma \le \frac{1}{2 \sqrt{m}}$}, then ${\left\Vert t \right\Vert}_{1} \le 1$.

\end{lemma}

\draft{
\begin{proof}
\draft{
Each edge coordinate $x_{e}$ contributes $x_{e}$ and $- x_{e}$ to two
coordinates of $B_{G}^{T} x$ before contributions from different edges are
added. The triangle inequality therefore gives
}

\draft{

\begin{equation*}
{\left\Vert t \right\Vert}_{1} = \gamma {\left\Vert B_{G}^{T} x \right\Vert}_{1} \le 2 \gamma {\left\Vert x \right\Vert}_{1} \le 2 \gamma \sqrt{m} {\left\Vert x \right\Vert}_{2} \le 2 \gamma \sqrt{m} ,
\end{equation*}

}

\draft{
where the second inequality is Cauchy--Schwarz and the third uses
${\left\Vert x \right\Vert}_{2} \le 1$.
}
\end{proof}
}

\begin{lemma}\label{lem:simplex-arm-value}
  Suppose \draft{$\gamma \le \frac{1}{2 \sqrt{m}}$}. For each arm $x \in X$, set
  $t \coloneqq \gamma B_{G}^{T} x$. Then $K \left( x \right)$ is nonempty, and
  $v_{h_G} \left( x \right) = \operatorname*{min}_{y \in K \left( x \right)} r \left( x , y \right) = {\left\Vert t \right\Vert}_{1} = \gamma {\left\Vert B_{G}^{T} x \right\Vert}_{1}$. Moreover, the minimum is attained at the outcome $y = \left( p , q , s \right)$ given by
  $p_{i} \coloneqq \operatorname*{max} \left( t_{i} , 0 \right)$, $q_{i} \coloneqq \operatorname*{max} \left( - t_{i} , 0 \right)$, and $s \coloneqq 1 - {\left\Vert t \right\Vert}_{1}$.

\end{lemma}

\draft{
\begin{proof}
\draft{
Fix $x \in X$ and set $t \coloneqq \gamma B_{G}^{T} x$. Lemma~\ref{lem:simplex-l1-bound}
gives ${\left\Vert t \right\Vert}_{1} \le 1$. For the outcome in the statement, $p , q , s$ are
nonnegative, $p - q = t$, and
}

\draft{

\begin{equation*}
\sum_{i = 1}^{n} \left( p_{i} + q_{i} \right) + s = {\left\Vert t \right\Vert}_{1} + 1 - {\left\Vert t \right\Vert}_{1} = 1 .
\end{equation*}

}

\draft{
Thus this outcome belongs to $K \left( x \right)$. For any $\left( p , q , s \right) \in K \left( x \right)$,
nonnegativity gives $p_{i} + q_{i} \ge \left| p_{i} - q_{i} \right| = \left| t_{i} \right|$ for every $i$, and hence
$r \left( x , \left( p , q , s \right) \right) \ge {\left\Vert t \right\Vert}_{1}$. The outcome in the statement attains equality in
every coordinate, so its reward is the minimum.
}
\end{proof}
}

\draft{Given $G$ and an arm $x$, we can therefore compute both its value
and a minimizing outcome, as needed for Lemma~\ref{lem:planning-from-learning}.}

For the final step, apply Lemma~\ref{lem:simplex-arm-value}, Lemma~\ref{lem:simplex-duality}, and
Lemma~\ref{lem:simplex-cut-norm} to obtain

\begin{equation*}
v_{h_G} \left( x_{\sigma} \right) = \gamma {\left\Vert B_{G}^{T} x_{\sigma} \right\Vert}_{1} \ge \gamma {\left\Vert B_{G} \sigma \right\Vert}_{2} = 2 \gamma \sqrt{c_{G} \left( \sigma \right)} \ge 2 \gamma \sqrt{k} .
\end{equation*}

For the reverse implication, suppose there is an arm $x \in X$ satisfying
$v_{h_G} \left( x \right) \ge 2 \gamma \sqrt{k}$. By $\ell_{1}$--$\ell_{\infty}$ duality, choose
$\sigma \in {\left\{ - 1 , 1 \right\}}^{n}$ such that

\begin{equation*}
{\left\Vert B_{G}^{T} x \right\Vert}_{1} = \sigma^{T} B_{G}^{T} x .
\end{equation*}

Using Lemma~\ref{lem:simplex-arm-value} and Lemma~\ref{lem:simplex-cut-norm},

\begin{equation*}
2 \gamma \sqrt{k} \le v_{h_G} \left( x \right) = \gamma {\left\Vert B_{G}^{T} x \right\Vert}_{1} = \gamma x^{T} B_{G} \sigma \le \gamma {\left\Vert x \right\Vert}_{2} {\left\Vert B_{G} \sigma \right\Vert}_{2} \le 2 \gamma \sqrt{c_{G} \left( \sigma \right)} .
\end{equation*}

Since $\gamma > 0$, this implies $c_{G} \left( \sigma \right) \ge k$, so $\left( G , k \right)$ is a yes-instance.
This proves the claimed equivalence.

To apply Lemma~\ref{lem:planning-from-learning}, we need to show that \emph{approximate}
planning is hard. Although the lemma returns only an approximate optimizer,
this is sufficient because the maximum cut size $M_{G}$ is an integer between
$0$ and $m$, \draft{and $V_{h_G} = 2 \gamma \sqrt{M_{G}}$}. Taking $\gamma \coloneqq \frac{1}{2 m n}$ ensures that every
$K_{h_G} \left( x \right)$ is nonempty and that \draft{the values of $V_{h_G}$ for consecutive values of
$M_G$} are separated by
at least $\frac{1}{2 m^{2} n}$. Thus an additive approximation with error
$\varepsilon \coloneqq \frac{1}{6 m^{2} n}$ distinguishes $M_{G} \ge k$ from $M_{G} \le k - 1$.

\draft{Both $|\phi|$ and $|h_G|$ are polynomial in $n,m$. Run the learner on this instance,
using Lemma~\ref{lem:simplex-arm-value} to supply
minimizing outcomes for $h_G$. Lemma~\ref{lem:planning-from-learning}
then returns an approximately optimal arm in polynomial time.}
We can therefore decide MAX-CUT
by comparing the returned arm's value with the midpoint between
$2 \gamma \sqrt{k - 1}$ and $2 \gamma \sqrt{k}$, computed to sufficient precision.
Hence Lemma~\ref{lem:planning-from-learning} implies that a randomized polynomial-time
learner would place MAX-CUT in $\text{\normalfont BPP}$, and therefore that
$\text{\normalfont NP} \subseteq \text{\normalfont BPP}$. A deterministic learner would instead imply
$\text{\normalfont P} = \text{\normalfont NP}$.
\draft{Finally, the constructed hypothesis satisfies the transversality
condition with constant $(1+4\sqrt n)^{-1}\ge S$, so $h_G\in H$. To see this, fix $x$ and put
$t=\gamma B_G^Tx$, so $\|t\|_1\le2\gamma\sqrt m\le1/2$ by
Lemma~\ref{lem:simplex-l1-bound} and the choice of $\gamma$.
For any $u$ with $C_{h_G}u=t$, let $v$ be a nearest point in $D$, and set
$e=t-C_{h_G}v$ and $a=\|e\|_1\le\sqrt{2n}\|u-v\|_2$.
If $a=0$, then $v\in K(x)$. Otherwise,
$\|t+e/(2a)\|_1\le1$, so there is a $w\in D$ with
$C_{h_G}w=t+e/(2a)$. The point
$v'=(v+2aw)/(1+2a)$ lies in $K(x)$.
Since $D$ has diameter $\sqrt2$, we obtain
\begin{equation*}
\operatorname{dist}(u,K(x))
\le\|u-v\|_2+\frac{2a\sqrt2}{1+2a}
\le(1+4\sqrt n)\operatorname{dist}(u,D),
\end{equation*}
which proves the condition.}
\end{proof}

\subsection{A Euclidean outcome ball and a polytope of arms}\label{app:hardness-polytope}

\draft{
We now keep the Euclidean outcome ball and the affine constraints of
Section~\ref{sec:linear-imprecise-bandits}, but allow the arm set to be a polytope given by inequalities with rational coefficients.
The parameters $\phi$ are those of Section~\ref{sec:linear-imprecise-bandits},
with these inequalities in place of the centre and radius of $X$.
Planning is NP-hard even when the hypothesis is known and the arm set is
the cube. The cube has only $2 n$ defining inequalities, but $2^{n}$ vertices.
}

\begin{theorem}\label{thm:additive-polytope-np-hardness}
  \draft{
  Consider the setting of Section~\ref{sec:linear-imprecise-bandits} with
  $X$ a polytope instead of a ball. If, for some fixed $\alpha > 0$ and
  polynomial $P$, one randomized polynomial-time learner
  guarantees $R_{T} \le P \left( | \phi | \right) T^{1 - \alpha}$ for all
  parameters $\phi$ and every true hypothesis $h^{\star} \in H$, then
  $\text{\normalfont NP} \subseteq \text{\normalfont BPP}$. If a deterministic learner achieves the same
  regret guarantee, then $\text{\normalfont P} = \text{\normalfont NP}$.
  }

\end{theorem}

\draft{
\begin{proof}
\draft{
We again reduce from MAX-CUT. Let $G = \left( V , E \right)$ have $n$ vertices and
$m \ge 1$ edges, and let $B_{G} \in \mathbb{R}^{m \times n}$ be its edge-vertex
incidence matrix, as in Appendix~\ref{app:hardness-simplex}. Put
}

\draft{

\begin{equation*}
\gamma \coloneqq \frac{1}{4 m} , \quad X \coloneqq {\left[ - 1 , 1 \right]}^{n} .
\end{equation*}

}

\draft{
Write outcomes as $y = \left( u , s \right) \in \mathbb{R}^{m} \times \mathbb{R}$ and take
}

\draft{

\begin{equation*}
D \coloneqq \left\{ \left( u , s \right) : {\left\Vert u \right\Vert}_{2}^{2} + s^{2} \le 1 \right\} .
\end{equation*}

}

\draft{
Take $d_X=n$, $d_D=m+1$, $d_C=m$, the cutoff $S\coloneqq\frac{1}{2}$, and
$C_{\max}\coloneqq2$, let $H$ be the hypothesis class of
Section~\ref{sec:linear-imprecise-bandits} with these sets $X$ and $D$, and let
$r \left( x , \left( u , s \right) \right) \coloneqq s$, so that $C_r=2$.
The parameters $\phi$, and hence the sets $X$ and $D$, the reward $r$, and
the class $H$, depend only on $n,m$, not on $G$. The graph determines only the true hypothesis. We let $h_G$ be the
hypothesis with coefficients $C_{h_G}=\left[I_m\ \ 0\right]$,
$B_{h_G}=-\gamma B_G$, and $d_{h_G}=0$. These coefficients have total
encoding length polynomial in $n,m$, and
}

\draft{

\begin{equation*}
K_{h_G} \left( x \right) = \left\{ \left( u , s \right) \in D : u = \gamma B_G x \right\} .
\end{equation*}

}

\draft{
Every $x \in X$ satisfies
}

\draft{

\begin{equation*}
{\left\Vert B_{G} x \right\Vert}_{2}^{2} = \sum_{\left\{ i , j \right\} \in E} {\left( x_{i} - x_{j} \right)}^{2} \le 4 m ,
\end{equation*}

}

\draft{
so
}

\draft{

\begin{equation*}
{\left\Vert \gamma B_{G} x \right\Vert}_{2}^{2} \le \frac{1}{4 m} \le \frac{1}{4} .
\end{equation*}

}

\draft{
Thus $K_{h_G} \left( x \right)$ is nonempty for every $x \in X$, so $h_G$
belongs to $H$. Minimizing $s$ over this set gives
}

\draft{

\begin{equation*}
v_{h_G} \left( x \right) = - \sqrt{1 - \gamma^{2} {\left\Vert B_{G} x \right\Vert}_{2}^{2}} .
\end{equation*}

}

\draft{
Set $q \left( x \right) \coloneqq {\left\Vert B_{G} x \right\Vert}_{2}^{2}$. With every coordinate except $x_{i}$ fixed,
$q \left( x \right)$ is a convex quadratic in $x_{i}$, so replacing $x_{i}$ by one of the
endpoints of $\left[ - 1 , 1 \right]$ does not decrease $q$. Applying this replacement one
coordinate at a time produces a sign vector $\sigma \in {\left\{ - 1 , 1 \right\}}^{n}$ with
$q \left( \sigma \right) \ge q \left( x \right)$. By Lemma~\ref{lem:simplex-cut-norm},
$q \left( \sigma \right) = 4 c_{G} \left( \sigma \right)$, where $c_{G} \left( \sigma \right)$
is the number of edges between the two sides of $\sigma$. Consequently, the maximum cut
size $M_G$ satisfies
}

\draft{

\begin{equation*}
\operatorname*{max}_{x \in {\left[ - 1 , 1 \right]}^{n}} q \left( x \right) = 4 M_{G} .
\end{equation*}

}

\draft{
Since $v_{h_G} \left( x \right)$ is strictly increasing in $q \left( x \right)$, if the maximum cut has
size $k$, then $V_{h_G} = V_{k}$, where
}

\draft{

\begin{equation*}
V_{k} = - \sqrt{1 - \frac{k}{4 m^{2}}} .
\end{equation*}

}

\draft{
For $k = 1 , \ldots , m$, consecutive possible values satisfy
}

\draft{

\begin{equation*}
V_{k} - V_{k - 1} = \frac{\frac{1}{4 m^{2}}}{\sqrt{1 - \frac{k - 1}{4 m^{2}}} + \sqrt{1 - \frac{k}{4 m^{2}}}} \ge \frac{1}{8 m^{2}} .
\end{equation*}

}

\draft{
Hence an estimate of $V_{h_G}$ within
}

\draft{

\begin{equation*}
\Delta \coloneqq \frac{1}{32 m^{2}}
\end{equation*}

}

\draft{
identifies $k$ after we compute the $m + 1$ candidate values to polynomially
many bits. Thus approximating $V_{h_G}$ is NP-hard.
}

\draft{
We now run the same learner for all graphs with $n$ vertices and $m$ edges,
using $h_G$ only to simulate its observations. To apply
Lemma~\ref{lem:planning-from-learning}, we need, for every arm $x$ and every
$\eta \in \left( 0 , 1 \right)$, an outcome in $K_{h_G} \left( x \right)$ that has rational coordinates and whose reward is within $\eta$
of $v_{h_G} \left( x \right)$. The outcome attaining $v_{h_G} \left( x \right)$ need not have rational coordinates, so we
compute an approximation that does. For an arm $x$ with rational coordinates, binary
search gives, in time polynomial in $| \phi |$, $|h_G|$, the encoding
length of $x$, and $\operatorname{log} \left( \frac{1}{\eta} \right)$, a rational
$a \ge 0$ such that

\begin{equation*}
0 \le \sqrt{1 - \gamma^{2} q \left( x \right)} - a \le \eta , \quad a^{2} \le 1 - \gamma^{2} q \left( x \right) .
\end{equation*}

The point $y \coloneqq \left( \gamma B_{G} x , - a \right)$ then lies in $K_{h_G} \left( x \right)$, and its reward
satisfies $0 \le r \left( x , y \right) - v_{h_G} \left( x \right) \le \eta$. So
Lemma~\ref{lem:planning-from-learning}, applied with accuracy $\Delta$, returns with probability at
least $\frac{2}{3}$ an arm $\hat{x}$ with $V_{h_G} - v_{h_G} \left( \hat{x} \right) \le \Delta$.

Round the coordinates of $\hat{x}$ to endpoints without decreasing $q$. If the
maximum cut has size $k$, the resulting sign vector $\sigma$ satisfies
$v_{h_G} \left( \sigma \right) \ge V_{k} - \Delta$. If $c_{G} \left( \sigma \right) \le k - 1$, it would instead satisfy
$v_{h_G} \left( \sigma \right) \le V_{k - 1} \le V_{k} - \frac{1}{8 m^{2}} < V_{k} - \Delta$. The rounded sign vector therefore encodes
a maximum cut. The binary search, the simulation, and the rounding all take
polynomial time, proving $\text{\normalfont NP} \subseteq \text{\normalfont BPP}$. If the
learner is deterministic, the same argument succeeds with certainty and gives
$\text{\normalfont P} = \text{\normalfont NP}$.
}

\draft{
It remains to verify the transversality condition. For a fixed $x$, put
}

\draft{

\begin{equation*}
a \coloneqq {\left\Vert \gamma B_{G} x \right\Vert}_{2} , \quad \rho \coloneqq \sqrt{1 - a^{2}} \ge \frac{\sqrt{3}}{2} .
\end{equation*}

}

\draft{
The outcomes satisfying the affine constraint form the line $L_{h_G} \left( x \right)$,
whose intersection with $D$ is $K_{h_G} \left( x \right)$:
}

\draft{

\begin{equation*}
L_{h_G} \left( x \right) = \left\{ \left( \gamma B_{G} x , s \right) : s \in \mathbb{R} \right\} , \quad K_{h_G} \left( x \right) = \left\{ \left( \gamma B_{G} x , s \right) : \left| s \right| \le \rho \right\} .
\end{equation*}

}

\draft{
If $p = \left( \gamma B_{G} x , s \right) \in L_{h_G} \left( x \right)$ lies outside $D$ and $t \coloneqq \left| s \right| > \rho$,
then
}

\draft{

\begin{equation*}
\operatorname{dist} \left( p , K_{h_G} \left( x \right) \right) = t - \rho , \quad \operatorname{dist} \left( p , D \right) = \sqrt{a^{2} + t^{2}} - 1 .
\end{equation*}

}

\draft{
Factoring the second expression gives
}

\draft{

\begin{equation*}
\frac{\operatorname{dist} \left( p , D \right)}{\operatorname{dist} \left( p , K_{h_G} \left( x \right) \right)} = \frac{t + \rho}{\sqrt{a^{2} + t^{2}} + 1} \ge \rho \ge \frac{\sqrt{3}}{2} .
\end{equation*}

}

\draft{
To check the first inequality, multiply by the positive denominator,
subtract $\rho$ from both sides, and square. This gives the equivalent inequality
$\left( 1 - \rho^{2} \right) \left( t^{2} - \rho^{2} \right) \ge 0$. Thus the transversality condition holds with
constant $\frac{\sqrt{3}}{2} \ge S$, so $h_G \in H$.
}
\end{proof}
}

\subsection{Trilinear constraints with Euclidean arm and outcome balls}\label{app:hardness-trilinear}

\draft{
In this section, we keep $X$ and $D$ Euclidean balls and the reward
linear, but make the constraints more general. We now allow the
coefficients of the outcome in the constraints to depend on the arm. The resulting
constraints are still a special case of those in the linear robust
bandits of \citet[Section~2.2]{kosoy2025imprecise}. We show that
this change already makes planning NP-hard, even approximately.

We first recall Kosoy's constraints. There, outcomes are vectors in a
space $Y$, and $D$ lies in a hyperplane $\{y\in Y:y_0=1\}$ given by one of
the outcome coordinates. A hypothesis is specified by a parameter $z$ in a
compact set $Z\subseteq\mathbb{R}^{d_Z}$, through a map

\begin{equation*}
F : \mathbb{R}^{d_{X}} \times \mathbb{R}^{d_{Z}} \times Y \to W
\end{equation*}

into a constraint space $W$. This map must be bilinear in $(z,y)$, but it
may depend on the arm in any continuous way. Writing
$F_{x , z} \left( y \right) \coloneqq F \left( x , z , y \right)$, the
hypothesis $h_z$ allows the distributions whose expected value lies in
the kernel of $F_{x,z}$:

\begin{equation*}
L_{z} \left( x \right) \coloneqq \operatorname{ker} F_{x , z} , \quad K_{z} \left( x \right) \coloneqq L_{z} \left( x \right) \cap D , \quad h_{z} \left( x \right) \coloneqq \left\{ P \in \Delta D : \mathbb{E}_{y \sim P} \left[ y \right] \in K_{z} \left( x \right) \right\} .
\end{equation*}

The constraints of Section~\ref{sec:linear-imprecise-bandits} are of this form. To see this, add an outcome coordinate $y_0$
that is fixed to one, write $z=(B,C,d)$, and set
$F(x,z,(y_0,y))\coloneqq Cy+(Bx+d)y_0$. Then $K_z(x)$ is the set in
Equation~\ref{eq:setting-compatible-set}. In this map, the arm enters only
through the term $(Bx+d)y_0$, so the coefficients $C$ of the outcome do
not depend on $x$.

We now let these coefficients depend on $x$ by allowing $F$ to be
linear in each of $x$, $z$, and $y$ separately, that is, trilinear. Such a
map is still allowed in Kosoy's framework. An instance now consists of a
Euclidean ball $X$, a Euclidean ball $D$ in the hyperplane $y_0=1$, a
linear reward $r$, a Euclidean ball $Z$ of parameters, and a trilinear map
$F$. As in
Section~\ref{sec:linear-imprecise-bandits}, the parameters $\phi$ of the
family are rational data, namely the centres and radii of $X$, $D$, and
$Z$, the coefficients of $r$ and $F$, the cutoff $S$, and the bound
$C_{\max}\ge C_r$, with $S^{-1}$ and $C_{\max}$ integers given in unary. Since the reward is linear, the value functions of
$h_z$ are

\begin{equation*}
v_{z} \left( x \right) \coloneqq v_{h_z} \left( x \right) = \operatorname*{min}_{y \in K_{z} \left( x \right)} r \left( x , y \right) , \quad V_{z} \coloneqq \operatorname*{max}_{x \in X} v_{z} \left( x \right) .
\end{equation*}

The hypothesis class $H$ consists of the $h_z$ with $z\in Z$ that satisfy
the transversality condition of
Section~\ref{sec:linear-imprecise-bandits} with the cutoff $S$. For these
constraints, the condition asks that
$\operatorname{dist}(p,D)\ge S\operatorname{dist}(p,K_{z}(x))$
for every arm $x$ and every point $p\in L_{z}(x)$ outside $D$ with
$p_0=1$. The
proof also checks Kosoy's requirements that every $K_z(x)$ is nonempty
and every $F_{x,z}$ is onto.
}

\begin{theorem}\label{thm:trilinear-balls-np-hardness}
  \draft{Consider the setting of Section~\ref{sec:linear-imprecise-bandits}
  with the affine constraints replaced by the trilinear constraints described
  above. If, for some fixed $\alpha > 0$ and polynomial $P$, one randomized
  polynomial-time learner guarantees
  $R_{T} \le P \left( | \phi | \right) T^{1 - \alpha}$ for all
  parameters $\phi$ and every true hypothesis $h^{\star} \in H$, then
  $\text{\normalfont NP} \subseteq \text{\normalfont BPP}$. If a deterministic
  learner achieves the same regret guarantee, then
  $\text{\normalfont P} = \text{\normalfont NP}$.}
\end{theorem}

\draft{
Let $G = \left( V , E \right)$ be an undirected graph with $V = \left\{ 1 , \ldots , n \right\}$ and
$E = \left\{ e_{1} , \ldots , e_{m} \right\}$, where $m \ge 1$ and $e_{\ell} = \left\{ i_{\ell} , j_{\ell} \right\}$. Put $d \coloneqq n + m$.
For $q = \left( u , w \right) \in \mathbb{R}^{n} \times \mathbb{R}^{m}$, define the homogeneous cubic
}

\draft{

\begin{equation*}
p_{G} \left( q \right) \coloneqq \sum_{\ell = 1}^{m} u_{i_{\ell}} u_{j_{\ell}} w_{\ell} .
\end{equation*}

}

\begin{lemma}\label{lem:clique-cubic}
  \draft{
  If $\omega \left( G \right)$ is the clique number of $G$, then
  }

  \draft{

\begin{equation*}
\operatorname*{max}_{{\left\Vert q \right\Vert}_{2} = 1} p_{G} {\left( q \right)}^{2} = \frac{2}{27} \left( 1 - \frac{1}{\omega \left( G \right)} \right) .
\end{equation*}

  }

\end{lemma}

\draft{
\begin{proof}
\draft{
Write $\rho \coloneqq {\left\Vert u \right\Vert}_{2}$ and $\sigma \coloneqq {\left\Vert w \right\Vert}_{2}$. For fixed $u$ and
$\sigma$, Cauchy--Schwarz gives
}

\draft{

\begin{equation*}
\operatorname*{max}_{{\left\Vert w \right\Vert}_{2} = \sigma} p_{G} \left( u , w \right) = \sigma \sqrt{\sum_{\ell = 1}^{m} u_{i_{\ell}}^{2} u_{j_{\ell}}^{2}} .
\end{equation*}

}

\draft{
Write $u = \rho s$ with ${\left\Vert s \right\Vert}_{2} = 1$ and set $\pi_{i} \coloneqq s_{i}^{2}$. The
Motzkin--Straus theorem \citep{motzkin1965maxima} gives
}

\draft{

\begin{equation*}
\operatorname*{max}_{{\left\Vert s \right\Vert}_{2} = 1} \sum_{\left\{ i , j \right\} \in E} s_{i}^{2} s_{j}^{2} = \operatorname*{max}_{\pi_{i} \ge 0 , \sum_{i} \pi_{i} = 1} \sum_{\left\{ i , j \right\} \in E} \pi_{i} \pi_{j} = \frac{1}{2} \left( 1 - \frac{1}{\omega \left( G \right)} \right) .
\end{equation*}

}

\draft{
Finally,
}

\draft{

\begin{equation*}
\operatorname*{max}_{\rho^{2} + \sigma^{2} = 1 , \rho \ge 0 , \sigma \ge 0} \rho^{2} \sigma = \frac{2}{3 \sqrt{3}} .
\end{equation*}

}

\draft{
Substituting these bounds into the first display and squaring proves
the claim. Since $p_{G}$ is odd, its
maximum equals its maximum absolute value.
}
\end{proof}
}

\draft{
\begingroup
\renewcommand{\proofname}{Proof of Theorem~\ref{thm:trilinear-balls-np-hardness}}
\begin{proof}
\draft{
We reduce from CLIQUE. Given a graph $G$ and an integer $k$, we construct
an instance in which $X$ is a full-dimensional Euclidean ball, $D$ is a
Euclidean ball in the hyperplane $y_0=1$, and $Z=[1,2]$. Every set
$K_z(x)$ is a single point that does not depend on $z$, and the
inequality of the transversality condition holds with a constant greater
than $0.88$, even for
the points of $L_z(x)$ with $p_0\ne1$. So we can take the cutoff
$S\coloneqq\frac{1}{2}$, and every $h_z$ belongs to $H$. We also take
$C_{\max}\coloneqq2$. Nevertheless, approximating $V_z$
to inverse-polynomial additive accuracy decides whether $G$ has a
$k$-clique, and Lemma~\ref{lem:planning-from-learning} turns a learner
with the assumed regret into such an approximation.

We may restrict CLIQUE to instances with $m \ge 1$ and
$3 \le k \le n$, since we can decide the other cases in polynomial time. Given
$\left( G , k \right)$, use the cubic $p_{G}$ above and write an arm as
$x = \left( x_{0} , \xi \right) \in \mathbb{R} \times \mathbb{R}^{d}$. Take
}

\draft{

\begin{equation*}
X \coloneqq \left\{ \left( x_{0} , \xi \right) : {\left( x_{0} - \frac{5}{3} \right)}^{2} + {\left\Vert \xi \right\Vert}_{2}^{2} \le 1 \right\} .
\end{equation*}

}

\draft{
This is a full-dimensional Euclidean ball and $x_{0} \ge \frac{2}{3}$ throughout $X$.
The ratio $q \coloneqq \frac{\xi}{x_{0}}$ ranges over exactly the ball
${\left\Vert q \right\Vert}_{2} \le \frac{3}{4}$. Indeed,
}

\draft{

\begin{equation*}
1 - {\left( x_{0} - \frac{5}{3} \right)}^{2} - \frac{9}{16} x_{0}^{2} = - \frac{{\left( 15 x_{0} - 16 \right)}^{2}}{144} \le 0 ,
\end{equation*}

}

\draft{
so every ratio has norm at most $\frac{3}{4}$. Conversely, for any such
$q$, the choice $x_{0} = \frac{16}{15}$ and $\xi = x_{0} q$ belongs to $X$.
}

\draft{
Let $Z \coloneqq \left[ 1 , 2 \right]$, and take the outcome and constraint spaces to be
}

\draft{

\begin{equation*}
Y \coloneqq \mathbb{R} \times \mathbb{R}^{d} \times \mathbb{R}^{d \times d} \times \mathbb{R}^{d \times d \times d} ,
\end{equation*}

}

\draft{

\begin{equation*}
W \coloneqq \mathbb{R}^{d} \times \mathbb{R}^{d \times d} \times \mathbb{R}^{d \times d \times d} ,
\end{equation*}

}

\draft{
and write $y = \left( y_{0} , \eta , \Theta , \Xi \right) \in Y$. We fix $y_{0} = 1$ and take the
remaining coordinates to lie in the Euclidean unit ball:
}

\draft{

\begin{equation*}
D \coloneqq \left\{ \left( 1 , \eta , \Theta , \Xi \right) : {\left\Vert \eta \right\Vert}_{2}^{2} + {\left\Vert \Theta \right\Vert}_{F}^{2} + {\left\Vert \Xi \right\Vert}_{F}^{2} \le 1 \right\} .
\end{equation*}

}

\draft{
For $x = \left( x_{0} , \xi \right)$, $z \in \mathbb{R}$, and $y \in Y$, define $F \left( x , z , y \right)$ coordinatewise
by
}

\draft{

\begin{equation*}
F_{i} \left( x , z , y \right) \coloneqq z \left( x_{0} \eta_{i} - \frac{1}{2} \xi_{i} y_{0} \right) ,
\end{equation*}

}

\draft{

\begin{equation*}
F_{i j} \left( x , z , y \right) \coloneqq z \left( x_{0} \Theta_{i j} - \xi_{i} \eta_{j} \right) ,
\end{equation*}

}

\draft{

\begin{equation*}
F_{i j \ell} \left( x , z , y \right) \coloneqq z \left( x_{0} \Xi_{i j \ell} - \xi_{i} \Theta_{j \ell} \right) .
\end{equation*}

}

\draft{
Each term contains one coordinate from each of $x$, $z$, and $y$, so $F$ is
trilinear. Define the reward by
}

\draft{

\begin{equation*}
r \left( x , y \right) \coloneqq \frac{1}{m} \sum_{\ell = 1}^{m} \Xi_{i_{\ell} , j_{\ell} , n + \ell} .
\end{equation*}

}

\draft{
It is linear and independent of $x$; its coefficient vector has Euclidean
norm $\frac{1}{\sqrt{m}} \le 1$, so $C_{r} = \frac{2}{\sqrt{m}} \le 2$.
}

\draft{
Fix $x \in X$ and $z \in Z$. Since $z x_{0} \ne 0$ and $y_{0} = 1$ on $D$, we can solve
$F \left( x , z , y \right) = 0$ first for $\eta$, then for $\Theta$, and then for $\Xi$. This gives
}

\draft{

\begin{equation*}
\eta_{i} = \frac{1}{2} q_{i} , \quad \Theta_{i j} = \frac{1}{2} q_{i} q_{j} , \quad \Xi_{i j \ell} = \frac{1}{2} q_{i} q_{j} q_{\ell} .
\end{equation*}

}

\draft{
Writing $t \coloneqq {\left\Vert q \right\Vert}_{2} \le \frac{3}{4}$, the squared norm of $\left( \eta , \Theta , \Xi \right)$ is
}

\draft{

\begin{equation*}
\frac{1}{4} \left( t^{2} + t^{4} + t^{6} \right) \le \frac{4329}{16384} < 1 .
\end{equation*}

}

\draft{
Thus this point lies in $D$, so $K_{z} \left( x \right)$ is a singleton independent of $z$.
The map $F_{x , z} : Y \to W$ is also onto. For any right-hand side, we can set
$y_{0} = 0$ and solve successively for $\eta$, $\Theta$, and $\Xi$.
}

\draft{
Within the hyperplane $y_{0} = 1$, the equations have just one solution, and
that solution lies in $D$. So there are no solutions outside $D$ to check
in the transversality condition. We can also prove the distance inequality
for points in the full kernel $L_{z} \left( x \right)$. Since $F_{x , z}$ is onto and the
dimension of $Y$ is one more than that of $W$, this kernel is the line
spanned by $y \left( x \right) = \left( 1 , g \left( q \right) \right)$, where $g \left( q \right) \coloneqq \left( \eta , \Theta , \Xi \right)$ is given by the
equations above. If $p = \tau y \left( x \right) \in L_{z} \left( x \right)$ lies outside $D$, then
}

\draft{

\begin{equation*}
\operatorname{dist} \left( p , D \right) \ge \left| \tau - 1 \right| , \quad \operatorname{dist} \left( p , K_{z} \left( x \right) \right) = \left| \tau - 1 \right| \sqrt{1 + {\left\Vert g \left( q \right) \right\Vert}_{2}^{2}} .
\end{equation*}

}

\draft{
Thus, for all these points, $\operatorname{dist} \left( p , D \right)$ is at least
$\operatorname{dist} \left( p , K_{z} \left( x \right) \right)$ times
}

\draft{

\begin{equation*}
\frac{1}{\sqrt{1 + \frac{4329}{16384}}} = \frac{128}{\sqrt{20713}} > 0.88 \ge S .
\end{equation*}

}

\draft{
Evaluating the reward at the unique point in $K_{z} \left( x \right)$ gives
}

\draft{

\begin{equation*}
v_{z} \left( x \right) = \frac{1}{2 m} p_{G} \left( q \right) .
\end{equation*}

}

\draft{
Since $q$ ranges over the radius-$\frac{3}{4}$ ball, homogeneity and
Lemma~\ref{lem:clique-cubic} imply, for every $z \in Z$,
}

\draft{

\begin{equation*}
V_{z} = \frac{27}{128 m} \operatorname*{max}_{{\left\Vert h \right\Vert}_{2} = 1} p_{G} \left( h \right) = \frac{27}{128 m} \sqrt{\frac{2}{27}} \sqrt{1 - \frac{1}{\omega \left( G \right)}} .
\end{equation*}

}

\draft{
For $j = 2 , \ldots , n$, let
}

\draft{

\begin{equation*}
U_{j} \coloneqq \frac{27}{128 m} \sqrt{\frac{2}{27}} \sqrt{1 - \frac{1}{j}} .
\end{equation*}

}

\draft{
If $\omega \left( G \right) \ge k$, then $V_{z} \ge U_{k}$; if $\omega \left( G \right) \le k - 1$, then
$V_{z} \le U_{k - 1}$. The gap between these bounds satisfies
}

\draft{

\begin{equation*}
U_{k} - U_{k - 1} \ge \frac{27}{128 m} \sqrt{\frac{2}{27}} \frac{1}{2 k \left( k - 1 \right)} = \Omega \left( \frac{1}{m k^{2}} \right) .
\end{equation*}

}

\draft{
Since $\sqrt{\frac{2}{27}} > \frac{1}{4}$ and $k \le n$, the gap $\Delta_{k} \coloneqq U_{k} - U_{k - 1}$
satisfies
}

\draft{

\begin{equation*}
\Delta_{k} > \frac{1}{40 m k \left( k - 1 \right)} \ge \frac{1}{40 m n^{2}} .
\end{equation*}

}

\draft{
Set $\varepsilon \coloneqq \frac{1}{500 m n^{2}}$. Using polynomially many bits, we can compute a
rational $\tau_{k}$ such that
}

\draft{

\begin{equation*}
\left| \tau_{k} - \frac{1}{2} \left( U_{k} + U_{k - 1} \right) \right| \le \varepsilon .
\end{equation*}

}

\draft{
An estimate of $V_{z}$ with additive error at most $\varepsilon$ lies above
$\tau_{k}$ when $\omega \left( G \right) \ge k$ and below it when $\omega \left( G \right) \le k - 1$, because
$2 \varepsilon < \frac{\Delta_{k}}{2}$. It therefore decides whether $G$ contains a $k$-clique.
The constructed instance has dimension $O \left( d^{3} \right)$, and all its defining
coefficients are rational with polynomial encoding length. This proves that
approximating $V_{z}$ is NP-hard even when $z$ is known.
}

\draft{
Given any arm $x$ with rational coordinates, we can compute the unique point in $K_{z} \left( x \right)$
exactly in polynomial time using the equations above; $x_{0} \ge \frac{2}{3}$ ensures
that we can divide by $x_{0}$. We can therefore simulate the adversary by returning
this point. Applying Lemma~\ref{lem:planning-from-learning} with the chosen $\varepsilon$
gives, with probability at least $\frac{2}{3}$, an arm $\hat{x}$ with
$V_{z} - v_{z} \left( \hat{x} \right) \le \varepsilon$. We compute $v_{z} \left( \hat{x} \right)$ exactly from the unique point
in $K_{z} \left( \hat{x} \right)$, which gives an estimate of $V_{z}$ with additive error at most
$\varepsilon$. Comparing it with $\tau_{k}$ decides CLIQUE, so
$\text{\normalfont NP} \subseteq \text{\normalfont BPP}$. For a deterministic learner, the estimate is
deterministic and gives
$\text{\normalfont P} = \text{\normalfont NP}$.
}
\end{proof}
\endgroup
}

\subsection{Efficient planning does not imply efficient learning}\label{app:hardness-easy-planning}

\draft{
An efficient planner for every hypothesis does not always give an
efficient learner, so the converse of Lemma~\ref{lem:planning-from-learning} fails.
We show this with instances outside the setting of
Section~\ref{sec:linear-imprecise-bandits}. They have finite arm and
hypothesis sets with short descriptions, and an outcome set that is a
polytope. We also allow the reward to be linear in the arm and in the
outcome separately, so it can contain products of arm and outcome
coordinates.
}

\begin{theorem}\label{thm:easy-planning-hard-learning}
  \draft{There is a family of instances with parameters $\phi=(n,k)$, for
  integers $2\le k\le n$ given in unary, in which planning for a known
  hypothesis takes $O(n^2)$ time, however, if for some fixed $\alpha > 0$ and
  polynomial $P$, one randomized polynomial-time learner guarantees
  $R_{T} \le P \left( | \phi | \right) T^{1 - \alpha}$ for all
  parameters $\phi$ and every true hypothesis $h^{\star} \in H$, then
  $\text{\normalfont NP} = \text{\normalfont RP}$. If a deterministic
  learner achieves the same regret guarantee, then
  $\text{\normalfont P} = \text{\normalfont NP}$.}
\end{theorem}

\draft{
\begin{proof}
\draft{
We reduce from CLIQUE. Both the arms and the hypotheses will correspond to
the $k$-element subsets of $[n]$, and the optimal arm for the hypothesis
corresponding to a set $U$ is the arm corresponding to $U$. So planning is easy once the hypothesis is
known. But an adversary can make every observation consistent with every
$k$-clique of a graph $G$, so a learner with low regret must find a
$k$-clique of $G$ on its own.

Fix $2 \le k \le n$, and let
}

\draft{

\begin{equation*}
E_{n} \coloneqq \left\{ \left\{ i , j \right\} : 1 \le i < j \le n \right\} , \quad m \coloneqq \binom{n}{2} , \quad q \coloneqq \binom{k}{2} .
\end{equation*}

}

\draft{
For every $k$-element set $A \subseteq \left[ n \right]$, define $x^{A} \in \mathbb{R}^{m}$ by
}

\draft{

\begin{equation*}
x_{e}^{A} \coloneqq \begin{cases}
\frac{1}{q} & \text{\normalfont if } e \subseteq A \\
0 & \text{\normalfont otherwise}
\end{cases} , \quad e \in E_{n} ,
\end{equation*}

}

\draft{
and take
}

\draft{

\begin{equation*}
X \coloneqq \left\{ x^{A} : A \subseteq \left[ n \right] , \left| A \right| = k \right\} .
\end{equation*}

}

\draft{
An arm is represented succinctly by the vertex set $A$. For every
$k$-element set $U \subseteq \left[ n \right]$, define $z^{U} \in \mathbb{R}^{1 + m}$ by
}

\draft{

\begin{equation*}
z_{0}^{U} \coloneqq 1 , \quad z_{e}^{U} \coloneqq \begin{cases}
1 & \text{\normalfont if } e \subseteq U \\
0 & \text{\normalfont otherwise}
\end{cases} ,
\end{equation*}

}

\draft{
and let $Z \coloneqq \left\{ z^{U} : U \subseteq \left[ n \right] , \left| U \right| = k \right\}$.
}

\draft{
Write an outcome as $y = \left( y_{0} , g , s \right) \in \mathbb{R}^{1 + 2 m}$ and set
}

\draft{

\begin{equation*}
D \coloneqq \left\{ \left( 1 , g , s \right) : g \in {\left[ 0 , 1 \right]}^{m} , s \in {\left[ - 1 , 1 \right]}^{m} \right\} .
\end{equation*}

}

\draft{
With constraint space $W = \mathbb{R}^{m}$, define the linear map $C_{z} : \mathbb{R}^{1 + 2 m} \to W$ by
}

\draft{

\begin{equation*}
{\left( C_{z} y \right)}_{e} \coloneqq z_{e} \left( g_{e} - y_{0} - s_{e} \right) + z_{0} s_{e} , \quad e \in E_{n} .
\end{equation*}

}

\draft{
For $z \in Z$, set
}

\draft{

\begin{equation*}
K_{z} \left( x \right) \coloneqq \left\{ y \in D : C_{z} y = 0 \right\} , \quad h_{z} \left( x \right) \coloneqq \left\{ Q \in \Delta D : \mathbb{E}_{y \sim Q} \left[ y \right] \in K_{z} \left( x \right) \right\} ,
\end{equation*}

}

\draft{
and let $H \coloneqq \left\{ h_{z} : z \in Z \right\}$.
}

\draft{
The map $\left( y , z \right) \mapsto C_{z} y$ is bilinear. Since $y_{0}$ is an outcome
coordinate, the constraint has the form of Equation~\ref{eq:setting-compatible-set}
with $B_{z} = 0$, $d_{z} = 0$, and $z \mapsto C_{z}$ linear. For $z = z^{U}$,
}

\draft{

\begin{equation*}
{\left( C_{z^{U}} y \right)}_{e} = \begin{cases}
g_{e} - y_{0} & \text{\normalfont if } e \subseteq U \\
s_{e} & \text{\normalfont otherwise} .
\end{cases}
\end{equation*}

}

\draft{
Thus, inside $D$, the constraints force $g_{e} = 1$ for every
$e \subseteq U$ and $s_{e} = 0$ for every $e \nsubseteq U$. The remaining
$g_{e}$ coordinates are free to vary in $\left[ 0 , 1 \right]$. The feasible set is nonempty.
It contains the outcome with $g_{e} = 1$ on the edges of $U$, $g_{e} = 0$ elsewhere, and $s = 0$. Moreover,
$C_{z^{U}}$ is onto $W$. Given $w \in W$, we can set $y_{0} = 0$, use
$g_{e} = w_{e}$ and $s_{e} = 0$ on the edges of $U$, and use $g_{e} = 0$ and $s_{e} = w_{e}$
elsewhere.
}

\draft{
Define the reward
}

\draft{

\begin{equation*}
r \left( x , \left( y_{0} , g , s \right) \right) \coloneqq \sum_{e \in E_{n}} x_{e} g_{e} .
\end{equation*}

}

\draft{
The parameters $\phi=(n,k)$ determine $X$, $D$, $r$, and $H$, with the maps
$C_z$ built into the learner. Only the true $U$ is hidden.
The reward lies in $\left[ 0 , 1 \right]$ on $X \times D$. For an arm $x^{A}$, the adversary
minimizes it by setting every coordinate $g_{e}$ not forced by $U$ to zero, so
}

\draft{

\begin{equation*}
v_{z^{U}} \left( x^{A} \right) = \frac{\binom{\left| U \cap A \right|}{2}}{q} .
\end{equation*}

}

\draft{
Thus $V_{z^{U}} = 1$, uniquely attained by $x^{U}$. Given $z^{U}$, an exact planner
returns $x^{U}$ by setting $x_{e}^{U} = \frac{z_{e}^{U}}{q}$ for every $e \in E_{n}$, which takes
$O \left( m \right) = O \left( n^{2} \right)$ time.
}

\draft{
We now show that a learner with the stated regret bound would solve
CLIQUE. Given a graph $G = \left( \left[ n \right] , E \left( G \right) \right)$ and target clique size $k$
\citep{garey1979computers}, let
}

\draft{

\begin{equation*}
g_{e}^{G} \coloneqq \begin{cases}
1 & \text{\normalfont if } e \in E \left( G \right) \\
0 & \text{\normalfont otherwise}
\end{cases} , \quad y^{G} \coloneqq \left( 1 , g^{G} , 0 \right) \in D .
\end{equation*}

}

\draft{
If $U$ is a $k$-clique of $G$, the adversary can return $y^{G}$ on every round
under hypothesis $z^{U}$. Indeed, every edge inside $U$ has $g_{e}^{G} = 1$, and all the
$s_{e}$ coordinates are zero. For every candidate $A$,
}

\draft{

\begin{equation*}
r \left( x^{A} , y^{G} \right) = \frac{\left| E \left( G \left[ A \right] \right) \right|}{q} .
\end{equation*}

}

\draft{
This reward is $1$ exactly when $A$ is a $k$-clique, and is at most
$1 - \frac{1}{q}$ otherwise.
}

\draft{
Assume the learner in the theorem exists, and give it the outcome $y^{G}$
on every round. If $G$ contains a $k$-clique $U$, these outcomes are
compatible with $z^{U}$, whose optimal value is $1$. If the learner plays
$x^{A_{1}} , \ldots , x^{A_{T}}$, the total reward it loses relative to this value is
}

\draft{

\begin{equation*}
{\mathcal{L}}_{T} \coloneqq T - \sum_{t = 1}^{T} r \left( x^{A_{t}} , y^{G} \right) .
\end{equation*}

}

\draft{
This quantity is nonnegative. Let $\mathcal{E}$ be the event that none of
$A_{1} , \ldots , A_{T}$ is a clique. On this event, the learner loses at least
$\frac{1}{q}$ on every round, so ${\mathcal{L}}_{T} \ge \frac{T}{q}$. Markov's inequality and the
assumed regret bound give
}

\draft{

\begin{equation*}
\operatorname{Pr} \left( \mathcal{E} \right) \le q \frac{E \left[ {\mathcal{L}}_{T} \right]}{T} \le q P \left( | \phi | \right) T^{- \alpha} .
\end{equation*}

}

\draft{
Choose
}

\draft{

\begin{equation*}
T \coloneqq \left\lceil {\left( 4 q P \left( | \phi | \right) \right)}^{\frac{1}{\alpha}} \right\rceil .
\end{equation*}

}

\draft{
Since $|\phi|=O(n)$, both $T$ and the
time needed to simulate the learner are polynomial in the CLIQUE input
length. Check each $k$-set proposed by the learner and accept as soon as
one is a clique in $G$. If $G$ has a $k$-clique, the preceding bound gives
acceptance probability at least $\frac{3}{4}$. If $G$ has none, none of the
proposed sets can pass this check. In that case, the outcomes may be
incompatible with every hypothesis, but the learner still runs
in polynomial time on every history of arms in $X$ and outcomes in $D$. Thus the
simulation also takes polynomial time when $G$ has no $k$-clique.
}

\draft{
This is an RP algorithm for CLIQUE. Since CLIQUE is NP-complete, it implies
$\text{\normalfont NP} = \text{\normalfont RP}$. If the learner is deterministic, the same verified search gives
$\text{\normalfont P} = \text{\normalfont NP}$.
The adversary in this argument returns the same outcome $y^{G}$ on every
round, so the conclusion holds even if the regret guarantee is required
only against such adversaries.
}
\end{proof}
}

\draft{
This example shows that efficient planning alone does not suffice for
efficient learning. It uses exponentially large finite sets $X$ and $Z$
represented by $k$-subsets, a polytope $D$, and the bilinear reward
$r \left( x , y \right) = \sum_{e} x_{e} g_{e}$. It therefore falls outside our setting of Euclidean
balls and linear rewards.
}

%% file: arxiv.bbl
\begin{thebibliography}{81}
\providecommand{\natexlab}[1]{#1}
\providecommand{\url}[1]{\texttt{#1}}
\expandafter\ifx\csname urlstyle\endcsname\relax
  \providecommand{\doi}[1]{doi: #1}\else
  \providecommand{\doi}{doi: \begingroup \urlstyle{rm}\Url}\fi

\bibitem[Alon et~al.(2022)Alon, Hanneke, Holzman, and Moran]{alon2021theory}
Noga Alon, Steve Hanneke, Ron Holzman, and Shay Moran.
\newblock A theory of {PAC} learnability of partial concept classes.
\newblock In \emph{2021 IEEE 62nd Annual Symposium on Foundations of Computer
  Science (FOCS)}, pages 658--671, 2022.
\newblock \doi{10.1109/FOCS52979.2021.00070}.

\bibitem[Amodei et~al.(2016)Amodei, Olah, Steinhardt, Christiano, Schulman, and
  Man{\'e}]{amodei2016concrete}
Dario Amodei, Chris Olah, Jacob Steinhardt, Paul Christiano, John Schulman, and
  Dan Man{\'e}.
\newblock Concrete problems in {AI} safety.
\newblock arXiv preprint arXiv:1606.06565, 2016.

\bibitem[Appel and Kosoy(2025)]{appel2025regret}
Alexander Appel and Vanessa Kosoy.
\newblock Regret bounds for robust online decision making.
\newblock In \emph{Proceedings of the 38th Conference on Learning Theory},
  volume 291 of \emph{Proceedings of Machine Learning Research}, pages 64--146,
  2025.
\newblock URL \url{https://proceedings.mlr.press/v291/appel25a.html}.

\bibitem[Armstrong and Mindermann(2018)]{armstrong2018occam}
Stuart Armstrong and S{\"o}ren Mindermann.
\newblock {Occam's} razor is insufficient to infer the preferences of
  irrational agents.
\newblock In \emph{Advances in Neural Information Processing Systems},
  volume~31, 2018.
\newblock URL \url{https://arxiv.org/abs/1712.05812}.

\bibitem[Auer et~al.(2002)Auer, Cesa-Bianchi, Freund, and
  Schapire]{auer2002nonstochastic}
Peter Auer, Nicol{\`o} Cesa-Bianchi, Yoav Freund, and Robert~E. Schapire.
\newblock The nonstochastic multiarmed bandit problem.
\newblock \emph{SIAM Journal on Computing}, 32\penalty0 (1):\penalty0 48--77,
  2002.
\newblock \doi{10.1137/S0097539701398375}.

\bibitem[Bai et~al.(2022)]{bai2022training}
Yuntao Bai et~al.
\newblock Training a helpful and harmless assistant with reinforcement learning
  from human feedback.
\newblock arXiv preprint arXiv:2204.05862, 2022.

\bibitem[B{\'a}r{\'a}sz et~al.(2014)B{\'a}r{\'a}sz, Christiano, Fallenstein,
  Herreshoff, LaVictoire, and Yudkowsky]{barasz2014robust}
Mih{\'a}ly B{\'a}r{\'a}sz, Paul Christiano, Benja Fallenstein, Marcello
  Herreshoff, Patrick LaVictoire, and Eliezer Yudkowsky.
\newblock Robust cooperation in the {Prisoner's Dilemma}: Program equilibrium
  via provability logic.
\newblock arXiv preprint arXiv:1401.5577, 2014.

\bibitem[Bart{\'o}k et~al.(2014)Bart{\'o}k, Foster, P{\'a}l, Rakhlin, and
  Szepesv{\'a}ri]{bartok2014partial}
G{\'a}bor Bart{\'o}k, Dean~P. Foster, D{\'a}vid P{\'a}l, Alexander Rakhlin, and
  Csaba Szepesv{\'a}ri.
\newblock Partial monitoring---classification, regret bounds, and algorithms.
\newblock \emph{Mathematics of Operations Research}, 39\penalty0 (4):\penalty0
  967--997, 2014.
\newblock \doi{10.1287/moor.2014.0663}.

\bibitem[Ben-David et~al.(2009)Ben-David, P{\'a}l, and
  Shalev-Shwartz]{bendavid2009agnostic}
Shai Ben-David, D{\'a}vid P{\'a}l, and Shai Shalev-Shwartz.
\newblock Agnostic online learning.
\newblock In \emph{Proceedings of the 22nd Conference on Learning Theory},
  2009.
\newblock URL
  \url{https://home.ttic.edu/~shai/papers/Ben-DavidPalShalev09.pdf}.

\bibitem[Bengio et~al.(2024)]{bengio2024managing}
Yoshua Bengio et~al.
\newblock Managing extreme {AI} risks amid rapid progress.
\newblock \emph{Science}, 384\penalty0 (6698):\penalty0 842--845, 2024.
\newblock \doi{10.1126/science.adn0117}.

\bibitem[Bienstock(2016)]{bienstock2016cdt}
Daniel Bienstock.
\newblock A note on polynomial solvability of the {CDT} problem.
\newblock \emph{SIAM Journal on Optimization}, 26\penalty0 (1):\penalty0
  488--498, 2016.
\newblock \doi{10.1137/15M1009871}.

\bibitem[Block et~al.(2022)Block, Dagan, Golowich, and
  Rakhlin]{block2022smoothed}
Adam Block, Yuval Dagan, Noah Golowich, and Alexander Rakhlin.
\newblock Smoothed online learning is as easy as statistical learning.
\newblock In \emph{Proceedings of the 35th Conference on Learning Theory},
  volume 178 of \emph{Proceedings of Machine Learning Research}, pages
  1716--1786, 2022.
\newblock URL \url{https://proceedings.mlr.press/v178/block22a.html}.

\bibitem[Bostrom(2014)]{bostrom2014superintelligence}
Nick Bostrom.
\newblock \emph{Superintelligence: Paths, Dangers, Strategies}.
\newblock Oxford University Press, 2014.

\bibitem[Bubeck et~al.(2011{\natexlab{a}})Bubeck, Munos, and
  Stoltz]{bubeck2011pure}
S{\'e}bastien Bubeck, R{\'e}mi Munos, and Gilles Stoltz.
\newblock Pure exploration in finitely-armed and continuous-armed bandits.
\newblock \emph{Theoretical Computer Science}, 412\penalty0 (19):\penalty0
  1832--1852, 2011{\natexlab{a}}.
\newblock \doi{10.1016/j.tcs.2010.12.059}.

\bibitem[Bubeck et~al.(2011{\natexlab{b}})Bubeck, Munos, Stoltz, and
  Szepesv{\'a}ri]{bubeck2011xarmed}
S{\'e}bastien Bubeck, R{\'e}mi Munos, Gilles Stoltz, and Csaba Szepesv{\'a}ri.
\newblock {$X$}-armed bandits.
\newblock \emph{Journal of Machine Learning Research}, 12\penalty0
  (46):\penalty0 1655--1695, 2011{\natexlab{b}}.
\newblock URL \url{https://www.jmlr.org/papers/v12/bubeck11a.html}.

\bibitem[Casper et~al.(2023)]{casper2023open}
Stephen Casper et~al.
\newblock Open problems and fundamental limitations of reinforcement learning
  from human feedback.
\newblock \emph{Transactions on Machine Learning Research}, 2023.
\newblock URL \url{https://arxiv.org/abs/2307.15217}.

\bibitem[Christiano et~al.(2017)Christiano, Leike, Brown, Martic, Legg, and
  Amodei]{christiano2017deep}
Paul~F. Christiano, Jan Leike, Tom~B. Brown, Miljan Martic, Shane Legg, and
  Dario Amodei.
\newblock Deep reinforcement learning from human preferences.
\newblock In \emph{Advances in Neural Information Processing Systems},
  volume~30, 2017.
\newblock URL \url{https://arxiv.org/abs/1706.03741}.

\bibitem[Critch(2019)]{critch2019bounded}
Andrew Critch.
\newblock A parametric, resource-bounded generalization of {L{\"o}b's} theorem,
  and a robust cooperation criterion for open-source game theory.
\newblock \emph{The Journal of Symbolic Logic}, 84\penalty0 (4):\penalty0
  1368--1381, 2019.
\newblock \doi{10.1017/jsl.2017.42}.

\bibitem[Dalrymple(2026)]{dalrymple2026imprecise}
David~A. Dalrymple.
\newblock Imprecise beliefs: A tiny introduction.
\newblock AI Alignment Forum, July 2026.
\newblock URL
  \url{https://www.alignmentforum.org/posts/e7Pd4Q9TF7jFdmPgz/imprecise-beliefs-a-tiny-introduction}.
\newblock Published as davidad.

\bibitem[Dani et~al.(2008)Dani, Hayes, and Kakade]{dani2008stochastic}
Varsha Dani, Thomas~P. Hayes, and Sham~M. Kakade.
\newblock Stochastic linear optimization under bandit feedback.
\newblock In \emph{Proceedings of the 21st Conference on Learning Theory},
  pages 355--366, 2008.

\bibitem[Daniely(2016)]{daniely2016halfspaces}
Amit Daniely.
\newblock Complexity theoretic limitations on learning halfspaces.
\newblock In \emph{Proceedings of the 48th Annual ACM Symposium on Theory of
  Computing}, pages 105--117, 2016.
\newblock \doi{10.1145/2897518.2897520}.

\bibitem[Daniely and Shalev-Shwartz(2016)]{daniely2016dnf}
Amit Daniely and Shai Shalev-Shwartz.
\newblock Complexity theoretic limitations on learning {DNF's}.
\newblock In \emph{Proceedings of the 29th Conference on Learning Theory},
  volume~49 of \emph{Proceedings of Machine Learning Research}, pages 815--830,
  2016.
\newblock URL \url{https://proceedings.mlr.press/v49/daniely16.html}.

\bibitem[Demski and Garrabrant(2019)]{demski2019embedded}
Abram Demski and Scott Garrabrant.
\newblock Embedded agency.
\newblock arXiv preprint arXiv:1902.09469, 2019.

\bibitem[Dewey(2011)]{dewey2011learning}
Daniel Dewey.
\newblock Learning what to value.
\newblock In \emph{Artificial General Intelligence}, volume 6830 of
  \emph{Lecture Notes in Computer Science}, pages 309--314, Berlin, 2011.
  Springer.
\newblock \doi{10.1007/978-3-642-22887-2_35}.

\bibitem[Ellsberg(1961)]{ellsberg1961risk}
Daniel Ellsberg.
\newblock Risk, ambiguity, and the {Savage} axioms.
\newblock \emph{The Quarterly Journal of Economics}, 75\penalty0 (4):\penalty0
  643--669, 1961.
\newblock \doi{10.2307/1884324}.

\bibitem[Everitt et~al.(2018)Everitt, Lea, and Hutter]{everitt2018agi}
Tom Everitt, Gary Lea, and Marcus Hutter.
\newblock {AGI} safety literature review.
\newblock In \emph{Proceedings of the 27th International Joint Conference on
  Artificial Intelligence}, pages 5441--5449, 2018.
\newblock \doi{10.24963/ijcai.2018/768}.

\bibitem[Eyraud et~al.(2016)Eyraud, Heinz, and Yoshinaka]{eyraud2016efficiency}
R{\'e}mi Eyraud, Jeffrey Heinz, and Ryo Yoshinaka.
\newblock Efficiency in the identification in the limit learning paradigm.
\newblock In Jeffrey Heinz and Jos{\'e}~M. Sempere, editors, \emph{Topics in
  Grammatical Inference}, pages 25--46. Springer, Berlin, Heidelberg, 2016.
\newblock \doi{10.1007/978-3-662-48395-4_2}.

\bibitem[Fallenstein et~al.(2015)Fallenstein, Taylor, and
  Christiano]{fallenstein2015reflective}
Benja Fallenstein, Jessica Taylor, and Paul~F. Christiano.
\newblock Reflective oracles: A foundation for classical game theory.
\newblock arXiv preprint arXiv:1508.04145, 2015.

\bibitem[Foster et~al.(2016)Foster, Kale, and Karloff]{foster2016online}
Dean Foster, Satyen Kale, and Howard Karloff.
\newblock Online sparse linear regression.
\newblock In \emph{Proceedings of the 29th Conference on Learning Theory},
  volume~49 of \emph{Proceedings of Machine Learning Research}, pages 960--970,
  2016.
\newblock URL \url{https://proceedings.mlr.press/v49/foster16.html}.

\bibitem[Foster et~al.(2021)Foster, Kakade, Qian, and
  Rakhlin]{foster2021statistical}
Dylan~J. Foster, Sham~M. Kakade, Jian Qian, and Alexander Rakhlin.
\newblock The statistical complexity of interactive decision making.
\newblock arXiv preprint arXiv:2112.13487, 2021.

\bibitem[Garey and Johnson(1979)]{garey1979computers}
Michael~R. Garey and David~S. Johnson.
\newblock \emph{Computers and Intractability: A Guide to the Theory of
  {NP}-Completeness}.
\newblock W. H. Freeman, San Francisco, 1979.

\bibitem[Garey et~al.(1976)Garey, Johnson, and Stockmeyer]{garey1976simplified}
Michael~R. Garey, David~S. Johnson, and Larry Stockmeyer.
\newblock Some simplified {NP}-complete graph problems.
\newblock \emph{Theoretical Computer Science}, 1\penalty0 (3):\penalty0
  237--267, 1976.
\newblock \doi{10.1016/0304-3975(76)90059-1}.

\bibitem[Gilboa and Schmeidler(1989)]{gilboa1989maxmin}
Itzhak Gilboa and David Schmeidler.
\newblock Maxmin expected utility with non-unique prior.
\newblock \emph{Journal of Mathematical Economics}, 18\penalty0 (2):\penalty0
  141--153, 1989.
\newblock \doi{10.1016/0304-4068(89)90018-9}.

\bibitem[Hadfield-Menell et~al.(2016)Hadfield-Menell, Dragan, Abbeel, and
  Russell]{hadfieldmenell2016cooperative}
Dylan Hadfield-Menell, Anca Dragan, Pieter Abbeel, and Stuart Russell.
\newblock Cooperative inverse reinforcement learning.
\newblock In \emph{Advances in Neural Information Processing Systems},
  volume~29, pages 3909--3917, 2016.
\newblock URL \url{https://arxiv.org/abs/1606.03137}.

\bibitem[Hadfield-Menell et~al.(2017)Hadfield-Menell, Milli, Abbeel, Russell,
  and Dragan]{hadfieldmenell2017inverse}
Dylan Hadfield-Menell, Smitha Milli, Pieter Abbeel, Stuart Russell, and Anca
  Dragan.
\newblock Inverse reward design.
\newblock In \emph{Advances in Neural Information Processing Systems},
  volume~30, 2017.
\newblock URL \url{https://arxiv.org/abs/1711.02827}.

\bibitem[Haghtalab et~al.(2022)Haghtalab, Han, Shetty, and
  Yang]{haghtalab2022oracle}
Nika Haghtalab, Yanjun Han, Abhishek Shetty, and Kunhe Yang.
\newblock Oracle-efficient online learning for smoothed adversaries.
\newblock In \emph{Advances in Neural Information Processing Systems},
  volume~35, 2022.

\bibitem[Haghtalab et~al.(2024)Haghtalab, Roughgarden, and
  Shetty]{haghtalab2024smoothed}
Nika Haghtalab, Tim Roughgarden, and Abhishek Shetty.
\newblock Smoothed analysis with adaptive adversaries.
\newblock \emph{Journal of the ACM}, 71\penalty0 (3):\penalty0 19:1--19:34,
  2024.
\newblock \doi{10.1145/3656638}.

\bibitem[Hubinger et~al.(2019)Hubinger, van Merwijk, Mikulik, Skalse, and
  Garrabrant]{hubinger2019risks}
Evan Hubinger, Chris van Merwijk, Vladimir Mikulik, Joar Skalse, and Scott
  Garrabrant.
\newblock Risks from learned optimization in advanced machine learning systems.
\newblock arXiv preprint arXiv:1906.01820, 2019.

\bibitem[Hutter(2009)]{hutter2009open}
Marcus Hutter.
\newblock Open problems in universal induction \& intelligence.
\newblock \emph{Algorithms}, 2\penalty0 (3):\penalty0 879--906, 2009.
\newblock \doi{10.3390/a2030879}.

\bibitem[Jaynes(2003)]{jaynes2003probability}
Edwin~T. Jaynes.
\newblock \emph{Probability Theory: The Logic of Science}.
\newblock Cambridge University Press, 2003.
\newblock Edited by G. Larry Bretthorst.

\bibitem[Jeon et~al.(2020)Jeon, Milli, and Dragan]{jeon2020reward}
Hong~Jun Jeon, Smitha Milli, and Anca Dragan.
\newblock Reward-rational (implicit) choice: A unifying formalism for reward
  learning.
\newblock In \emph{Advances in Neural Information Processing Systems},
  volume~33, 2020.
\newblock URL \url{https://arxiv.org/abs/2002.04833}.

\bibitem[Ji et~al.(2026)]{ji2023alignment}
Jiaming Ji et~al.
\newblock {AI} alignment: A contemporary survey.
\newblock \emph{ACM Computing Surveys}, 58\penalty0 (5):\penalty0
  132:1--132:38, 2026.
\newblock \doi{10.1145/3770749}.

\bibitem[Jia et~al.(2023)Jia, Li, Rakhlin, Sekhari, and
  Srebro]{jia2023agnostic}
Zeyu Jia, Gene Li, Alexander Rakhlin, Ayush Sekhari, and Nathan Srebro.
\newblock When is agnostic reinforcement learning statistically tractable?
\newblock In \emph{Advances in Neural Information Processing Systems},
  volume~36, 2023.
\newblock URL \url{https://arxiv.org/abs/2310.06113}.

\bibitem[Kalai and Lehrer(1993)]{kalai1993rational}
Ehud Kalai and Ehud Lehrer.
\newblock Rational learning leads to {Nash} equilibrium.
\newblock \emph{Econometrica}, 61\penalty0 (5):\penalty0 1019--1045, 1993.
\newblock \doi{10.2307/2951492}.

\bibitem[Karp(1972)]{karp1972reducibility}
Richard~M. Karp.
\newblock Reducibility among combinatorial problems.
\newblock In Raymond~E. Miller, James~W. Thatcher, and Jean~D. Bohlinger,
  editors, \emph{Complexity of Computer Computations}, pages 85--103. Plenum
  Press, New York, 1972.
\newblock \doi{10.1007/978-1-4684-2001-2_9}.

\bibitem[Kearns et~al.(1994)Kearns, Schapire, and Sellie]{kearns1994agnostic}
Michael~J. Kearns, Robert~E. Schapire, and Linda~M. Sellie.
\newblock Toward efficient agnostic learning.
\newblock \emph{Machine Learning}, 17\penalty0 (2--3):\penalty0 115--141, 1994.
\newblock \doi{10.1007/BF00993468}.

\bibitem[Kirschner et~al.(2020)Kirschner, Lattimore, and
  Krause]{kirschner2020information}
Johannes Kirschner, Tor Lattimore, and Andreas Krause.
\newblock Information directed sampling for linear partial monitoring.
\newblock In \emph{Proceedings of the 33rd Conference on Learning Theory},
  volume 125 of \emph{Proceedings of Machine Learning Research}, pages
  2328--2369, 2020.
\newblock URL \url{https://proceedings.mlr.press/v125/kirschner20a.html}.

\bibitem[Kosoy(2018)]{kosoy2018learningagenda}
Vanessa Kosoy.
\newblock The learning-theoretic {AI} alignment research agenda.
\newblock AI Alignment Forum, July 2018.
\newblock URL
  \url{https://www.alignmentforum.org/posts/5bd75cc58225bf0670375575/the-learning-theoretic-ai-alignment-research-agenda}.

\bibitem[Kosoy(2023)]{kosoy2023agenda}
Vanessa Kosoy.
\newblock The learning-theoretic agenda: Status 2023.
\newblock AI Alignment Forum, April 2023.
\newblock URL
  \url{https://www.alignmentforum.org/posts/ZwshvqiqCvXPsZEct/the-learning-theoretic-agenda-status-2023}.

\bibitem[Kosoy(2024)]{kosoy2024thesis}
Vanessa Kosoy.
\newblock Imprecise multi-armed bandits.
\newblock M.sc. thesis, The Rachel and Selim Benin School of Computer Science,
  The Hebrew University of Jerusalem, 2024.
\newblock URL \url{https://arxiv.org/abs/2405.05673v1}.

\bibitem[Kosoy(2025)]{kosoy2025imprecise}
Vanessa Kosoy.
\newblock Imprecise multi-armed bandits: Representing irreducible uncertainty
  as a zero-sum game.
\newblock \emph{Journal of Machine Learning Research}, 26\penalty0
  (184):\penalty0 1--75, 2025.
\newblock URL \url{https://www.jmlr.org/papers/v26/24-2001.html}.

\bibitem[Kosoy(2026)]{kosoy2026ambiguous}
Vanessa Kosoy.
\newblock Ambiguous online learning.
\newblock In \emph{Proceedings of the 39th Conference on Learning Theory},
  volume 336 of \emph{Proceedings of Machine Learning Research}, pages
  4229--4266, 2026.
\newblock URL \url{https://proceedings.mlr.press/v336/kosoy26a.html}.

\bibitem[Krishnamurthy et~al.(2016)Krishnamurthy, Agarwal, and
  Langford]{krishnamurthy2016pac}
Akshay Krishnamurthy, Alekh Agarwal, and John Langford.
\newblock {PAC} reinforcement learning with rich observations.
\newblock In \emph{Advances in Neural Information Processing Systems},
  volume~29, 2016.
\newblock URL
  \url{https://proceedings.neurips.cc/paper/2016/hash/2387337ba1e0b0249ba90f55b2ba2521-Abstract.html}.

\bibitem[Leike et~al.(2016)Leike, Taylor, and Fallenstein]{leike2016grain}
Jan Leike, Jessica Taylor, and Benya Fallenstein.
\newblock A formal solution to the grain of truth problem.
\newblock In \emph{Proceedings of the 32nd Conference on Uncertainty in
  Artificial Intelligence}, pages 427--436, 2016.
\newblock URL \url{https://arxiv.org/abs/1609.05058}.

\bibitem[Leike et~al.(2018)Leike, Krueger, Everitt, Martic, Maini, and
  Legg]{leike2018scalable}
Jan Leike, David Krueger, Tom Everitt, Miljan Martic, Vishal Maini, and Shane
  Legg.
\newblock Scalable agent alignment via reward modeling: A research direction.
\newblock arXiv preprint arXiv:1811.07871, 2018.

\bibitem[Levi(1974)]{levi1974indeterminate}
Isaac Levi.
\newblock On indeterminate probabilities.
\newblock \emph{The Journal of Philosophy}, 71\penalty0 (13):\penalty0
  391--418, 1974.
\newblock \doi{10.2307/2025161}.

\bibitem[Levi(1980)]{levi1980enterprise}
Isaac Levi.
\newblock \emph{The Enterprise of Knowledge: An Essay on Knowledge, Credal
  Probability, and Chance}.
\newblock MIT Press, Cambridge, MA, 1980.

\bibitem[Li(2025)]{li2025agnostic}
Gene~X. Li.
\newblock \emph{Agnostic Reinforcement Learning: Foundations and Algorithms}.
\newblock PhD thesis, Toyota Technological Institute at Chicago, Chicago, IL,
  2025.
\newblock URL \url{https://arxiv.org/abs/2506.01884}.

\bibitem[Littlestone(1988)]{littlestone1988learning}
Nick Littlestone.
\newblock Learning quickly when irrelevant attributes abound: A new
  linear-threshold algorithm.
\newblock \emph{Machine Learning}, 2\penalty0 (4):\penalty0 285--318, 1988.
\newblock \doi{10.1007/BF00116827}.

\bibitem[Milli et~al.(2017)Milli, Hadfield-Menell, Dragan, and
  Russell]{milli2017obedient}
Smitha Milli, Dylan Hadfield-Menell, Anca Dragan, and Stuart Russell.
\newblock Should robots be obedient?
\newblock In \emph{Proceedings of the 26th International Joint Conference on
  Artificial Intelligence}, pages 4754--4760, 2017.
\newblock \doi{10.24963/ijcai.2017/662}.

\bibitem[Motzkin and Straus(1965)]{motzkin1965maxima}
Theodore~S. Motzkin and Ernst~G. Straus.
\newblock Maxima for graphs and a new proof of a theorem of {Tur{\'a}n}.
\newblock \emph{Canadian Journal of Mathematics}, 17:\penalty0 533--540, 1965.
\newblock \doi{10.4153/CJM-1965-053-6}.

\bibitem[Nachbar(1997)]{nachbar1997prediction}
John~H. Nachbar.
\newblock Prediction, optimization, and learning in repeated games.
\newblock \emph{Econometrica}, 65\penalty0 (2):\penalty0 275--309, 1997.
\newblock \doi{10.2307/2171894}.

\bibitem[Nachbar(2005)]{nachbar2005beliefs}
John~H. Nachbar.
\newblock Beliefs in repeated games.
\newblock \emph{Econometrica}, 73\penalty0 (2):\penalty0 459--480, 2005.
\newblock \doi{10.1111/j.1468-0262.2005.00585.x}.

\bibitem[Ng and Russell(2000)]{ng2000algorithms}
Andrew~Y. Ng and Stuart Russell.
\newblock Algorithms for inverse reinforcement learning.
\newblock In \emph{Proceedings of the 17th International Conference on Machine
  Learning}, pages 663--670, 2000.

\bibitem[Ngo et~al.(2024)Ngo, Chan, and Mindermann]{ngo2024alignment}
Richard Ngo, Lawrence Chan, and S{\"o}ren Mindermann.
\newblock The alignment problem from a deep learning perspective.
\newblock In \emph{International Conference on Learning Representations}, 2024.
\newblock URL \url{https://arxiv.org/abs/2209.00626v6}.

\bibitem[Noguchi(2015)]{noguchi2015bayesian}
Yuichi Noguchi.
\newblock {Bayesian} learning, smooth approximate optimal behavior, and
  convergence to $\varepsilon$-{Nash} equilibrium.
\newblock \emph{Econometrica}, 83\penalty0 (1):\penalty0 353--373, 2015.
\newblock \doi{10.3982/ECTA9332}.

\bibitem[Oesterheld(2019)]{oesterheld2019robust}
Caspar Oesterheld.
\newblock Robust program equilibrium.
\newblock \emph{Theory and Decision}, 86\penalty0 (1):\penalty0 143--159, 2019.
\newblock \doi{10.1007/s11238-018-9679-3}.

\bibitem[Ouyang et~al.(2022)]{ouyang2022training}
Long Ouyang et~al.
\newblock Training language models to follow instructions with human feedback.
\newblock In \emph{Advances in Neural Information Processing Systems},
  volume~35, 2022.
\newblock URL \url{https://arxiv.org/abs/2203.02155}.

\bibitem[Pour et~al.(2026)Pour, Mansouri, and Ben-David]{pour2026learning}
Alireza~F. Pour, Farnam Mansouri, and Shai Ben-David.
\newblock Learning with multiple correct answers -- regret bounds under
  different feedback models.
\newblock arXiv preprint arXiv:2602.09402, 2026.

\bibitem[Rakhlin et~al.(2011)Rakhlin, Sridharan, and Tewari]{rakhlin2011online}
Alexander Rakhlin, Karthik Sridharan, and Ambuj Tewari.
\newblock Online learning: Stochastic, constrained, and smoothed adversaries.
\newblock In \emph{Advances in Neural Information Processing Systems},
  volume~24, 2011.

\bibitem[Russell(2019)]{russell2019human}
Stuart Russell.
\newblock \emph{Human Compatible: Artificial Intelligence and the Problem of
  Control}.
\newblock Viking, New York, 2019.

\bibitem[Russell et~al.(2015)Russell, Dewey, and Tegmark]{russell2015research}
Stuart Russell, Daniel Dewey, and Max Tegmark.
\newblock Research priorities for robust and beneficial artificial
  intelligence.
\newblock \emph{AI Magazine}, 36\penalty0 (4):\penalty0 105--114, 2015.
\newblock \doi{10.1609/aimag.v36i4.2577}.

\bibitem[Savage(1954)]{savage1954foundations}
Leonard~J. Savage.
\newblock \emph{The Foundations of Statistics}.
\newblock Wiley, New York, 1954.

\bibitem[Sekhari et~al.(2021)Sekhari, Dann, Mohri, Mansour, and
  Sridharan]{sekhari2021agnostic}
Ayush Sekhari, Christoph Dann, Mehryar Mohri, Yishay Mansour, and Karthik
  Sridharan.
\newblock Agnostic reinforcement learning with low-rank {MDPs} and rich
  observations.
\newblock In \emph{Advances in Neural Information Processing Systems},
  volume~34, pages 19033--19045, 2021.
\newblock URL
  \url{https://proceedings.neurips.cc/paper/2021/hash/9eed867b73ab1eab60583c9d4a789b1b-Abstract.html}.

\bibitem[Shah et~al.(2019)Shah, Krasheninnikov, Alexander, Abbeel, and
  Dragan]{shah2019preferences}
Rohin Shah, Dmitrii Krasheninnikov, Jordan Alexander, Pieter Abbeel, and Anca
  Dragan.
\newblock Preferences implicit in the state of the world.
\newblock In \emph{International Conference on Learning Representations}, 2019.
\newblock URL \url{https://arxiv.org/abs/1902.04198}.

\bibitem[Sion(1958)]{sion1958minimax}
Maurice Sion.
\newblock On general minimax theorems.
\newblock \emph{Pacific Journal of Mathematics}, 8\penalty0 (1):\penalty0
  171--176, 1958.
\newblock \doi{10.2140/pjm.1958.8.171}.

\bibitem[Soares(2016)]{soares2016value}
Nate Soares.
\newblock The value learning problem.
\newblock In \emph{Ethics for Artificial Intelligence Workshop at IJCAI-16},
  2016.
\newblock URL \url{https://intelligence.org/files/ValueLearningProblem.pdf}.

\bibitem[Spielman and Teng(2004)]{spielman2004smoothed}
Daniel~A. Spielman and Shang-Hua Teng.
\newblock Smoothed analysis of algorithms: Why the simplex algorithm usually
  takes polynomial time.
\newblock \emph{Journal of the ACM}, 51\penalty0 (3):\penalty0 385--463, 2004.
\newblock \doi{10.1145/990308.990310}.

\bibitem[Tennenholtz(2004)]{tennenholtz2004program}
Moshe Tennenholtz.
\newblock Program equilibrium.
\newblock \emph{Games and Economic Behavior}, 49\penalty0 (2):\penalty0
  363--373, 2004.
\newblock \doi{10.1016/j.geb.2004.02.002}.

\bibitem[Walley(1991)]{walley1991statistical}
Peter Walley.
\newblock \emph{Statistical Reasoning with Imprecise Probabilities}.
\newblock Chapman and Hall, London, 1991.

\bibitem[Wyeth et~al.(2025)Wyeth, Hutter, Leike, and Taylor]{wyeth2025grain}
Cole Wyeth, Marcus Hutter, Jan Leike, and Jessica Taylor.
\newblock Limit-computable grains of truth for arbitrary computable
  extensive-form (un)known games.
\newblock arXiv preprint arXiv:2508.16245, 2025.

\end{thebibliography}
